\documentclass{article}

\usepackage[utf8]{inputenc} % allow utf-8 input
\usepackage[T1]{fontenc}    % use 8-bit T1 fonts
\usepackage{hyperref}       % hyperlinks
\usepackage{url}            % simple URL typesetting
\usepackage{booktabs}       % professional-quality tables
\usepackage{amsfonts}       % blackboard math symbols
\usepackage{nicefrac}       % compact symbols for 1/2, etc.
\usepackage{microtype}      % microtypography
\usepackage{xcolor}         % colors
\usepackage{xcolor,colortbl}
\usepackage{caption}
\usepackage{amssymb}% http://ctan.org/pkg/amssymb
\usepackage{pifont}% http://ctan.org/pkg/pifont

\usepackage{graphicx}
\usepackage{subfigure}
\usepackage{wrapfig,amsmath,amssymb,bm,comment,color,mathbbol}
\usepackage{breakurl,epsfig,epsf,fmtcount,semtrans,multirow,boldline}
\usepackage{tcolorbox}
\tcbuselibrary{skins}
\usepackage{tikz}
\usetikzlibrary{tikzmark}
\usepackage{mathtools}

\definecolor{darkred}{RGB}{150,0,0}
\definecolor{darkgreen}{RGB}{0,150,0}
\definecolor{darkblue}{RGB}{0,0,200}
\hypersetup{colorlinks=true, linkcolor=darkred, citecolor=darkgreen, urlcolor=darkblue}

\newtheorem{theorem}{Theorem}%[section]

\newtheorem{assumption}{Assumption}

\newtheorem{lemma}{Lemma}

\newtheorem{definition}{Definition}

\newtheorem{remark}{Remark}

\numberwithin{equation}{section}

\def \endprf{\hfill {\vrule height6pt width6pt depth0pt}\medskip}

 \newenvironment{proof}{\noindent {\bf Proof.} }{\endprf\par}

\newcommand{\tsn}[1]{{\left\vert\kern-0.25ex\left\vert\kern-0.25ex\left\vert #1 
    \right\vert\kern-0.25ex\right\vert\kern-0.25ex\right\vert}}

\newcommand{\red}{\textcolor{red}}

\newcommand{\yl}[1]{\textcolor{black}{#1}}

\newcommand{\cln}[1]{\red{}}

\newenvironment{myitemize}{\begin{list}{$\bullet$}
		{\setlength{\topsep}{1mm}
			\setlength{\itemsep}{0.25mm}
			\setlength{\parsep}{0.25mm}
			\setlength{\itemindent}{0mm}
			\setlength{\partopsep}{0mm}
			\setlength{\labelwidth}{15mm}
			\setlength{\leftmargin}{4mm}}}{\end{list}}

\newcommand{\pmt}{{\p}}

\newcommand{\TF}{{\texttt{TF}}\xspace}
\newcommand{\TFLR}{{\texttt{TF}_{\texttt{LR}}}\xspace}
\newcommand{\TFBE}{{\texttt{TF}_{\texttt{BE}}}}

\newcommand{\eps}{\varepsilon}

\newcommand{\st}{\star}

\usepackage{xspace}

\newcommand{\xtest}{\x_{\text{\tiny test}}}

\newcommand{\beq}{\begin{equation}}
\newcommand{\ba}{\begin{align}}
\newcommand{\ea}{\end{align}}

\newcommand{\eeq}{\end{equation}}

\newcommand{\nn}{\nonumber}

\newcommand{\Sb}{{{\mtx{S}}}}

\newcommand{\Lc}{{\cal{L}}}

\newcommand{\Nc}{{\cal{N}}}

\newcommand{\Dc}{{\cal{D}}}

\newcommand{\Tb}{{\mtx{T}}}

\newcommand{\one}{{\bm{1}}}

\newcommand{\Iden}{{\mtx{I}}}
\newcommand{\M}{{\mtx{M}}}

\newcommand{\z}{{\vct{z}}}

\newcommand{\tht}{\hat{\tb}}
\newcommand{\tbt}{\tb_{\text{\tiny test}}}
\newcommand{\tn}[1]{\|{#1}\|}%_{\ell_2}

\newcommand{\bt}{{\boldsymbol{\theta}}}

\newcommand{\bbeta}{{\boldsymbol{\beta}}}

\newcommand{\vb}{\vct{v}}

\newcommand{\w}{\vct{w}}

\newcommand{\s}{\vct{s}}

\newcommand{\Fc}{\mathcal{F}}

\newcommand{\Xc}{\mathcal{X}}

\newcommand{\Yc}{\mathcal{Y}}

\newcommand{\m}{\vct{m}}

\newcommand{\abs}[1]{\left|#1\right|}

\newcommand{\x}{\vct{x}}

\newcommand{\y}{\vct{y}}

\newcommand{\W}{\mtx{W}}

\definecolor{emmanuel}{RGB}{255,127,0}

\newcommand{\p}{{\vct{p}}}

\newcommand{\pb}{{\vct{p}}}
\newcommand{\pbf}{{\vct{p}}^{\text{filter}}}

\newcommand{\R}{\mathbb{R}}

\newcommand{\E}{\operatorname{\mathbb{E}}}

\newcommand{\e}{\mathrm{e}}

\newcommand{\tb}{\vct{s}}

\newcommand{\vct}[1]{\bm{#1}}
\newcommand{\mtx}[1]{\bm{#1}}

\newcommand{\bPi}{\boldsymbol{\Pi}}

\newcommand{\X}{{\mtx{X}}}

\newcommand{\attn}{\textit{attn}}
\newcommand{\sfm}{\textit{softmax}}

\newcommand{\COT}{\text{CoT-I/O}\xspace}
\newcommand{\EXP}{\text{CoT-I}\xspace}
\newcommand{\ICL}{\text{ICL}\xspace}

\newcommand{\cellcolg}{\cellcolor{green!20}}
\newcommand{\cellcolr}{\cellcolor{red!20}}

\def\vb{{\bm{b}}}

\def\vv{{\bm{v}}}

\def\vx{{\bm{x}}}

\def\mA{{\bm{A}}}

\def\mW{{\bm{W}}}

    \usepackage[final]{neurips_2023}

\usepackage[utf8]{inputenc} % allow utf-8 input
\usepackage[T1]{fontenc}    % use 8-bit T1 fonts
\usepackage{hyperref}       % hyperlinks
\usepackage{url}            % simple URL typesetting
\usepackage{booktabs}       % professional-quality tables
\usepackage{amsfonts}       % blackboard math symbols
\usepackage{nicefrac}       % compact symbols for 1/2, etc.
\usepackage{microtype}      % microtypography
\usepackage{xcolor}         % colors

\usepackage{bm}
\usepackage{bbm}
\usepackage[mathscr]{eucal}

\def\1{\bm{1}}

\def\eps{{\epsilon}}

\def\rvr{{\mathbf{r}}}

\def\rmA{{\mathbf{A}}}
\def\rmB{{\mathbf{B}}}

\def\rmK{{\mathbf{K}}}

\def\rmM{{\mathbf{M}}}

\def\rmQ{{\mathbf{Q}}}

\def\rmV{{\mathbf{V}}}

\def\vZero{{\bm{0}}}

\def\vb{{\bm{b}}}

\def\vv{{\bm{v}}}

\def\vx{{\bm{x}}}

\def\mA{{\bm{A}}}

\def\mW{{\bm{W}}}

\DeclareMathAlphabet{\mathsfit}{\encodingdefault}{\sfdefault}{m}{sl}
\SetMathAlphabet{\mathsfit}{bold}{\encodingdefault}{\sfdefault}{bx}{n}

\newcommand{\seqlen}{\text{N}}
\newcommand{\weights}{\mW}
\def\vZero{{\bm{0}}}
\newcommand{\zero}{\vZero}

\newcommand{\val}{\rmV}
\newcommand{\key}{\rmK}
\newcommand{\query}{\rmQ}

\newcommand{\emb}{\rvr}

\usepackage[toc,page,header]{appendix}
\usepackage{minitoc}
\newcommand{\ie}{{\it i.e.}, }

\title{Dissecting Chain-of-Thought: A Study on Compositional In-Context Learning of MLPs}
\title{Dissecting Chain-of-Thought: Compositionality through In-Context Filtering and Learning}
\author{%
  Yingcong Li \\
  University of California, Riverside\\
  \texttt{yli692@ucr.edu} \\
  \And
  Kartik Sreenivasan\quad\quad Angeliki Giannou  \\
  University of Wisconsin-Madison\\
  \texttt{\{ksreenivasa2,giannou\}@wisc.edu} \\
  \And
  Dimitris Papailiopoulos \\
  University of Wisconsin-Madison\\
  \texttt{dimitris@papail.io} \\
  \And
  Samet Oymak \\
  University of Michigan \& UC Riverside\\
  \texttt{oymak@umich.edu} \\
}

\begin{document}
\doparttoc % Tell to minitoc to generate a toc for the parts
\faketableofcontents % Run a fake tableofcontents command for the partocs without creating table of contents for the main paper

\maketitle

\begin{abstract}
Chain-of-thought (CoT) is a method that enables language models to handle complex reasoning tasks by decomposing them into simpler steps. Despite its success, the underlying mechanics of CoT are not yet fully understood. In an attempt to shed light on this, our study investigates the impact of CoT on the ability of transformers to in-context learn a simple to study, yet general family of compositional functions: multi-layer perceptrons (MLPs).
In this setting, we find that the success of CoT can be attributed to breaking down in-context learning of a compositional function into two distinct phases: focusing on \yl{and filtering} data related to each step of the composition and in-context learning the single-step composition function.
Through both experimental and theoretical evidence, we demonstrate how CoT significantly reduces the sample complexity of in-context learning (ICL) and facilitates the learning of complex functions that non-CoT methods 
struggle with. \yl{Furthermore, we illustrate how transformers can transition from vanilla in-context learning to mastering a compositional function with CoT by simply incorporating additional layers that perform the necessary data-filtering for CoT via the attention mechanism.} In addition to these test-time benefits, we \yl{show CoT helps accelerate} pretraining by learning shortcuts to represent complex functions and \yl{filtering plays an important role in this process}.
These findings collectively provide insights into the mechanics of CoT, inviting further investigation of its role in complex reasoning tasks.
\end{abstract}

% \input{fig_sec/graph}

% \vspace{-10pt}
\section{Introduction}
% \vspace{-5pt}
%, such as machine translation, sentiment analysis, and question-answering systems
The advent of transformers~\citep{vaswani2017attention} has revolutionized natural language processing, paving the way for remarkable performance in a wide array of tasks. LLMs, such as~GPTs~\citep{brown2020language}, have demonstrated an unparalleled ability to capture and leverage vast amounts of data, thereby facilitating near human-level performance across a variety of language generation tasks. Despite this success, a deep understanding of their underlying mechanisms remains elusive.

Chain-of-thought prompting \citep{wei2022chain} is an emergent ability of transformers where the model solves a complex problem~\citep{wei2022emergent}, by decomposing it into intermediate steps. Intuitively, this underlies the ability of general-purpose language models to accomplish previously-unseen complex tasks by leveraging more basic skills acquired during the pretraining phase. Compositional learning and CoT has enjoyed significant recent success in practical language modeling tasks spanning question answering, code generation, and mathematical reasoning \citep{perez2021automatic, imani2023mathprompter, yuan2023well}. In this work, we attempt to demystify some of the mechanics underlying this success and the benefits of CoT in terms of sample complexity and approximation power. To do this we explore the role of CoT in learning multi-layer perceptrons (MLPs) in-context, which we believe can lead to a first set of insightful observations. Throughout, we ask:
% To do this we explore the role of CoT in learning in-context multi-layer perceptrons (MLP),a universal approximation to arbitrary composition functions, that we believe can lead to a first set of insightful observations. Throughout, we ask: 
% \vspace{-2pt}

\begin{center}
\emph{Does CoT improve in-context learning of MLPs, and what are the underlying mechanics?}
\end{center}
%
%``heterogeneous'' ``homogeneous'' 
% \noindent\textbf{Contributions:} As our central contribution, we establish a rigorous and experimentally-supported abstraction that decouples CoT prompting into a \emph{filtering phase} and an \emph{in-context learning (ICL) phase}. In \emph{filtering}, the model attends to the relevant tokens within the prompt based on an instruction. In \emph{ICL}, the model runs inference on the filtered prompt to output a \emph{step} and then moves onto the next \emph{step} in the chain. This process is formalized for the transformer architecture in Theorem \ref{thm:MLP}.
%
% Building on this
In this work, we identify and thoroughly compare three schemes as illustrated in Figure~\ref{fig:intro}. (a) ICL: In-context learning from input-output pairs provided in the prompt, (b) CoT-I: Examples in the prompt are augmented with intermediate steps, (c) CoT-I/O: The model also outputs intermediate steps during prediction. Our main contributions are:

\input{fig_sec/fig_intro}

\begin{myitemize}
\item \textbf{Decomposing CoT into filtering and ICL:} As our central contribution, we establish a rigorous and experimentally-supported abstraction that decouples CoT prompting into a \emph{filtering phase} and an \emph{in-context learning (ICL) phase}. In \emph{filtering}, the model attends to the relevant tokens within the prompt based on an instruction. In \emph{ICL}, the model runs inference on the filtered prompt to output a \emph{step} and then moves onto the next \emph{step} in the chain. 
% This process is formalized for the transformer architecture in Theorem \ref{thm:MLP}. 
\yl{Our Theorem~\ref{thm:MLP} develops a theoretical understanding of this two-step procedure and formalizes how filtering and ICL phases of CoT can be implemented via the self-attention mechanism to learn MLPs.
}
% Building on this, we identify and thoroughly compare three schemes as illustrated in Figure~\ref{fig:intro}. (a) ICL: In-context learning from input-output pairs provided in the prompt, (b) CoT-I: Examples in the prompt are augmented with intermediate steps, (c) CoT-I/O: The model also outputs intermediate steps during prediction.}
\item \textbf{Approximation and sample complexity:} Through experiments and theory, we establish that intermediate steps in CoT-I improves the sample complexity of learning whereas step-by-step output improves the approximation ability through looping. Specifically, CoT-I/O can learn an MLP with input dimension $d$ and $k$ neurons using ${\cal{O}}(\max(k,d))$ in-context samples by filtering individual layers and solving them via linear regression -- in contrast to the $\Omega(kd)$ lower bound without step-augmented prompt. As predicted by our theory, our experiments (\yl{see Sec.~\ref{sec:exp}}) identify a striking universality phenomenon (as $k$ varies) and also demonstrate clear approximation benefits of CoT compared to vanilla ICL.
%reveal that CoT accuracy is remarkably agnostic to $k$ whenever $k\leq d$ in support of our theory. In contrast, transformer can fail to learn a two-layer MLP with ICL or CoT-I even with large sample size  demonstrating the approximation benefits of CoT. 
% \item CoT increases model expressivity; in-context sample efficiency
%Learning shallow MLPs

\item \textbf{Accelerated pretraining via learning shortcuts:} We construct deep linear MLPs where each layer is chosen from a discrete set of matrices. This is in contrast to the above setting, where MLP weights can be arbitrary. We show that CoT can dramatically accelerate pretraining by memorizing these discrete matrices and can infer all layers correctly from a \emph{single} demonstration. Notably the pretraining loss goes to zero step-by-step where each step \emph{``learns to filter a layer''}. Together, these showcase how CoT identifies composable shortcuts to avoid the need for solving linear regression. In contrast, we show that ICL (without CoT) collapses to linear regression performance as it fails to memorize exponentially many candidates (due to lack of composition).

%(ii) Pretraining loss goes to zero step-by-step where each step \emph{``learns to filter a layer''}. (iii) As the chain gets shorter, there are more matrices (as combinations of the initial ones) to memorize and convergence substantially slows down. (iv) Without CoT, the model fails to learn shortcuts and its performance collapses to linear regression.

%learning shortcuts: Specifically, rather than learning linear regression, CoT can memorize

%with deep linear MLPs reveal that transformer indeed \emph{learns to filter} during pretraining which leads to sudden drops in the test risk and accelerated convergence via CoT. It also evidences the emergent ability of CoT: learning each step function separately and then composing them. In addition to these benefits during inference time, we highlight how CoT accelerates pretraining by learning shortcuts to representing complex functions and how filtering plays a visible role.

% \item Finally, through linear MLPs, we demonstrate that a transformer pretrained with only the short chains can compose long previously-unseen chains during test-time by acquiring skills (e.g.~memorizing MLP weights), and \emph{learning to filter and compose them}.
\end{myitemize}
The paper is organized as follows. In Section \ref{sec prelim}, we introduce the problem setup. Section \ref{sec:theory} states our main theoretical results which decouple CoT into filtering and ICL. Section \ref{sec:exp} provides empirical investigations of CoT with 2-layer MLPs, which validates our theoretical findings.
% Section \ref{sec:main} provides an empirical investigation of CoT with 2-layer MLPs and states our main theoretical results. Section~\ref{sec:exp} presents holistic experiments on the sample complexity and approximation benefits of CoT. 
Finally, we elucidate the benefits of CoT during pretraining via experiments on deep linear MLPs in Section \ref{sec:deep nn}. Related work and discussion  are provided in Sections~\ref{sec:related} and \ref{sec:discuss}, respectively.
% elucidating how CoT accelerates pretraining and enables the composition of long novel chains

%These contributions not only deepen our understanding of the inner workings of transformer models but also serve as a foundation for future research and optimization efforts. 

% \input{fig_sec/fig_intro}
% \input{fig_sec/fig_main}
% \vspace{-5pt}
\section{Preliminaries and Setup}\label{sec prelim}

\emph{Notation.} We denote the set $\{1,2,\ldots,n\}$ as $[n]$. Vectors and matrices are represented in bold text (e.g., $\vx, \mA$), while scalars are denoted in plain text (e.g., $y$). 
The input and output domains are symbolized as $\Xc$ and $\Yc$ respectively (unless specified otherwise), and $\x\in\Xc$, $\y\in\Yc$ denote the input and output. 
\yl{Additionally, let $\Fc$ be a set of functions from $\Xc$ to $\Yc$. Consider a transition function $f \in \Fc$ where $\y = f(\x)$. In this section, we explore the properties of learning $f$, assuming that it can be decomposed into simpler functions $(f_\ell)_{\ell=1}^L$, and thus can be expressed as $f = f_L \circ f_{L-1} \circ \dots \circ f_1$.}

% \yl{Additionally, let $f\in\Fc$ denote the transition function where $\Fc:\Xc\to\Yc$ is the function set and we have $\y=f(\x)$. In this section, we consider general transition function $f$ and suppose it can be decomposed into simpler functions, that is, $f:=f_L\circ f_{L-1}\circ\dots f_1$.}
% Then an $L$-layer MLP can be presented as $f:=f_L\circ f_{L-1}\circ\dots f_1$ where $f_\ell$ represents the function of $\ell$th layer. Note that }
% \subsection{The transformer architecture}

\subsection{In-context Learning}\label{subsec icl}
Following the study by \citet{garg2022can}, the fundamental problem of vanilla in-context learning (ICL) involves constructing a prompt with input-output pairs in the following manner: 
% considers a prompt with input-output pairs of the following form:
\begin{align}
\pmt_n(f)=(\x_i,\y_i)_{i=1}^n~~~~\text{where}~~~~\y_i=f(\x_i).\tag{P-ICL}\label{icl pmt}
\end{align}
Here the transition function $f\in\Fc:\Xc\to\Yc$ remains constant within a single prompt but can vary across prompts, and the subscript $n$ signifies the number of in-context samples contained in the prompt. Considering language translation as an example, $f$ is identified as the target language, and the prompt can be defined as  $\pmt(\texttt{Spanish})$ = ((\emph{apple, manzana}), (\emph{ball, pelota}), $\ldots$) or $\pmt(\texttt{French})$=((\emph{cat, chat}), (\emph{flower, fleur}), $\ldots$). Let $\TF$ denote any auto-regressive model (e.g., Decoder-only Transformer). The aim of in-context learning is to learn a model that can accurately predict the output, given a prompt $\pmt$ and the test input $\xtest$, as shown in the following equation:
\begin{equation}\label{icl out}
\TF(\pmt_n(\tilde f),\xtest) \approx \tilde f(\xtest)
\end{equation}
where $\tilde f\in\Fc$ is the test function which may differ from the functions used during training. Previous work \citep{zhou2022teaching, li2023transformers} has demonstrated that longer prompts (containing more examples $n$) typically enhance the performance of the model.
% \AG{Unnecessary-- For the rest of the paper, we will denote vanilla in-context-learning as ICL.}

\subsection{Chain-of-thought Prompt and Prediction}\label{subsec cot}
% \vspace{-5pt}

{As defined in \eqref{icl pmt}, the prompt in vanilla ICL only contains input-output pairs of the target function. This} 
%As defined in \eqref{icl pmt}, vanilla ICL prompting contains input-output pairs only, which 
demands that the model learns the function $f\in\Fc$ in one go, which becomes more challenging as $\Fc$ grows more complex, since larger models and increased prompt length ($n$) are needed to make correct predictions (\yl{as depicted by the green curves in Figures~\ref{fig:diff gpt} and \ref{fig:diff k}}). {Existing studies on chain-of-thought methods (e.g., \citep{wei2022chain}) observed that prompts containing step-by-step instructions assist }%Existing chain-of-thought research \citep{wei2022chain} has observed that prompt containing step-by-step \emph{thoughts} helps 
the model in decomposing the function and making better predictions. 
%% The main benefit is that, instead of predicting the final output ($\y$) directly, CoT takes intermediate outputs (referred to as \emph{thought}) into consideration. 
Specifically, consider a function composed of $L$ subfunctions, represented as $f:=f_L\circ f_{L-1}\circ\dots f_1$. Each intermediate output 
{can be viewed as a step, enabling us to define a }%is a thought and we can define 
length-$n$ CoT prompt related to $f$ with $L$ steps (expressed with $\tb^\ell,\ell\in[L]$) as follows:
% For clear notation, let $\tb_i^0=\x_i$, $\tb_i^L=\y_i$ and $\f:=(f_1,f_2,\cdots,f_L)$.
\begin{align}
\pmt_n(f)=(\x_i,\tb_i^1,\cdots\tb_i^{L-1},\tb_i^L)_{i=1}^n~~~~\text{where}~~~~\tb_i^{\ell}=f_\ell(\tb_i^{\ell-1}),~\ell\in[L].\tag{P-CoT}\label{cot pmt}
\end{align}
Here $\x_i=\tb_i^0$, $\y_i=\tb_i^L$ {and $f_\ell\in\Fc_\ell$, which implies that $f\in\Fc_L\times\cdots\Fc_1:=\Fc$.}
%Here let $f_\ell\in\Fc_\ell$ and then $f\in\Fc_1\times\cdots\Fc_L:=\Fc$. 
%% Let $\f(\x):=f_L(f_{L-1}(\cdots f_1(\x))$. 

% \AG{Not sure what these want to say: ***In addition to admitting CoT prompts, we are asking how the $\TF$ preforms with or without predicting thoughts.  
% Then given a test input $\x$ and prompt from \eqref{cot pmt} corresponding to test function $\tilde f\in\Fc$, we define the model prediction in two ways: ***}

Next we introduce two methodologies for making predictions within the CoT framework:
% different ways that predictions can be made in the context of CoT:

\textbf{CoT over input only ($\EXP$). }Contrasted with ICL, $\EXP$ considers step-by-step instructions as inputs, 
% breaking down the computation of the final value into a collection of simpler tasks.
nonetheless, the prediction for the last token is performed as a single entity. 
% Compared to ICL, $\EXP$ takes in step-by-step instructions as inputs, which {fragments the calculation of the target value into a set of simpler ones. However, the prediction over the last token is performed all together in one shot}. %help in `explaining' the problem better. 
    % can be thought as explanations. Here, the intermediate thoughts are provided within the input prompts. It can also be seen as explanations of the input function target ($\tilde\f$). 
 Our experiments indicate that this approach lowers the sample complexity for $\TF$ to comprehend the function $\tilde f$ being learned (\yl{see the orange curves in Figures~\ref{fig:diff gpt} and \ref{fig:diff k}}). The $\EXP$ prediction aligns with Eq. \eqref{icl out} as follows, while the prompt is determined by \eqref{cot pmt}.
    \begin{align}
    &\text{One-shot prediction: }\TF(\pmt_n(\tilde f),\xtest)\approx \tilde  f(\xtest).\label{exp out}
    \end{align}
%$\TF$ still needs to be able to express any function set $\Fc$, which will be very challenging (e.g., if assuming $\Fc_\ell=\tilde \Fc,\ell\in[L]$, then $\Fc=\tilde\Fc^L$ whose complexity is exponential to the thought step size $L$.).
\textbf{CoT over both input and output ($\COT$). } Despite the fact that $\EXP$  improves the sample complexity of learning $\tilde f$,  the $\TF$ must still possess the capacity to approximate functions from the function class $\Fc$, given that the prediction is made in one shot. To mitigate this challenge, we consider a scenario where in addition to implementing a CoT prompt, we also carry out CoT predictions. Specifically, for a composed problem with inputs formed via \eqref{cot pmt}, the model recurrently makes $L$-step predictions as outlined below:
    \begin{align}
    &\text{Step 1: }\TF(\pmt_n(\tilde f),\xtest):=\hat\tb^1\nn\\
    &\text{Step 2: }\TF(\pmt_n(\tilde f),\xtest,\hat\tb^1):=\hat \tb^2\nn\\
    &~~~~\vdots\nn\\
    &\text{Setp $L$: }\TF(\pmt_n(\tilde f),\xtest,\hat\tb^1\cdots,\hat\tb^{L-1})\approx \tilde 
 f(\xtest),\label{cot out}
    \end{align}
where at each step (step $\ell$), the model outputs an intermediate step ($\hat\tb^\ell$) which is then fed back to the input sequence %to help 
 to facilitate the next-step prediction ($\hat\tb^{\ell+1}$). 
% \YL{[MIGHT NOT NEEDED!]As depicted in Figure~\ref{fig:intro}, (\emph{blue, orange, yellow}) blocks represent ($\x$, $\tb$, $\y$) respectively, and thought output $\hat\tb$ (the toppest orange block) is feed back to the input which completes the recurrent prediction.} 
Following this strategy, the model only needs to learn the union of the sub-function sets, $\bigcup_{\ell=1}^L\Fc_\ell$, whose complexity scales linearly with the number of steps $L$. Empirical evidence of the benefits of $\COT$ over $\ICL$ and $\EXP$ in enhancing sample efficiency and model expressivity is reflected in the \yl{blue curves shown in Figures~\ref{fig:diff gpt} and \ref{fig:diff k}}.

\subsection{Model Training}\label{sec:train}
In Figure~\ref{fig:intro} and Section~\ref{sec prelim}, we have discussed vanilla $\ICL$, $\EXP$ and $\COT$ methods. 
% , whose inputs structure %data structures are 
% is shown in Table~\ref{tab:intro}. 
Intuitively, $\ICL$ can be viewed as a special case of $\EXP$ (or $\COT$) if we assume only one step is performed. Consequently, we will focus on implementing  $\EXP$ and $\COT$ for model training in the following. %, while $\ICL$ will be a special case of that implementation.%we will introduce the implementation of $\EXP$ and $\COT$ model training in the following and $\ICL$ can be easily applied by setting $L=1$. 

Consider the CoT prompt as in  %same CoT prompt setting as in 
\eqref{cot pmt}, and assume that $\x\sim\Dc_{\Xc}$,  %and training functions obey 
and $f_\ell\sim\Dc_{\ell},\ell\in[L]$, where $L$ denotes the number of compositions/steps, such that the final prediction should approximate $f(\x)=f_L(f_{L-1}\dots f_1(\x)):=\y\in\Yc$. We define $\ell(\hat\y,\y):\Yc\times\Yc\to\R$ as a loss function. For simplicity, we assume $f_\ell(\dots f_1(\x))\in\Yc$, $\ell\in[L]$. Let $N$ represent the in-context window of $\TF$, which implies that $\TF$ can only admit a prompt containing up to $N$ in-context samples. Generally, our goal is to ensure high prediction performance given any length-$n$ prompt, where $n\in[N]$. To this end, we train the model using prompts with length from $1$ to $N$ equally and aim to minimize the averaged risk over different prompt size. 
Assuming the model $\TF$ is parameterized by $\bt$ and considering meta learning problem, the objective functions for $\EXP$ and $\COT$ are defined as follows. 
% and CoT prompt structures (see \eqref{icl pmt} and \eqref{cot pmt}), and also one-shot and multi-step prediction. 
% and intuitively, ICL prompting is a special case of CoT prompting. Therefore, next, we introduce our training and evaluating process that admit CoT prompts as datasets, which is directly applied to the ICL setting.
% Consider meta learning problems where input $\x_i\in\Dc_{\x},i\in[N]$ and functions $\f\sim\Dc_{\f}$. Let $\ell(\hat\y,\y):\Yc\times\Yc\to\R$ be a loss function, and assume model $\TF$ is parameterized by $\bt$. Then the objective functions for $\EXP$ and $\COT$ are as follows. 
% \begin{align*}
% \hat\bt^{\EXP}=\arg\min_{\bt}\E_{(\x_n)_{n=1}^N,(f_\ell)_{\ell=1}^L}\left[\frac{1}{N}\sum_{n=1}^N\ell(\hat\y_n,f(\x_n))\right],~~\text{where}~~\hat\y_n=\TF(\pmt_n(f),\x_n)%\label{erm cot i}
% \end{align*}
% and
% \begin{align*}
% \hat\bt^{\COT}=\arg\min_{\bt}\E_{(\x_n)_{n=1}^N,(f_\ell)_{\ell=1}^L}\left[\frac{1}{NL}\sum_{n=1}^N\sum_{\ell=1}^L\ell(\hat\tb_n^\ell,\tb_{n}^\ell)\right],~~\text{where}~~\hat\tb_n^\ell=\TF(\pmt_n(f),\x_n\cdots\tb_n^{\ell-1}). %\label{erm cot io}
% \end{align*}

% Here $\pmt_n(f)$ is given by \eqref{cot pmt}, and as mentioned previously,  $\tb^0=\x$ and $\tb^L=\y$. All $\x$ and $f_\ell$ are independent, and we take the expectation of the risk over their respective distributions. 

\begin{align*}
\hat\bt^{\EXP}=\arg\min_{\bt}\E_{(\x_n)_{n=1}^N,(f_\ell)_{\ell=1}^L}\left[\frac{1}{N}\sum_{n=1}^N\ell(\hat\y_n,f(\x_n))\right]
\end{align*}
and
\begin{align*}
\hat\bt^{\COT}=\arg\min_{\bt}\E_{(\x_n)_{n=1}^N,(f_\ell)_{\ell=1}^L}\left[\frac{1}{NL}\sum_{n=1}^N\sum_{\ell=1}^L\ell(\hat\tb_n^\ell,\tb_{n}^\ell)\right]
\end{align*}

where $\hat\y_n=\TF(\pmt_n(f),\x_n)$ and $\hat\tb_n^\ell=\TF(\pmt_n(f),\x_n\cdots\tb_n^{\ell-1})$. $\pmt_n(f)$ is given by \eqref{cot pmt}, and as mentioned previously,  $\tb^0=\x$ and $\tb^L=\y$. All $\x$ and $f_\ell$ are independent, and we take the expectation of the risk over their respective distributions.

% \input{sec/setup (backup v2)}
% \section{Sample Complexity in 2-layer MLPs}\label{sec:2nn}
% \input{sec/sample-2nn}

% \section{Decomposing ICL}\label{sec:decompose}
% \input{sec/filtering}

% \section{Emergent abilities}\label{sec:emergent}
% \input{sec/emergent}
% \newpage
\input{fig_sec/fig_filter_v2}
\section{Provable Approximation of MLPs via Chain-of-Thought}\label{sec:theory}
In this section, we present our theoretical findings that demonstrate how $\COT$ can execute filtering over the CoT prompt, thereby learning a 2-layer MLP with input dimension of $d$ and hidden dimension of $k$, akin to resolving $k$ $d$-dimensional ReLU problems and $1$ $k$-dimensional linear regression problem. \yl{Subsequently, in Section~\ref{sec:emp}}, we examine the performance of $\COT$ when learning 2-layer random MLPs. Our experiments indicate that $\COT$ needs only $O(\max(d,k))$ in-context samples to learn the corresponding MLP. 

% The observations we made in Section~\ref{sec:emp} are indicative of the model processing each one of the two layers sequentially. Now in this section, w
We state our main contribution of
% ; an observation we also verify empirically in the subsequent sections. 
% Our main contribution is 
establishing a result that decouples CoT-based in-context learning ($\COT$) into two phases: (1) \emph{Filtering Phase:} Given a prompt that contains features of multiple MLP layers, retrieve only the features related to a target layer to create an ICL prompt. (2) \emph{ICL Phase:} Given filtered prompt, learn the target layer weights through gradient descent. Combining these two phases, and looping over all layers, we will show that there exists a transformer architecture such that $\COT$ can provably approximate a multilayer MLP up to a given resolution. \yl{An illustration is provided in Figure~\ref{fig:filter}.} To state our result, we assume access to an oracle that performs linear regression and consider the consider the condition number of the data matrix.%make the following assumption on the MLP features provided within the prompt.
%\begin{assumption} \end{assumption}
\begin{definition}[MLP and condition number]\label{MLP assump} Consider a multilayer MLP defined by the recursion $\tb_i^{\ell}=\phi(\W_\ell \tb_i^{\ell-1})$ for $\ell\in[L]$, $i\in[n]$ and $\tb_i^0=\x_i$. Here $\phi(x):=\max(\alpha x,x)$ is a Leaky-ReLU activation with $1\geq \alpha>0$. Define the feature matrix $\Tb_\ell=[\tb_1^\ell~\dots~\tb_n^\ell]^\top$ and  
% where $\tb_{\ell,k}=\tb_{\ell}(\x_k)$ is the $k$'th chain element of the example $\x_k$. 
define its condition number $\kappa_\ell=\sigma_{\max}(\Tb_\ell)/\sigma_{\min}(\Tb_\ell)$ (with $\sigma_{\min}:=0$ for fat matrices) and $\kappa_{\max}=\max_{0\leq\ell< L}\kappa_\ell$.
%efine the condition number associated with layer $\ell$ as
\end{definition}
\begin{assumption}[Oracle Model]\label{assume oracle} We assume access to a transformer $\TFLR$ which can run $T$ steps of gradient descent on the quadratic loss $\Lc(\w)=\sum_{i=1}^n(y_i-\w^\top \x_i)^2$ given a prompt of the form $(\x_1,y_1,\dots, \x_n,y_n)$.
\end{assumption}
We remark that this assumption is realistic and has been formally established by earlier work \citep{giannou2023looped,akyurek2022learning}. Our CoT abstraction builds on these to demonstrate that $\COT$ can call a blackbox \TF model to implement a compositional function when combined with filtering.

\yl{We now present our main theoretical contribution. Our result provides a transformer construction that first filters a particular MLP layer from the prompt through the attention mechanism, then applies in-context learning, and repeats this procedure to approximate the MLP output.} The precise statement is deferred to the supplementary material. 
\begin{theorem}[\yl{CoT$\Leftrightarrow$Filtering+ICL}]\label{thm:MLP} Consider a CoT prompt $\pmt_n(f)$ generated from an $L$-layer MLP $f(\cdot)$ as described in Definition \ref{MLP assump}, and assume given test example $(\xtest,\tb^1_{\text{\tiny test}}, \dots\tb^L_{\text{\tiny test}})$. 
% , containing $n$ examples and ends with $\x^{\texttt{test}}$. 
For any resolution $\eps>0$, there exists $\delta=\delta(\eps)$, iteration choice $T={\cal{O}}(\kappa_{\max}^2\log(1/\eps))$, and a backend transformer construction $\TFBE$ such that the concatenated transformer $\TF=\TFLR \circ \TFBE$ implements the following: Let $(\tht^{i})_{i=0}^{\ell-1}$ denote the first $\ell-1$ $\COT$ outputs of $\TF$ where $\tht^{0}=\xtest$ and set $\pb[\ell]=(\pb_n(f),\xtest,\tht^{1}\dots\tht^{\ell-1})$. At step $\ell$, $\TF$ implements
\begin{enumerate}
\item \textbf{Filtering.} Define the filtered prompt with input/output features of layer $\ell$,
\[
\arraycolsep = 1pt
\pbf_n=\left(\begin{array}{*{10}c}
\hdots \zero,&\tb^{\ell-1}_{1},&\zero&\hdots&\zero,&\tb^{\ell-1}_n, &\zero&\hdots&\zero,&\tht^{\ell-1}\\
\hdots\zero,&\zero,& \tb^{\ell}_{1}&\dots&\zero,&\zero,& \tb^{\ell}_n&\dots&\zero,&\zero\end{array}\right).
\]
There exists a fixed projection matrix $\bPi$ that applies individually on tokens such that the backend output obeys $\tn{\bPi(\TFBE(\pb[\ell]))-\pbf_n}\leq \delta$. %The more precise statement is provided in Lemma \ref{lem:filtering}. % up to $\eps$ accuracy via a fixed linear projection. 
\item \textbf{Gradient descent.} The combined model obeys $\tn{\TF(\pb[\ell])-\tb_{\text{\tiny test}}^{\ell}}\leq {\ell\cdot \eps}/{L}$.
\end{enumerate}
$\TFBE$ has constant number of layers independent of $n$ and $T$. Consequently, after $L$ rounds of $\COT$, \TF outputs $f(\xtest)$ up to $\eps$ accuracy.
\end{theorem}
% (more generally we expect
{\begin{remark}
Note that, this result effectively shows that, with a sufficiently good blackbox transformer $\TFLR$ (per Assumption \ref{assume oracle}), $\COT$ can learn an $L$-layer MLP using in-context sample size $n>\max_{\ell\in[L]}d_\ell$ where $d_\ell$ is the input dimension of $\ell$th layer. This is assuming condition number $\kappa_{\max}$ of the problem is finite as soon as all layers become over-determined. Consequently, $\COT$ needs $\max(k,d)$ sample complexity to learn a two layer MLP. This provides a formal justification for the observation that empirical $\COT$ performance is agnostic to $k$ as long as $k\leq d$.
%(so that following Def.~\ref{MLP assump}, $\Tb_\ell, \ell\in[L]$ are tall matrices and $\kappa_{\max}\to\infty$). It proves that to learn a 2-layer MLP with input dimension $d$ and hidden dimension $k$, $\COT$ needs ${\cal{O}}(\max(k,d))$ in-context samples, and when $k<d$, varying $k$ will not influence the performance (Fig.~\ref{fig:smallk} and Section~\ref{sec:emp}).
\end{remark}}
We provide the concrete filtering statements \yl{based on the transformer architecture} in Appendix~\ref{ss:construction}, and the key components of our construction are the following: (i) Inputs are projected through the embedding layer in which a set of encodings, an enumeration of the tokens ($1,2,\dots, N$), an enumeration of the layers ($1,2,\dots,L$) and an identifier for each layer already predicted are all attached. Notice that this ``modification'' to the input only depends on the sequence length and is agnostic to the token to-be-predicted. This allows for an automated looping over $L$ predictions. (ii) We use this information to extract the sequence length $N$ and the current layer $\ell$ to-be-predicted. (iii) With these at hand, we construct an `if-then' type of function using the ReLU layers to filter out the samples that are not needed for the prediction.

\begin{figure}[!t]
\vspace{-20pt}
\centering
% \subfigure[$\ICL~vs~\EXP~vs~\COT$]{
%     \includegraphics[height=.25\columnwidth]{fig_sec/figs/k=16.pdf}
%     \label{fig:k16}
%     \hspace{-1mm}
% }
\subfigure[$\COT$: $d=10~vs~20$]{
    \includegraphics[height=.25\columnwidth]{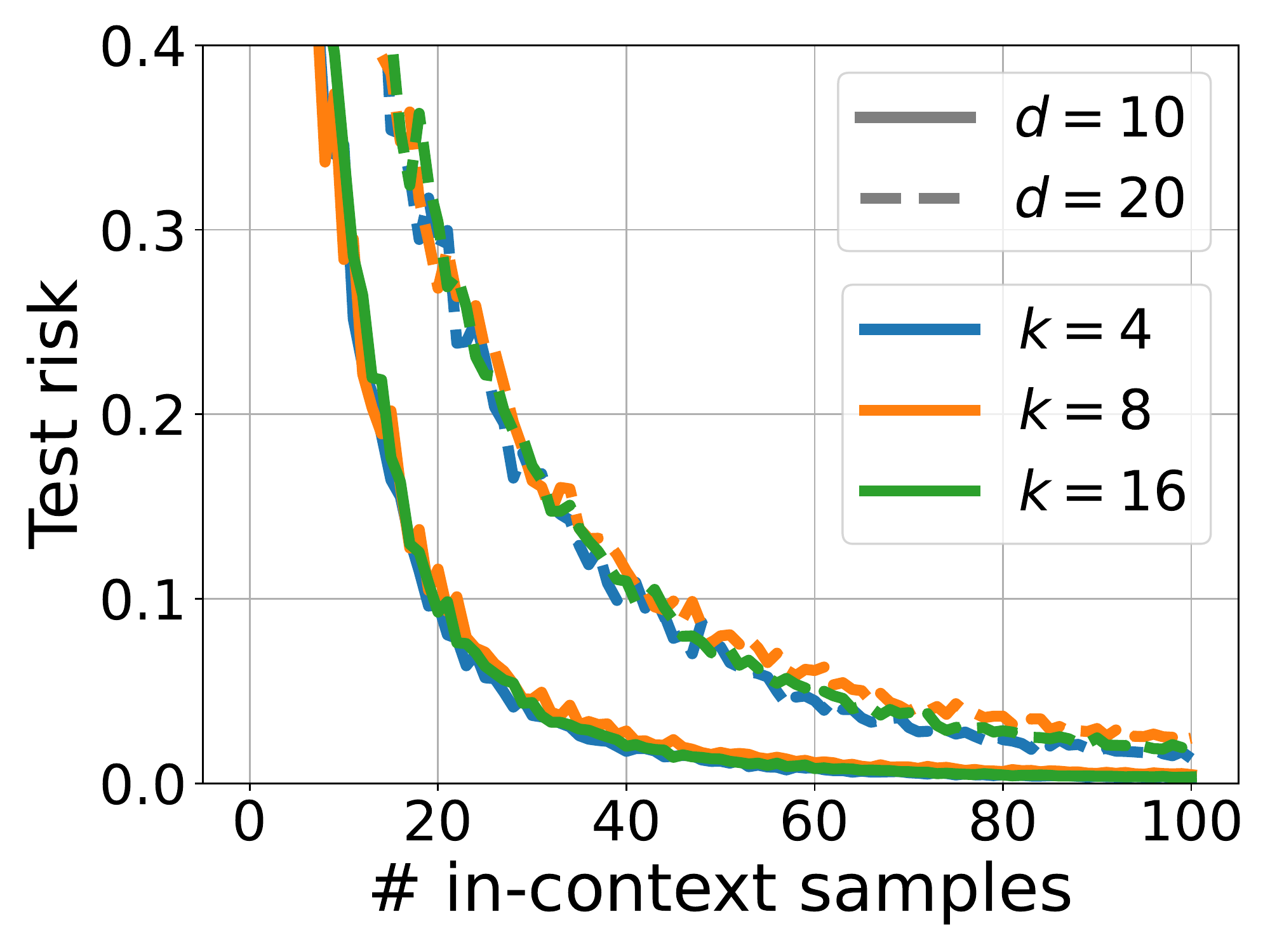}
    \label{fig:cot_d10_vs_d20}
    % \hspace{-1mm}
}
\subfigure[Alignment in Fig.~\ref{fig:cot_d10_vs_d20}]{
    \includegraphics[height=.25\columnwidth]{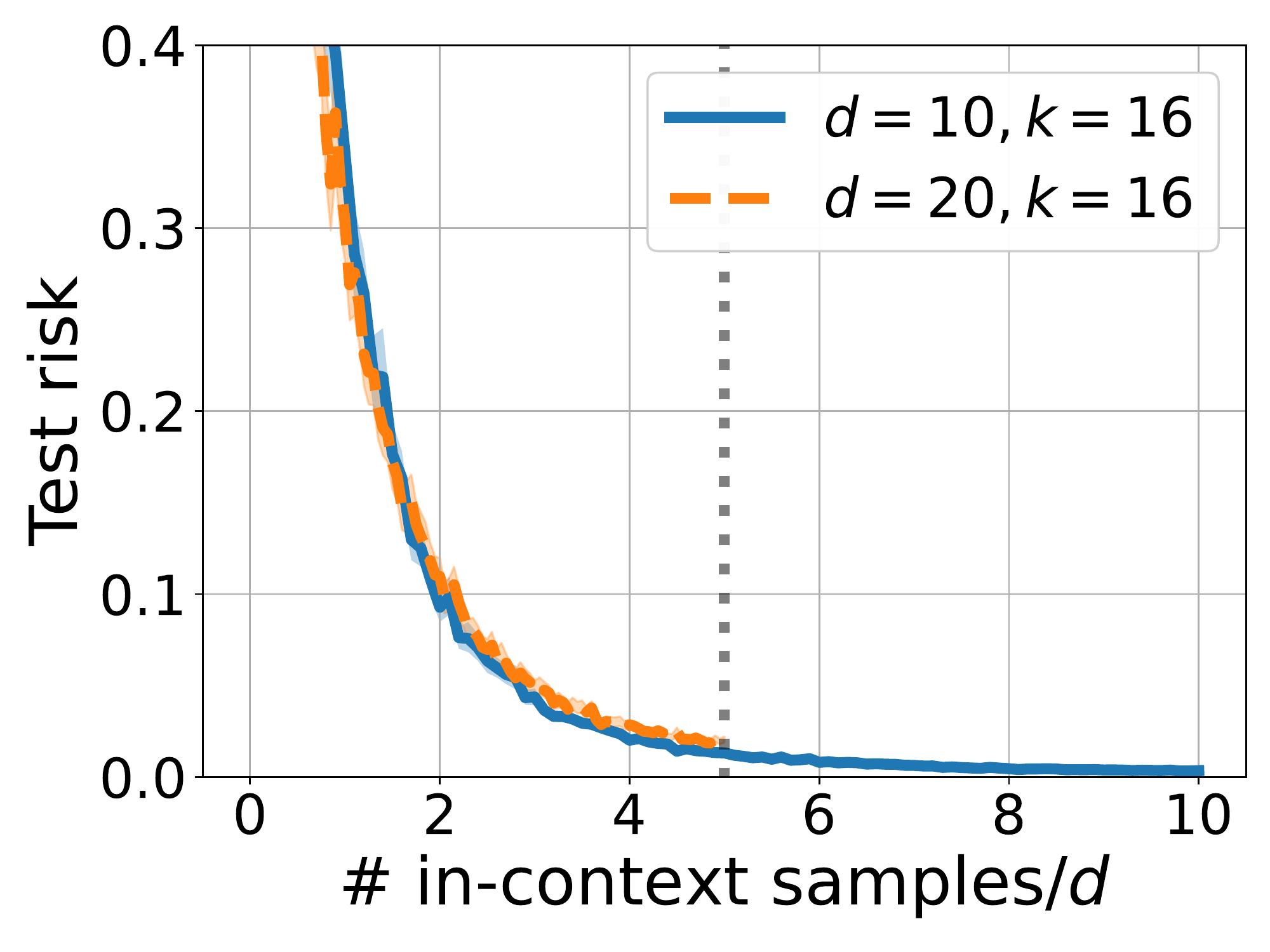}
    \label{fig:cot_d10_vs_d20_align}
    % \hspace{-2mm}
}
\vspace{-2mm}
\caption{Solving 2-layer MLPs with varying input dimension $d$ and hidden neuron size $k$. }\label{fig:smallk}

\centering
\subfigure[$\COT$: composed risk ($d=10$)]{
    \includegraphics[height=.25\columnwidth]{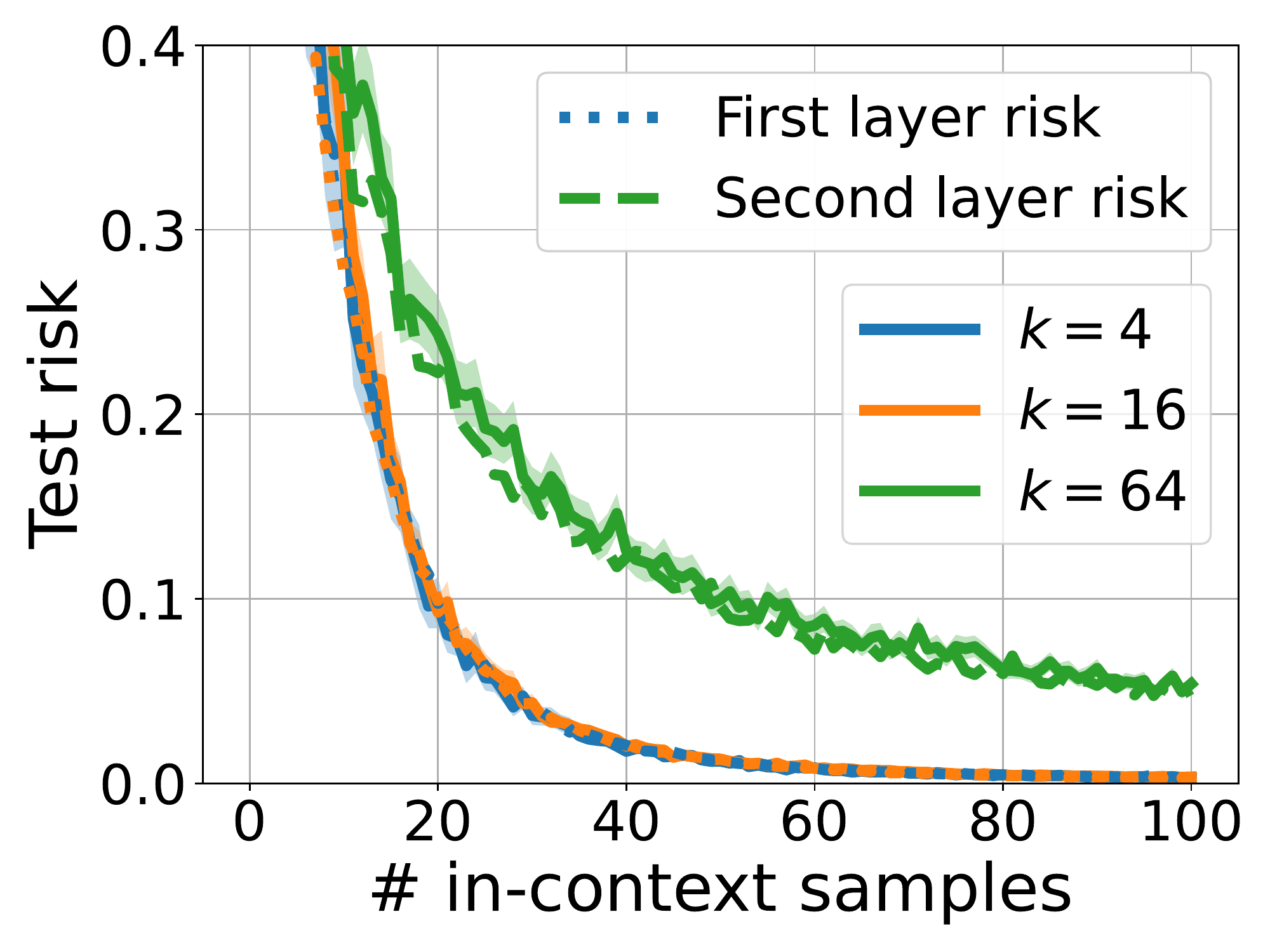}
    \label{fig:largek compose}
    \hspace{-1mm}
}
\subfigure[Risk of first layer]{
    \includegraphics[height=.25\columnwidth,trim={1.5cm 0 0 0},clip]{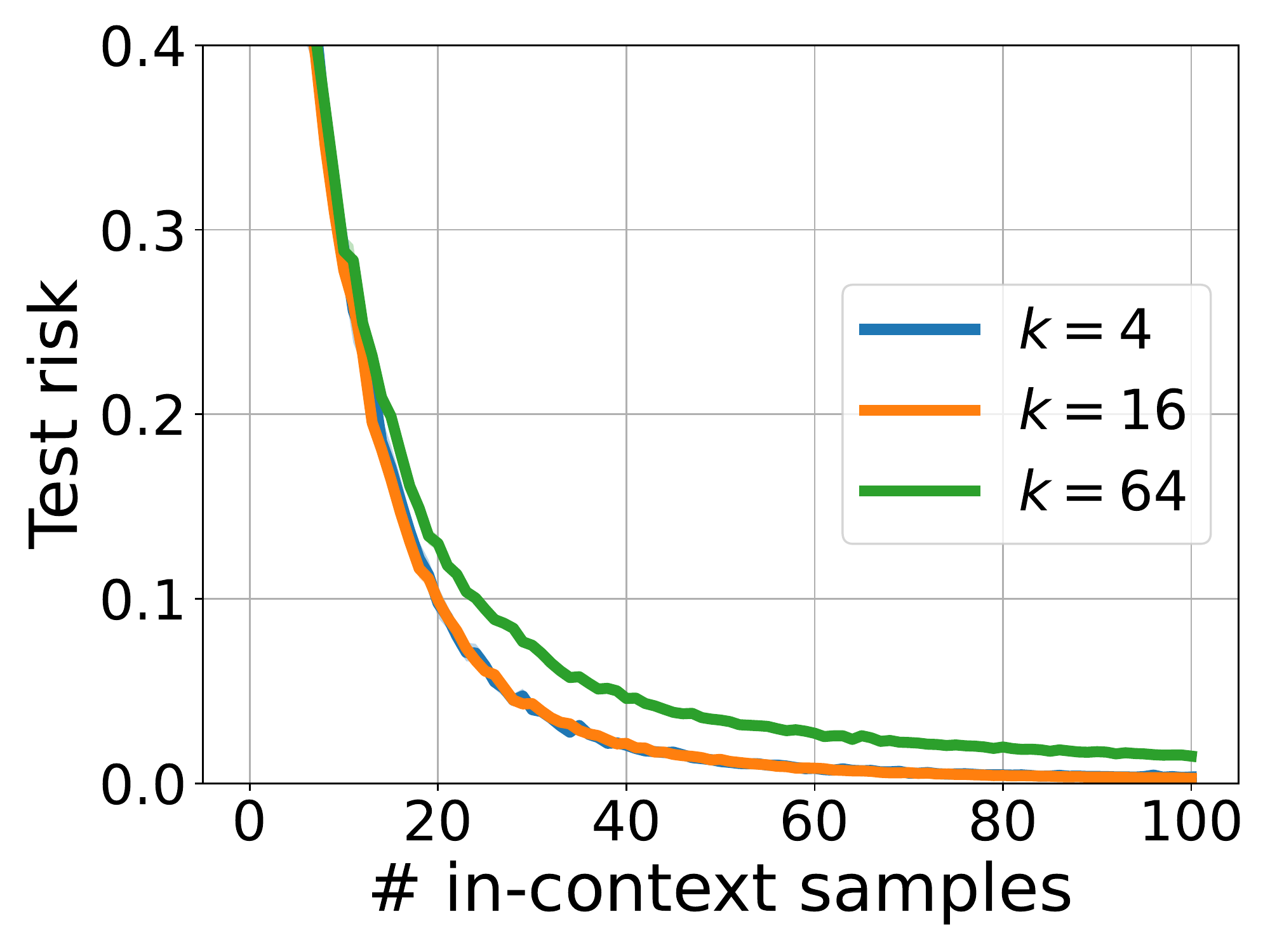}
    \label{fig:largek first}
    \hspace{-1mm}
}
\subfigure[Risk of second layer]{
    \includegraphics[height=.25\columnwidth,trim={1.5cm 0 0 0},clip]{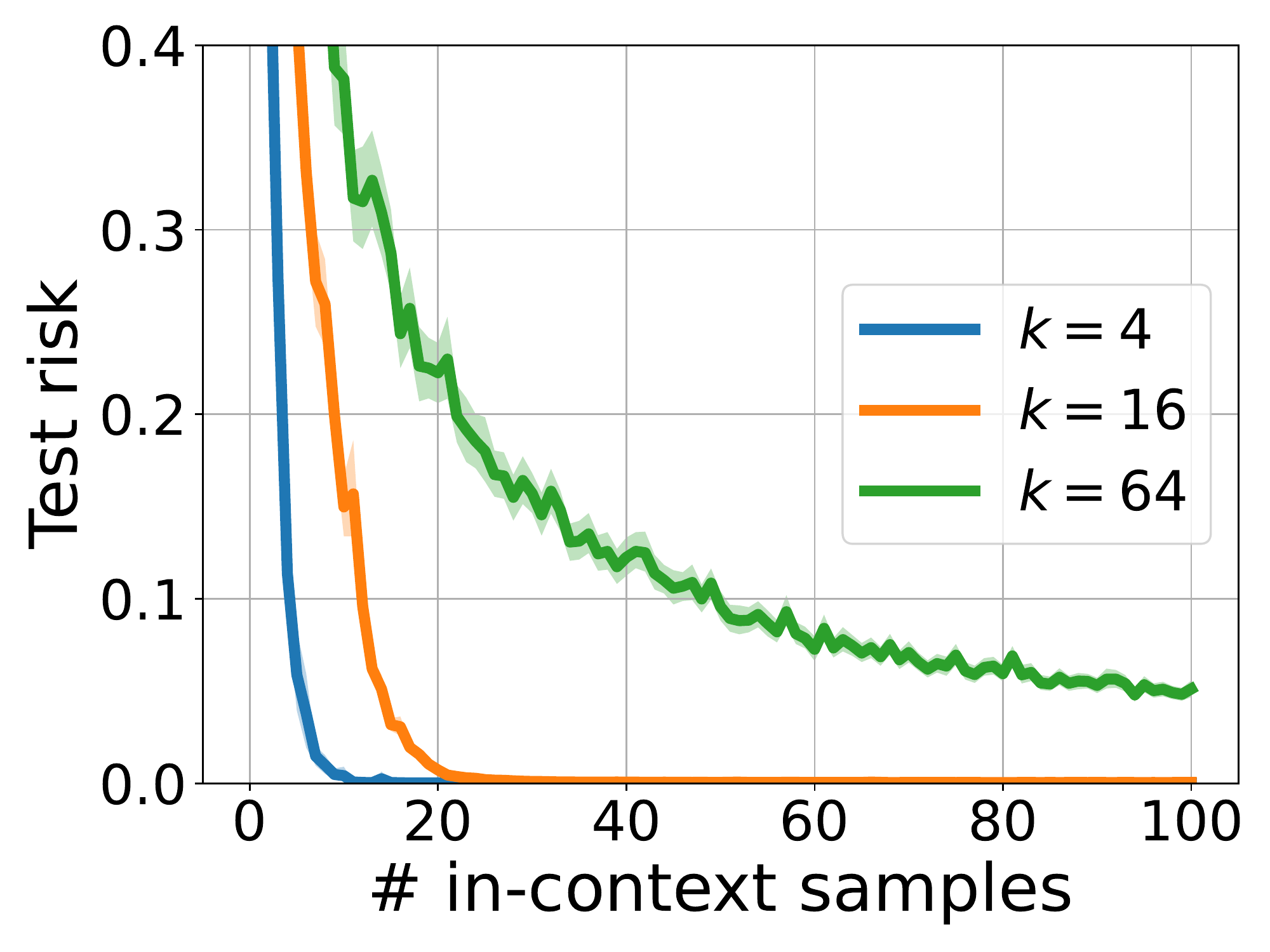}
    \label{fig:largek second}
    \hspace{-1mm}
}
\vspace{-2mm}
\caption{We decouple the composed risk of predicting 2-layer MLPs into risks of individual layers. }
\label{fig:largek}
\vspace{-10pt}
\end{figure}
\section{\yl{Experiments with 2-layer Random MLPs}}\label{sec:exp}
% \section{Experimental Results}\label{sec:exp}
% Our first evaluation concerns the statistical benefit of \AG{add text here}. For a clear exposition we focus on the case of two layer MLPs and compare \AG{text here}. 
% For a clear exposition, we first focus on the case of two layer MLPs which are 2-step tasks. 
\yl{For a clear exposition, we initially focus on two-layer MLPs, which represent 2-step tasks (e.g., $L=2$). We begin by validating Theorem~\ref{thm:MLP} using the $\COT$ method, demonstrating that in-context learning for a 2-layer MLP with $d$ input dimensions and $k$ hidden neurons requires $O(\max(d,k))$ samples. The results are presented in Section~\ref{sec:emp}. Subsequently, in Section~\ref{sec:2nn}, we compare three different methods: $\ICL$, $\EXP$, and $\COT$. The empirical evidence highlights the advantages of $\COT$, showcasing its ability to reduce sample complexity and enhance model expressivity.  }

% and 
% , whose prompt and prediction data structures have been discussed in Section~\ref{sec prelim}. 
% Comparison results are shown in Figures~\ref{fig:diff gpt} and \ref{fig:diff k}.

\textbf{Dataset.} Consider 2-layer MLPs with input $\x\in\R^d$, hidden feature (step-1 output) $\tb\in\R^k$, and output $y\in\R$. Here,  $\tb=f_1(\x):=(\W\x)_+$ and $y=f_2(\tb):=\vv^\top\tb$, with $\W\in\R^{k\times d}$ and $\vv\in\R^{k}$ being the parameters of the first and second layer, and $(x)_+=\max(x,0)$ being {ReLU} activation. The function is composed as $y=\vv^\top(\W\x)_+$. We define the function distributions as follows: each entry of $\W$ is sampled via $\W_{ij}\sim\Nc(0,\frac{2}{k})$, and $\vv\sim\Nc(0,\Iden_k)$, with inputs being randomly sampled through $\x\sim\Nc(0,\Iden_d)$\footnote{Following this strategy for data generation, the expected norms of $\x$, $\tb$ and $y$ are equivalent, and the risk curves displayed in the figures are normalized for comparison.}. We apply the quadratic loss in our experiments. To avoid the implicit bias due to distribution shift, both training and test datasets are generated following the same strategy.

% \section{Empirical and Theoretical Perspectives on CoT}\label{sec:main}
% \vspace{-5pt}
% In this section, we begin by examining the performance of $\COT$ when learning 2-layer MLPs with input dimension of $d$ and hidden dimension of $k$. Our experimentation indicates that $\COT$ necessitates only $O(\max(d,k))$ in-context samples. 
% Subsequently, in Section~\ref{sec:theory}, we present our theoretical findings that demonstrate how $\COT$ can execute filtering over the CoT prompt, thereby learning a 2-layer MLP, akin to resolving $k$ $d$-dimensional ReLU problems and $1$ $k$-dimensional linear regression problem.
% \vspace{-5pt}
% \newcommand{\COT}{\text{CoT-I/O}\xspace}

\subsection{\yl{Empirical Evaluation of CoT-I/O Performance}}\label{sec:emp}
% \vspace{-5pt}
To investigate how MLP task impacts $\COT$ performance, we train 2-layer MLPs with varying input dimensions ($d$) and hidden layer sizes ($k$). The results are presented in Figures~\ref{fig:smallk} and \ref{fig:largek},
% . Here, $\x,\tb,y$ represent input, hidden state, and output respectively. Detailed information on the implementation is deferred to Section~\ref{sec:2nn}.
and all experiments utilize small GPT-2 models for training\footnote{Our code is available at \url{https://github.com/yingcong-li/Dissecting-CoT}.}.

\textbf{$\COT$ performance is agnostic to $k$ when $k\leq d$ (Figure~\ref{fig:smallk}). } In Fig.~\ref{fig:cot_d10_vs_d20}, we train MLPs with $d=10,20$ and $k=4,8,16$. Solid and dashed curves represent the $\COT$ test risk of $d=10$ and $20$ respectively for varying in-context samples. The results indicate that an increase in $d$ will amplifies the number of samples needed for in in-context learning, while the performance remains unaffected by changes in $k\in\{4,8,16\}$. To further scrutinize the impact of $d$ on $\COT$ accuracy, in Fig.~\ref{fig:cot_d10_vs_d20_align}, we adjust the horizontal axis by dividing it by the input dimension $d$, and superimpose both $d=10,k=16$ (blue solid) and $d=20,k=16$ (orange dashed) results. This alignment of the two curves implies that the in-context sample complexity of $\COT$ is linearly dependent on $d$. 

% $\bullet$ \textbf{$\COT$ requires $\mathbf{N=O(\max(d,k))}$ to learn 2-layer MLPs (Figures~\ref{fig:smallk}\&\ref{fig:largek}). } \red{As stated in Theorem~\ref{?}, to solve 2-layer MLPs with input dimension $d$ and hidden size $k$, $\COT$ requires $O(\max(d,k))$ in-context samples.} 
% Then the $\COT$ averaged test risks keeping the same in Fig.~\ref{fig:avg risk k} may be from the fact that under the given MLP setting, the input dimension $d$ dominates the in-context sample complexity, and since we fix $d$ to be $10$, the performance keeps unchangeable. To further investigate how $d$ and $k$ affect the $\COT$ performance, we either enlarge input dimension $d$ or hidden dimension $k$, and results are respectively displayed in Figure~\ref{fig:smallk} and Figure~\ref{fig:largek}. In Fig.~\ref{fig:cot_d10_vs_d20}, except the $\COT$ results already presented in Figure~\ref{fig:diff k}, we train the same small GPT-2 but with input dimension $d=20$. Solid and dashed curves represent the point-to-point $\COT$ test risk of $d=10$ and $20$ respectively. Results show that enlarging $d$ will increase the samples needed in in-context learning, and varying $k$ will not influence the test performance. In Fig.~\ref{fig:cot_d10_vs_d20_align}, we rescale the horizontal axis by dividing with input dimension $d$, and put both $d=10$ (blue solid) and $20$ (orange dashed) results on it. The two curves are aligned, which verifies the linear dependence of in-context sample complexity on $d$. 

\textbf{Large $k$ dominates $\COT$ performance (Figure~\ref{fig:largek}). } We further investigate the circumstances under which $k$ begins to govern the $\COT$ performance. In Figure~\ref{fig:largek compose}, we replicate the same experiments with $d=10$, but train with wider MLPs ($k=64$). Blue, orange and green curves represent results for $k=4,16,64$ respectively. Since the hidden dimension $k=64$ is larger, learning the second layer requires more hidden features ($\tb$), thus $N=100$ in-context samples (providing $100$ $\tb$s) are insufficient to fully restore the second layer, leading to performance gaps between $k=4,16$ and $k=64$. 
% hidden dimension $k$ turns to dominate the $\COT$ performance and there is performance gap between $k=4,16$ and $k=64$. 
To quantify the existing gaps, we conduct single-step evaluations for both the first and the second layers, with the results shown in Figures~\ref{fig:largek first} and \ref{fig:largek second}. Specifically, let $\pmt_n(\tilde f)$ be a test prompt containing $n$ in-context samples where $\tilde f$ represents any arbitrary 2-layer MLP. Given a test sample $(\xtest,\tb_{\text{\tiny{test}}},y_{\text{\tiny{test}}})$, the layer predictions are performed as follows.
\begin{align}
\text{1st layer prediction: } &\TF(\pmt_n(\tilde f),\xtest):=\hat\tb,\nn\\
\text{2nd layer prediction: } &\TF(\pmt_n(\tilde f),\xtest,\tb_{\text{\tiny{test}}}):=\hat y.\nn
\end{align}
The test risks are calculated by $\|\hat\tb-\tb_{\text{\tiny{test}}}\|^2$ and $(\hat y-y_{\text{\tiny{test}}})^2$. The risks illustrated in the figures are normalized for comparability (refer to the appendix for more details). Evidence from Fig.~\ref{fig:largek first} and \ref{fig:largek second} shows that while increasing $k$ does not affect the first layer's prediction, it does augment the number of samples required to learn the second layer. Moreover, by plotting the first layer risks of $k=4,16$  (blue/orange dotted) and second layer risk of $k=64$ (green dashed) in Fig.~\ref{fig:largek compose}, we can see that they align with the $\COT$ composed risks. This substantiates the hypothesis that $\COT$ learns 2-layer MLP through compositional learning of separate layers.
% , and its performance is determined by aggregate risks over all layers. 

% The interpretation is that, $\COT$ treats the 2-layer MLP as $k$ $d$-dimensional $\texttt{ReLU}$ and $1$ $k$-dimensional linear regression problems and solves them in parallel. Therefore, the function complexities of first and second layers are $O(d)$ and $O(k)$ separately, and then it answers why only $N=O(\max(d,k))$ samples are required for $\COT$ to learn 2-layer MLPs.   
\begin{figure}[!t]
\centering
\hspace{-6mm}
\subfigure[Averaged risk]{
    \includegraphics[height=.2\columnwidth]{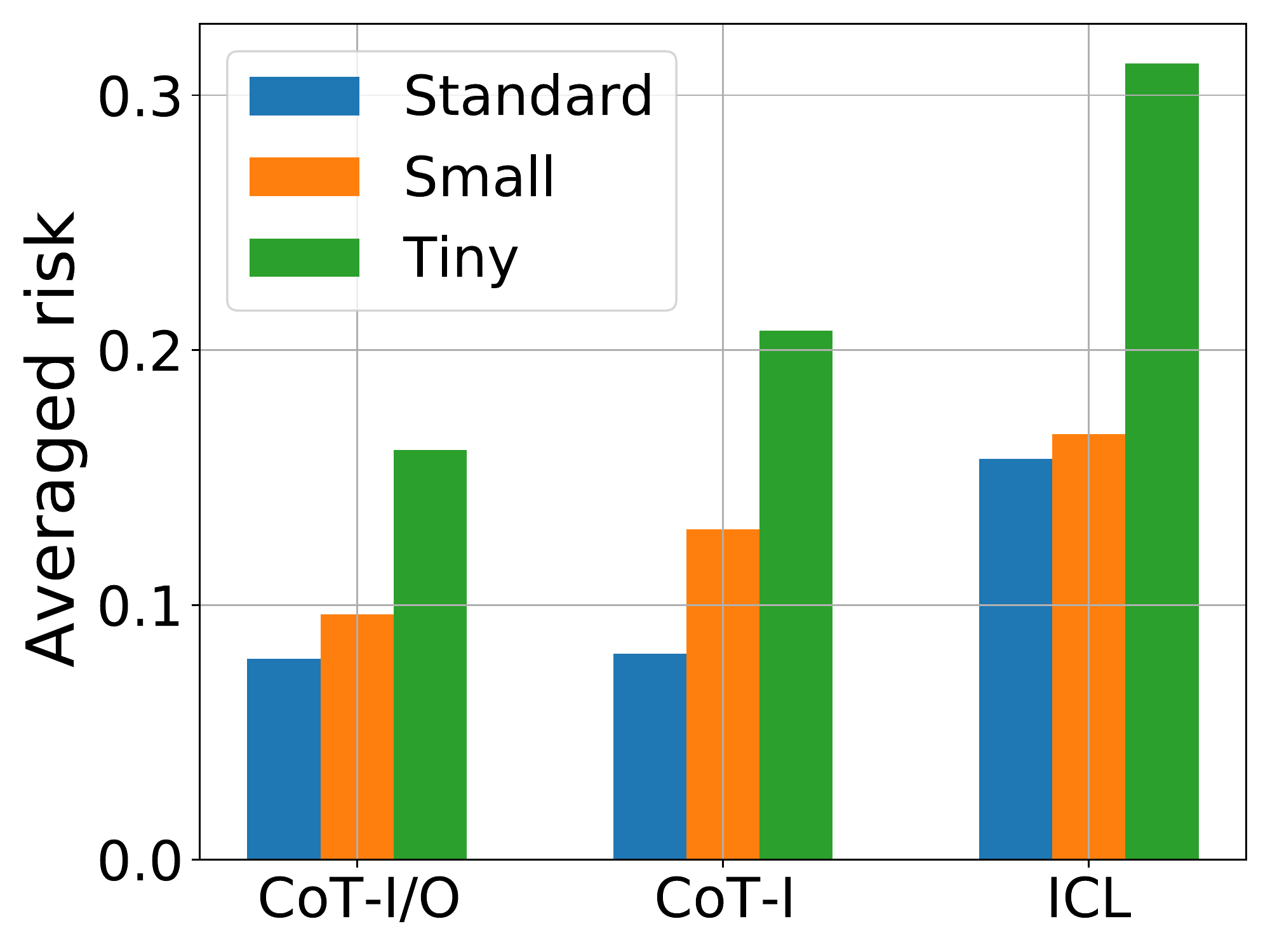}
    \label{fig:avg risk model}
    \hspace{-3mm}
}
\subfigure[Standard GPT-2]{
    \includegraphics[height=.2\columnwidth]{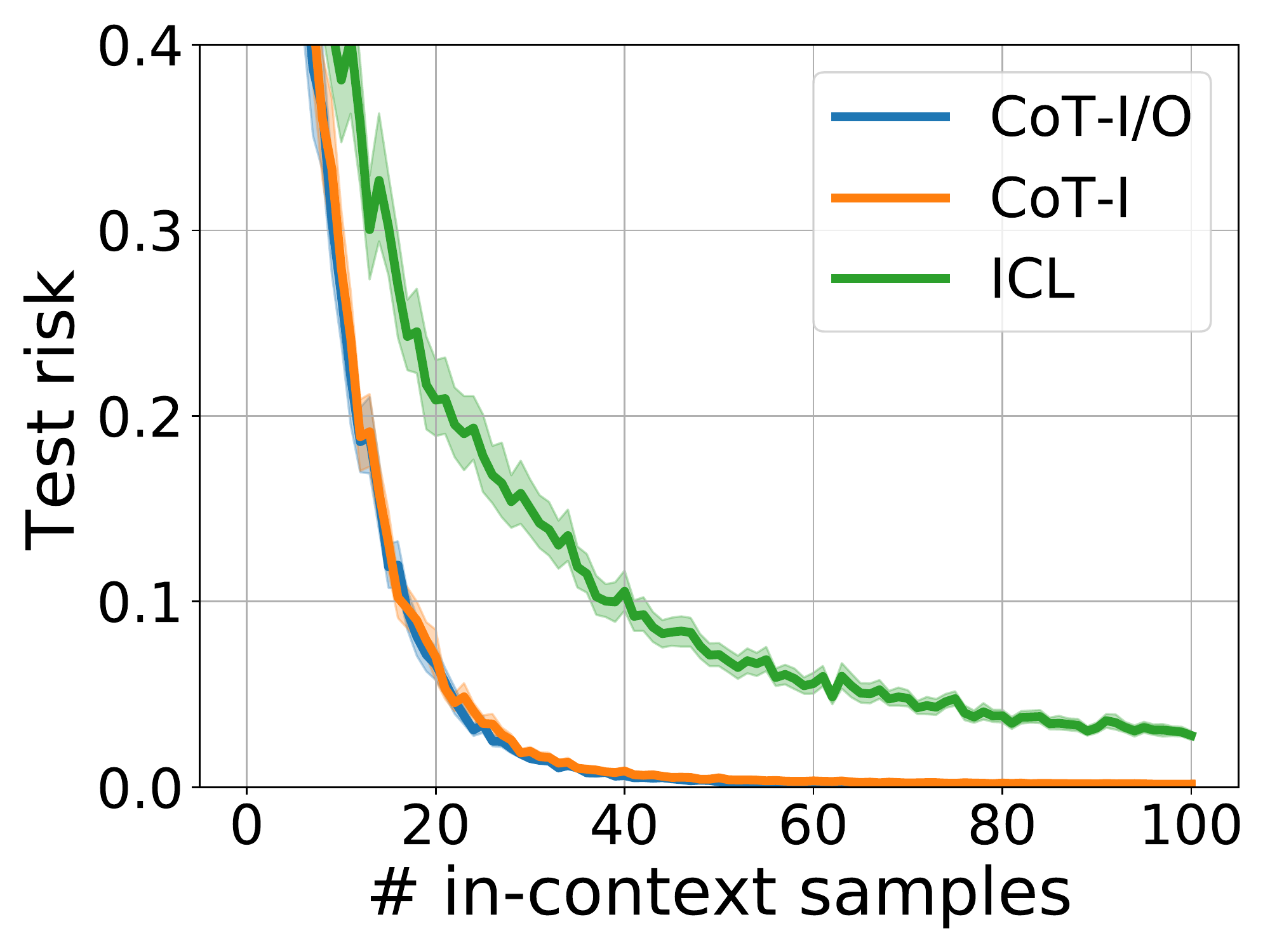}
    \label{fig:gpt standard}
    \hspace{-3mm}
}
\subfigure[Small GPT-2]{
    \includegraphics[height=.2\columnwidth,trim={1.5cm 0 0 0},clip]{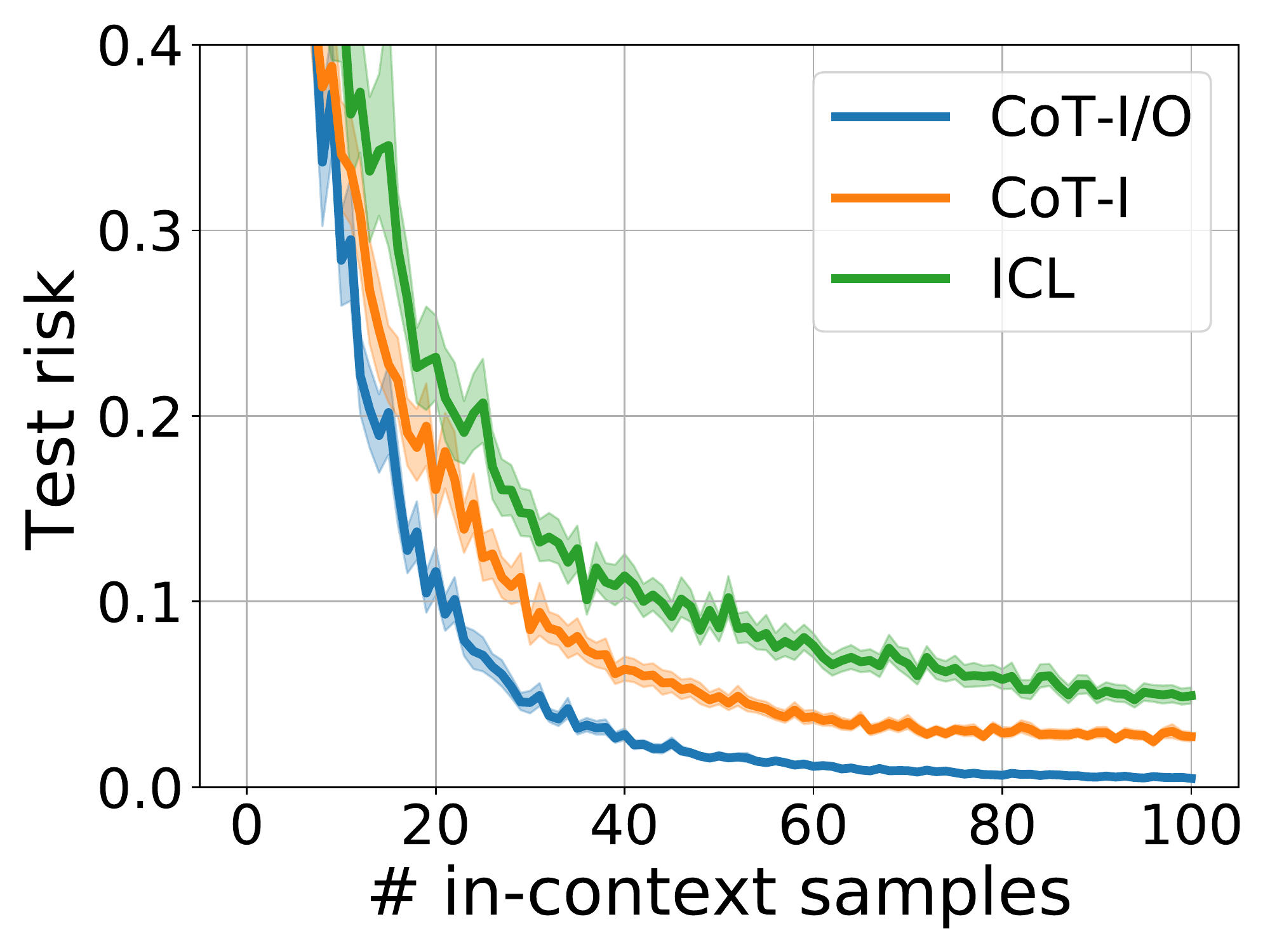}
    \label{fig:gpt small}
    \hspace{-3mm}
}
\subfigure[Tiny GPT-2]{
    \includegraphics[height=.2\columnwidth,trim={1.5cm 0 0 0},clip]{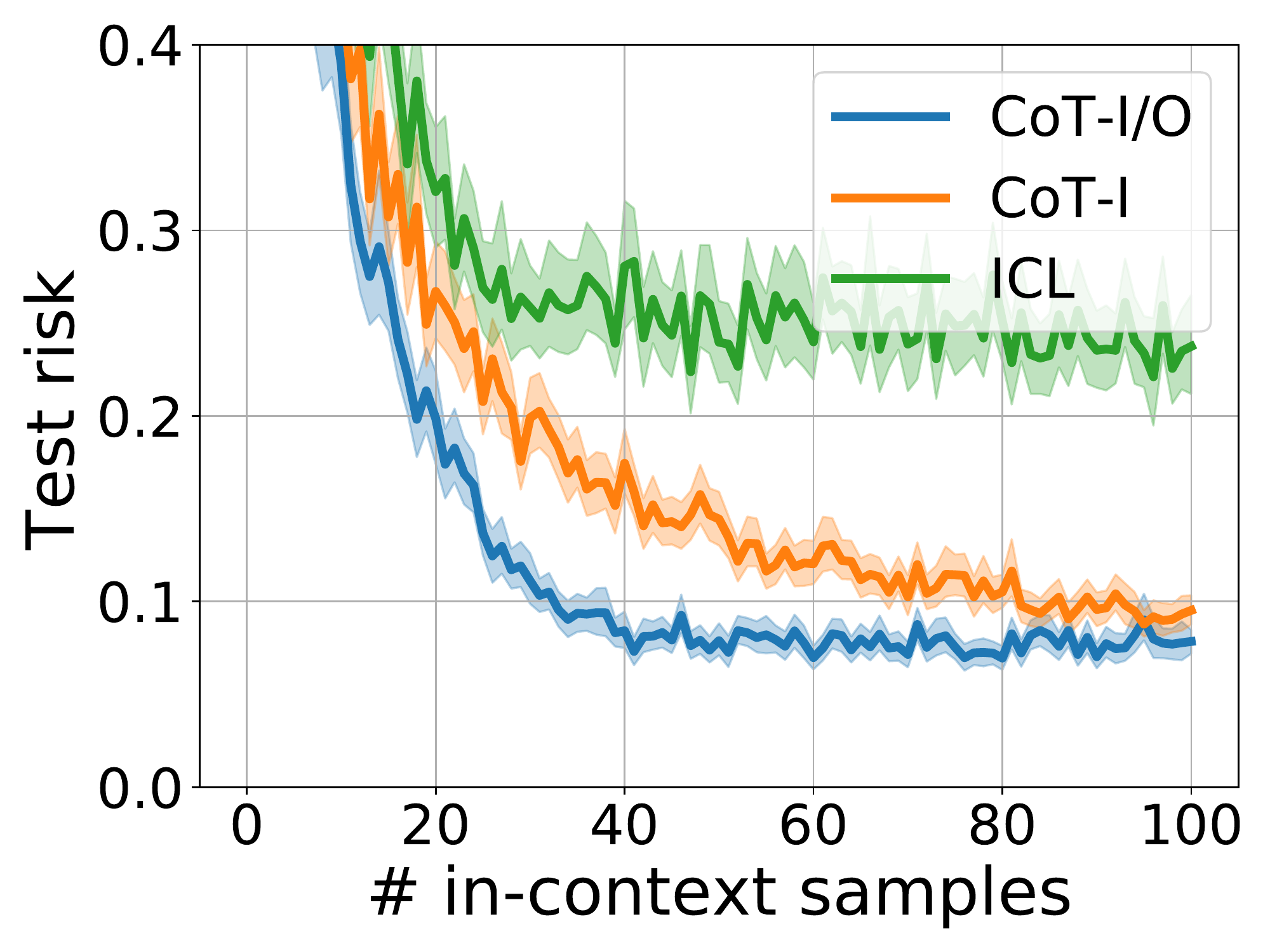}
    \label{fig:gpt tiny}
    \hspace{-3mm}
}
\vspace{-2mm}
\caption{Comparison of the three methods for solving $2$-layer MLPs using different GPT-2 models.}
\label{fig:diff gpt}
\vspace{-10pt}
\end{figure}

% \begin{figure}
%     \centering
%     \includegraphics[width=\linewidth]{fig_sec/figs/diff_gpt}
%     \caption{Compare the three methods with training $2$-layer neural nets with different GPT-2 models.}
%     \label{fig:diff gpt}
% \end{figure}
\begin{figure}[!t]
\centering
\hspace{-6mm}
\subfigure[Averaged risk]{
    \includegraphics[height=.2\columnwidth]{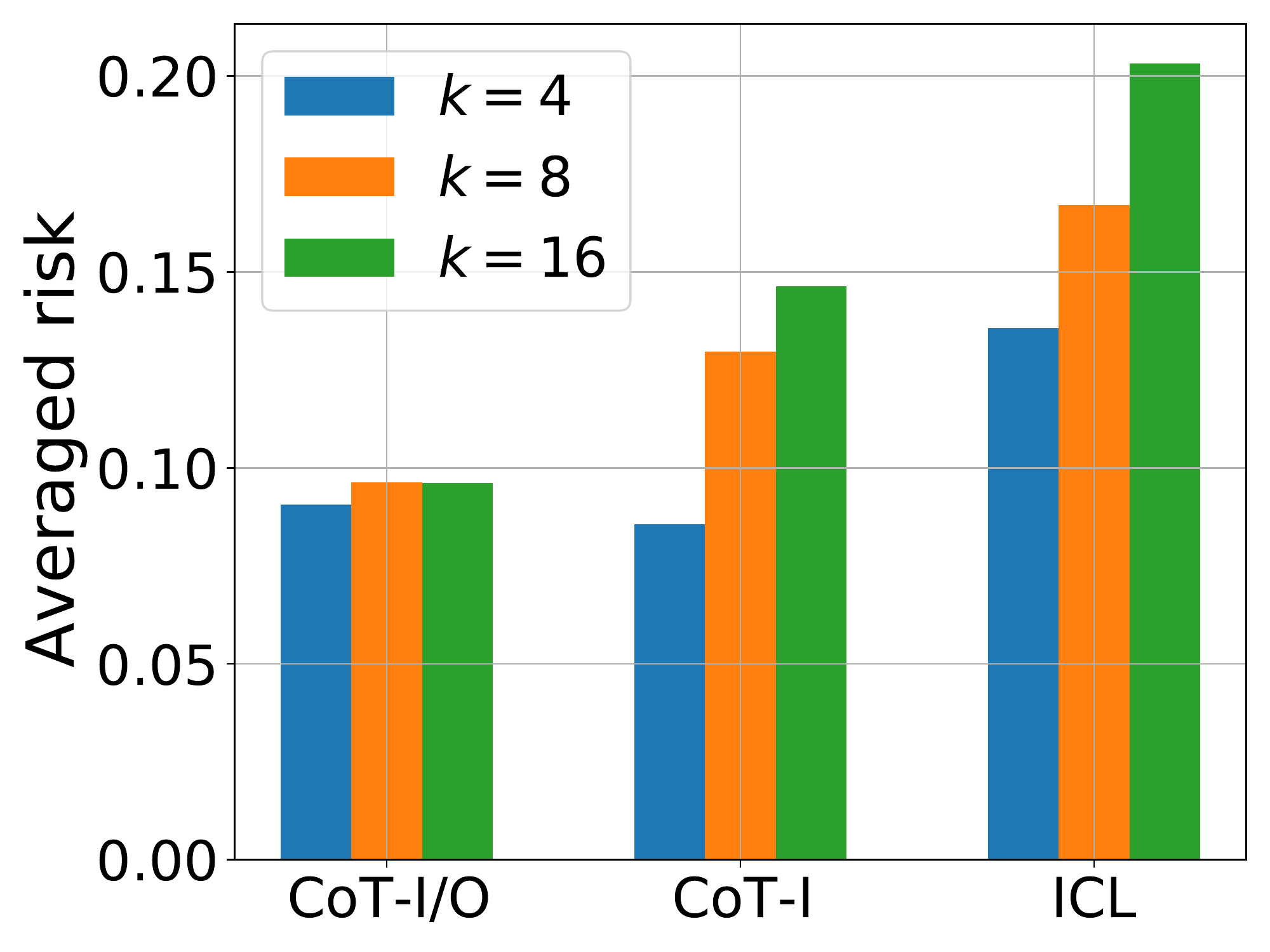}
    \label{fig:avg risk k}
    \hspace{-3mm}
}
\subfigure[$k=4$]{
    \includegraphics[height=.2\columnwidth]{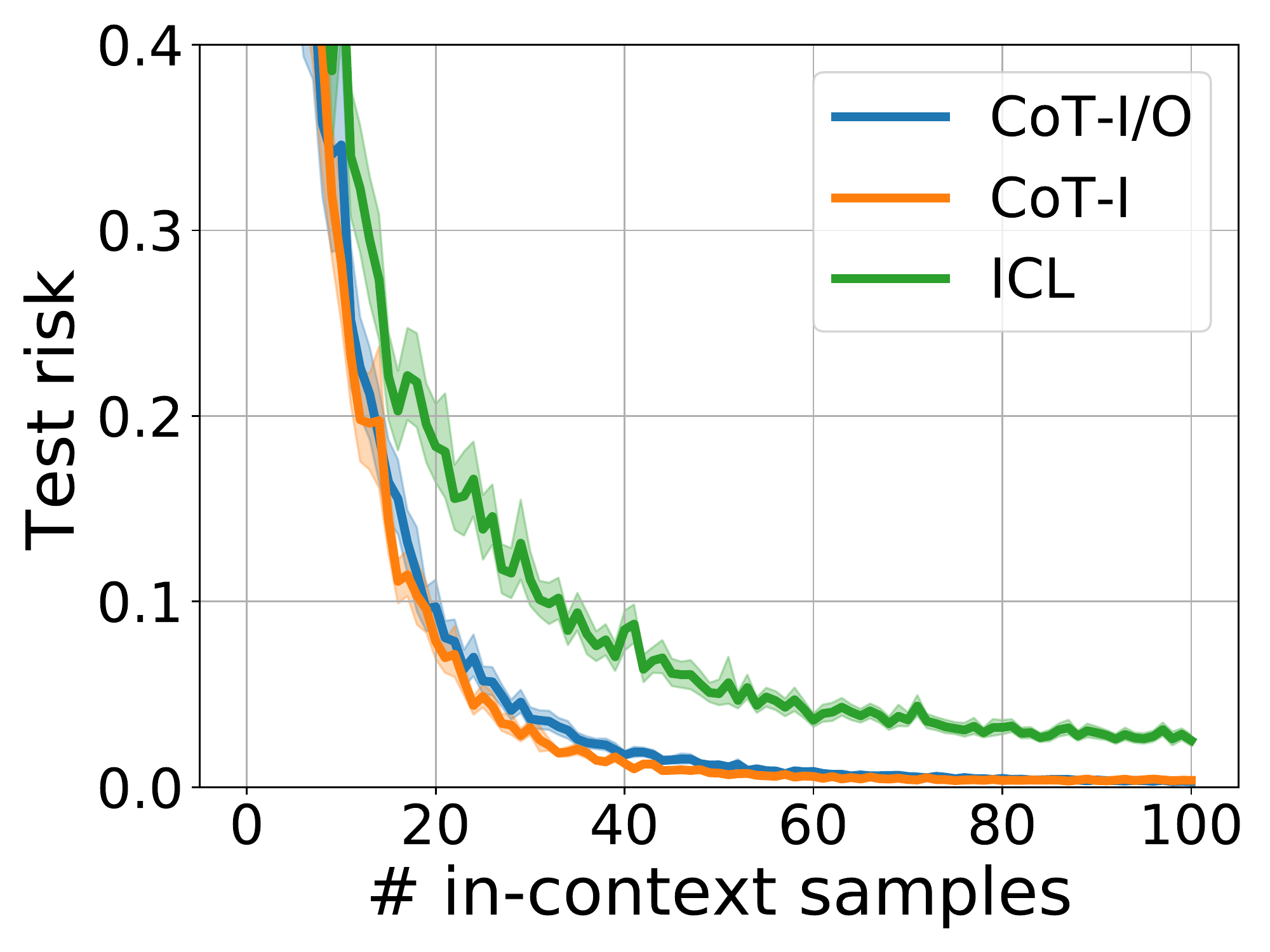}
    \label{fig:k4}
    \hspace{-3mm}
}
\subfigure[$k=8$]{
    \includegraphics[height=.2\columnwidth,trim={1.5cm 0 0 0},clip]{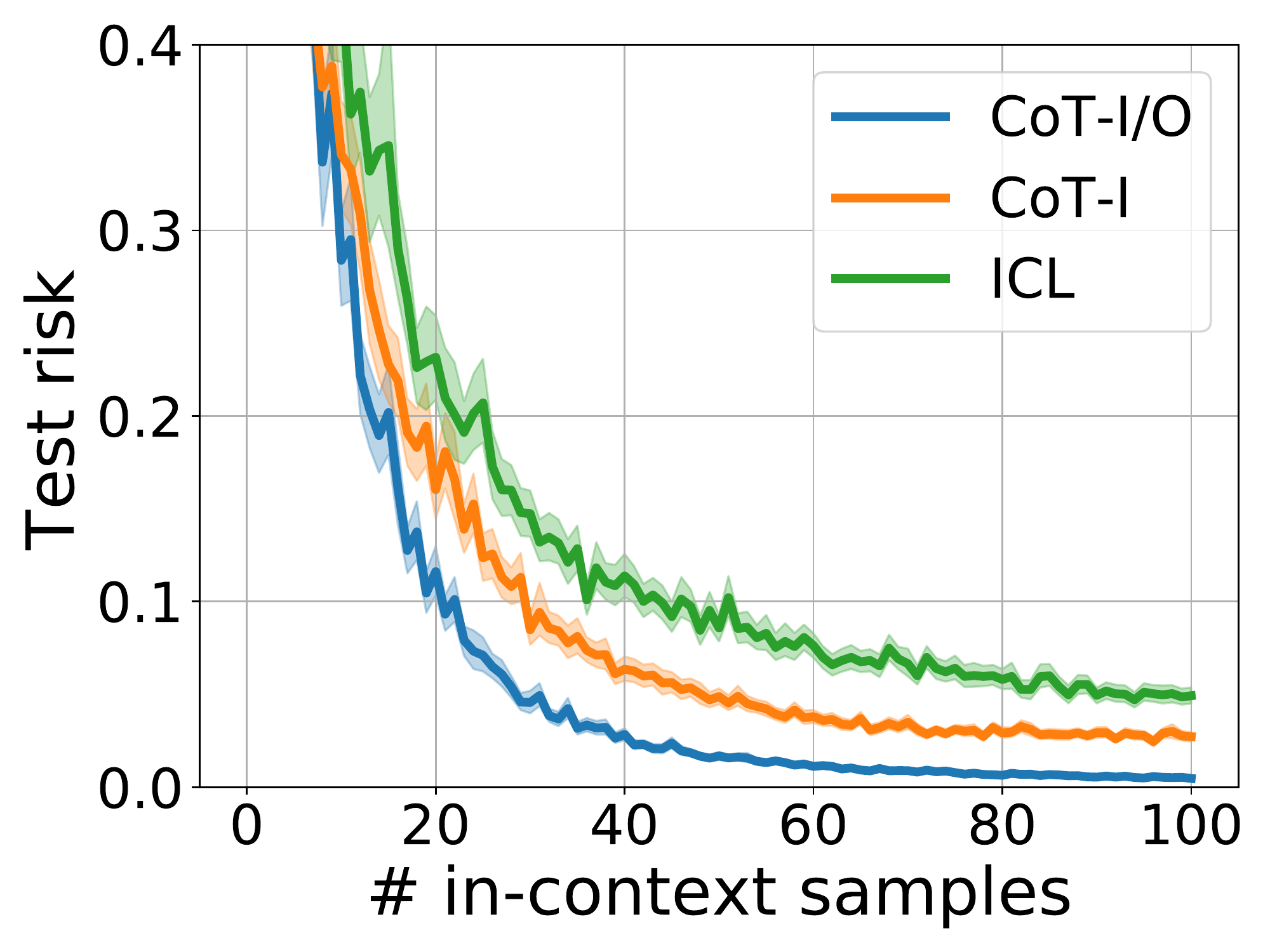}
    \label{fig:k8}
    \hspace{-3mm}
}
\subfigure[$k=16$]{
    \includegraphics[height=.2\columnwidth,trim={1.5cm 0 0 0},clip]{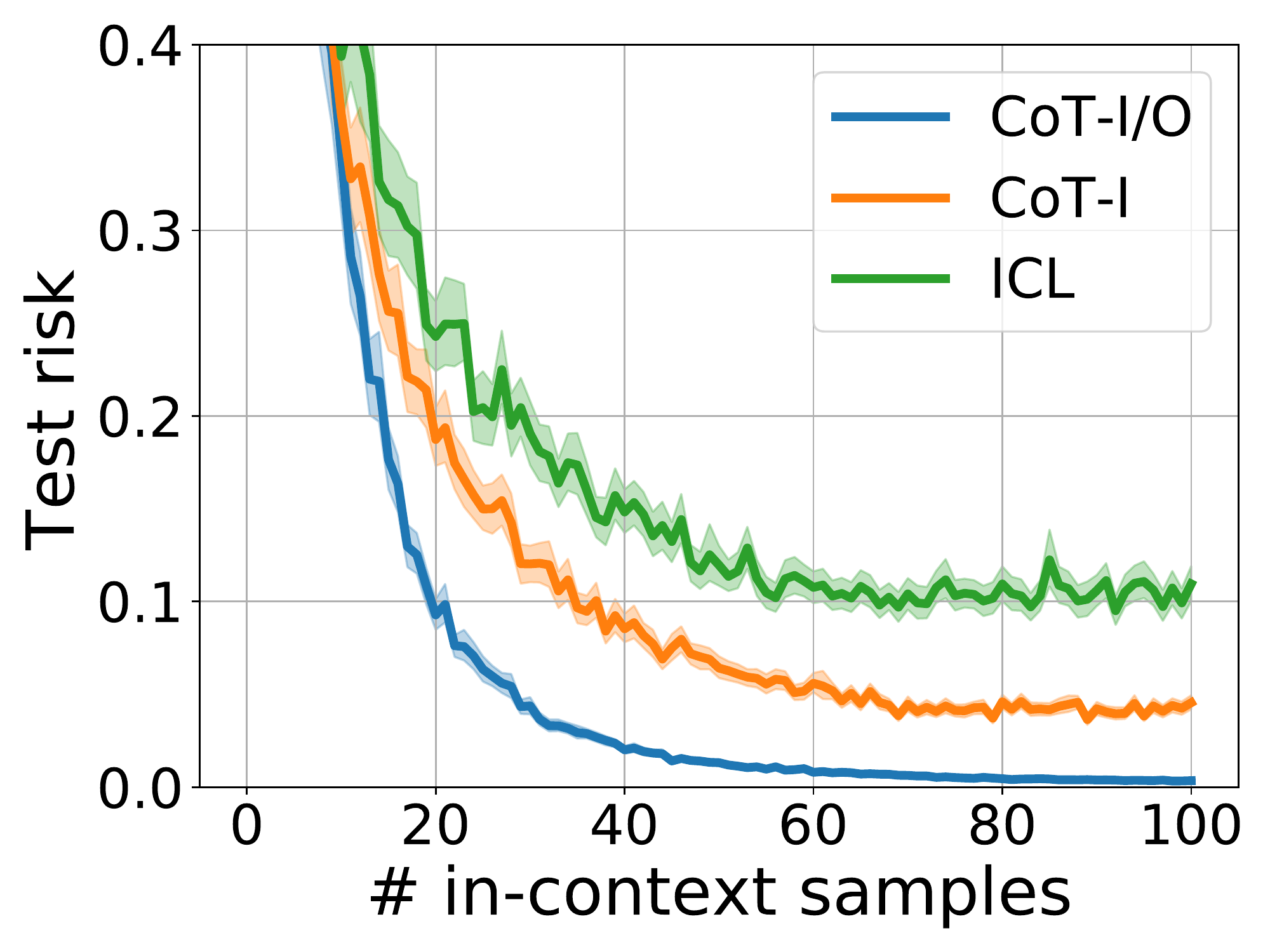}
    \label{fig:k16}
    \hspace{-3mm}
}

\vspace{-2mm}
\caption{Comparison of the three methods for solving $2$-layer MLPs with different hidden sizes.}
\label{fig:diff k}
\vspace{-10pt}
\end{figure}

% \begin{figure}
%     \centering    \includegraphics[width=\linewidth]{fig_sec/figs/diff_k}
%     \caption{Compare the three methods when training $2$-layer neural nets with different hidden sizes.}
%     \label{fig:diff k}
% \end{figure}

% \newcommand{\COT}{\text{CoT-I/O}\xspace}
% \newcommand{\EXP}{\text{CoT-I}\xspace}
% \newcommand{\ICL}{\text{ICL}\xspace}

\subsection{Comparative Analysis of ICL, CoT-I and CoT-I/O}\label{sec:2nn}

% \red{Additionally, we also investigate the exact in-context samples needed for $\COT$ to learn the 2-layer MLPs by varying the input dimension ($d$) and hidden neuron number ($k$). The phenomenon is shown in Figures~\ref{fig:smallk} and \ref{fig:largek}, which empirically verifies Theorem~\ref{}.
% }
% experiments by comparing three different methods: $\ICL$, $\EXP$ and $\COT$. Here $\ICL$ follows the same setting as in \cite{garg2022can}, where prompts are given by input-output pairs (also see \eqref{icl pmt}). While $\EXP$ and $\COT$ admit chain-of-thought prompting and $\COT$ does chain-of-thought predictions as well (see \eqref{exp pmt} for $\EXP$ and \eqref{cot pmt} for $\COT$). 

% While the difference between  and Vanilla is that, Explanation has prompt formed by input-instruction-output. Therefore, Transformer can learn more by following the step-by-step messages. In addition to taking all step-messages as input, CoT also trains to predict each step so that the model performs as expected. 

% and $\W$ and $\vv$ can also been seen as the 

% where $\W$, $\vv$ are the respective function parameters following i.i.d. Gaussian distribution, specifically, each entry of $\W$ is sampled via $\W_{ij}\sim\Nc(0,\frac{2}{k})$, and $\vv\sim\Nc(0,\Iden_k)$. 

% $\bullet$ \textbf{$\mathbf{\ICL~vs~\EXP~vs~\COT}$ (Figures~\ref{fig:diff gpt}\&\ref{fig:diff k}).} 
\textbf{Varying model sizes (Figure~\ref{fig:diff gpt}).} We initially assess the benefits of $\COT$ over $\ICL$ and $\EXP$ across different $\TF$ models. With $d=10$ and $k=8$ fixed, we train three different GPT-2 models: standard, small and tiny GPT-2. The small GPT-2 has $6$ layers, $4$ attention heads per layer and $128$ dimensional embeddings. The standard GPT-2 consists of twice the number of layers, attention heads and embedding dimensionality compared to the small GPT-2, and tiny GPT-2, on the other hand, possesses only half of these hyperparameters compared to the small GPT-2. We evaluate the performance using prompts containing $n$ in-context samples, where $n$ ranges from $1$ to $N$ ($N=100$). The associated test risks are displayed in Figs.~\ref{fig:gpt standard},~\ref{fig:gpt small} and \ref{fig:gpt tiny}. The blue, orange and green curves correspond to $\COT$, $\EXP$ and $\ICL$, respectively. In Fig.~\ref{fig:avg risk model}, we present the averaged risks. The results show that using $\COT$, the small GPT-2 can solve 2-layer MLPs with approximately $60$ samples, while $\EXP$ requires the standard GPT-2. Conversely, $\ICL$ is unable to achieve zero test risk even with the standard GPT-2 model and up to $100$ samples. This indicates that to learn 2-layer MLPs in a single shot, $\ICL$ requires at least ${\cal{O}}(dk+d)$ samples to restore all function parameters. Conversely, $\EXP$ and $\COT$ can leverage implicit samples contained in the CoT prompt. Let $f_1\in\Fc_1$ (first layer) and $f_2\in\Fc_2$ (second layer). By comparing the performance of $\EXP$ and $\COT$, it becomes evident that the standard GPT-2 is capable of learning the composed function $f=f_2\circ f_1\in\Fc$, which the small GPT-2 cannot express. 
% However, it can express the union of the function sets $\Fc_1\bigcup\Fc_2$. 

% The results show that $\COT$ achieves the best performance and larger model learns to solve 2-layer MLPs better. 

% Results show that small model is enough to solve the problem for $\COT$ (Fig.~\ref{fig:diff gpt cot}), while $\EXP$ requires standard GPT-2 model (Fig.~\ref{fig:diff gpt exp}), and $\ICL$ is not able to learn the model recovering the neural networks. It is because that to learn the problem in one shot, the model should be enough to represent all possible 2-layer neural networks and additionally, needs at least $O(dk+d)$ samples.

\textbf{Varying MLP widths (Figure~\ref{fig:diff k}). } Next, we explore how different MLP widths impact the performance (by varying the hidden neuron size $k\in\{4,8,16\}$). The corresponding results are depicted in Figure~\ref{fig:diff k}. The blue, orange and green curves in Fig.~\ref{fig:k4}, \ref{fig:k8} and \ref{fig:k16} correspond to hidden layer sizes of $k=4$, $8$, and $16$, respectively. Fig.~\ref{fig:avg risk k} displays the averaged risks. We keep $d=10,~N=100$ fixed and train with the small GPT-2 model. As discussed in Section~\ref{sec:emp}, $\COT$ can learn a 2-layer MLP using around $60$ samples for all $k=4,8,16$ due to its capability to deconstruct composed functions. However, $\EXP$ can only learn the narrow MLPs with $k=4$, and $\ICL$ is unable to learn any of them. Moreover, we observe a substantial difference in the performance of $\ICL$ and $\EXP$ with varying $k$ (e.g., see averaged risks in Fig.~\ref{fig:avg risk k}). This can be explained by the fact that enlarging $k$ results in more complex $\Fc_1$ and $\Fc_2$, thus making the learning of  $\Fc=\Fc_2\times\Fc_1$ more challenging for $\ICL$ and $\EXP$. 

\begin{figure}[!t]
\vspace{-35pt}
\centering
\subfigure[Point-to-point meta prediction]{
    \includegraphics[height=.25\columnwidth]{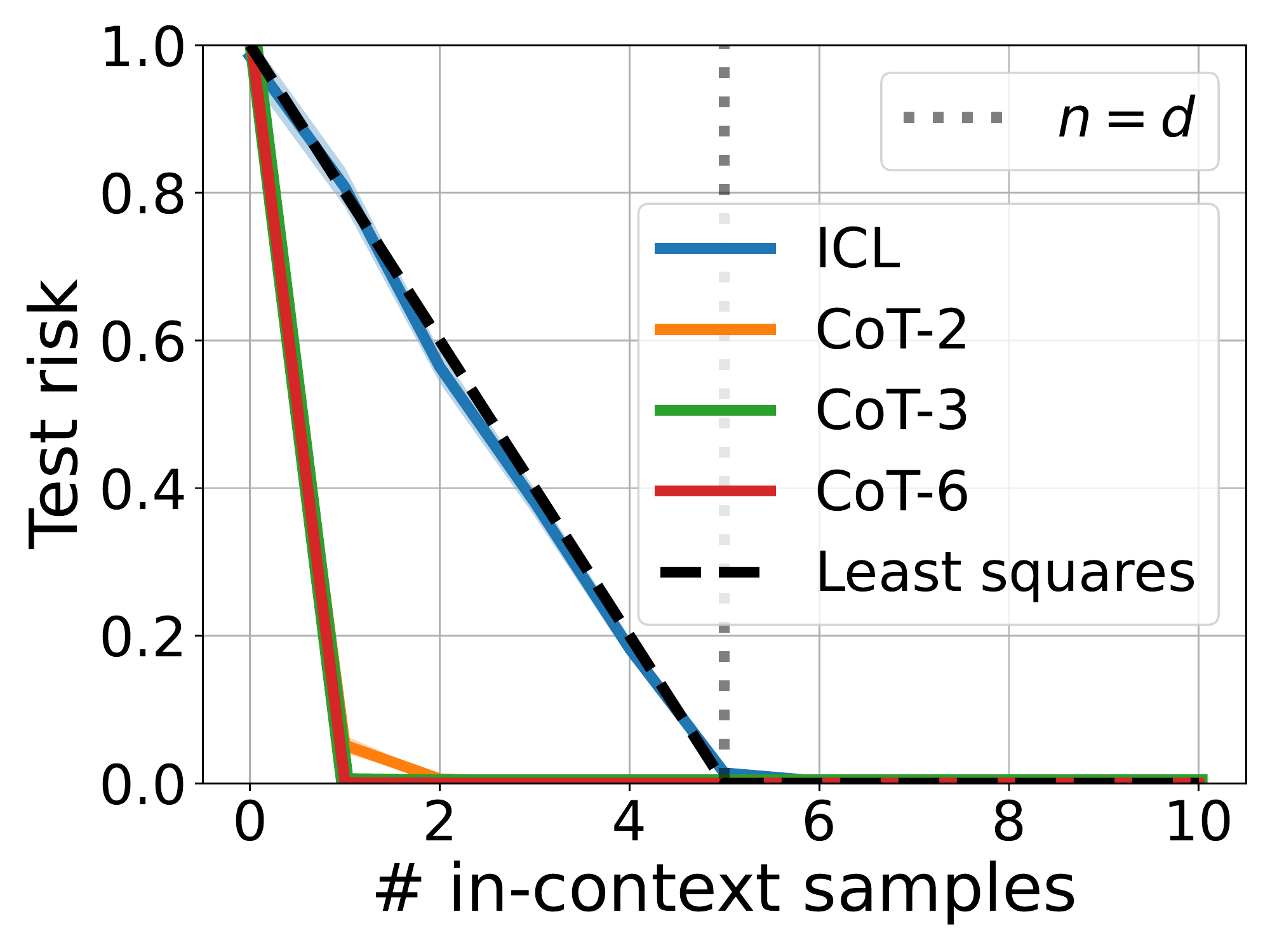}
    \label{fig:d5}
    \hspace{5mm}
}
\subfigure[One-shot prediction over time]{
    \includegraphics[height=.25\columnwidth]{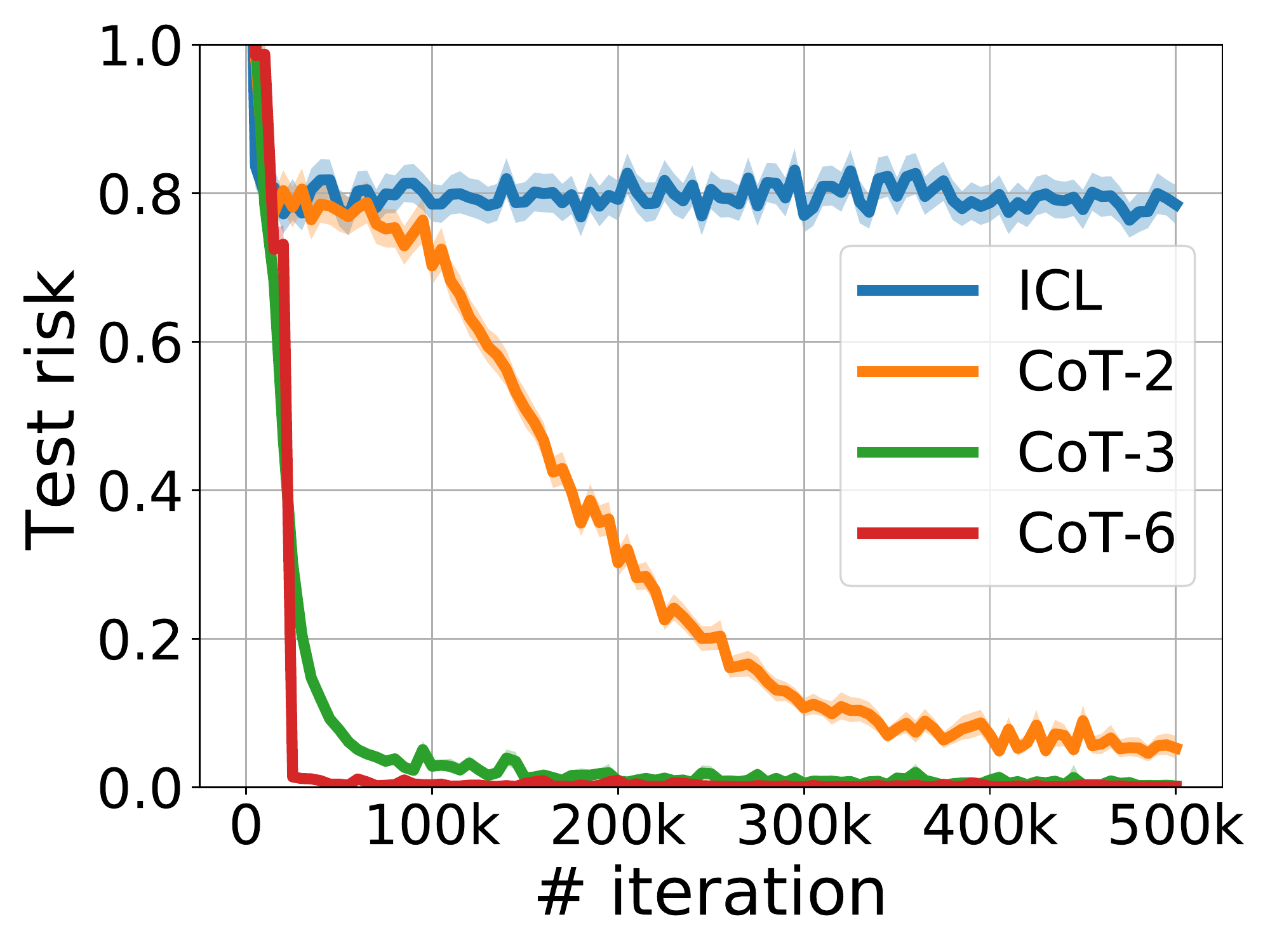}
    \label{fig:one shot}
    % \hspace{-3mm}
}
% \subfigure[Attentions of last layer]{
%     \includegraphics[height=.3\columnwidth,trim={1.5cm 0 0 0},clip]{fig_sec/figs/attention.pdf}
%     \label{fig:attn}
%     % \hspace{-3mm}
% }
\vspace{-2mm}
\caption{Evaluations over deep linear MLPs using $\COT$ and $\ICL$ where CoT-$X$ represents the $X$-step $\COT$. Fig.~\ref{fig:d5} illustrates point-to-point meta results where the model is trained with substantial number of samples. In contrast, Fig.~\ref{fig:one shot} displays the one-shot performance (with only one in-context sample provided) when making predictions during training. 
% See Section~\ref{sec:deep nn} for further implementation details.
}
\label{fig:long chain}
\vspace{-10pt}
\end{figure}
\vspace{-5pt}
\section{Further Investigation on Deep Linear MLPs}\label{sec:deep nn}
\vspace{-5pt}

% \subsection{Deep Linear Neural Networks}
In Section~\ref{sec:exp}, we have discussed the approximation benefits of $\COT$ and how it in-context learns 2-layer random  MLPs by parallel learning of $k$ $d$-dimensional ReLU and 1 $k$-dimensional linear regression. In this section, we investigate the capability of $\COT$ in learning longer compositions. For brevity, we will use CoT to refer to $\COT$ in the rest of the discussion.

\textbf{Dataset. } Consider $L$-layer linear MLPs with input $\x\in\R^d\sim\Nc(0,\Iden_d)$, and output generated by $\y=\W_L\W_{L-1}\cdots\W_1\x$, where the $\ell$th layer is parameterized by $\W_\ell\in\R^{d\times d}$, $\ell\in[L]$. In this work, to better understand the emerging ability of CoT, we assume that each layer draws from the same discrete sub-function set $\bar\Fc=\{\bar\W_k:\bar\W_k^\top\bar\W_k=\Iden,k\in[K]\}$\footnote{This assumption ensures that the norm of the feature remains constant across layers ( $\|\x\|=\|\y\|=\|\tb^\ell\|$), enabling fair evaluation across different layers.}. Therefore, to learn the $L$-layer neural net, CoT only needs to learn $\bar\Fc$ with $|\bar\Fc|=K$, whereas $\ICL$ needs to learn the function set $\bar\Fc^{L}$, which contains $K^L$ random matrices. 

\textbf{Composition ability of CoT (Figures~\ref{fig:long chain}). } Set $d=5$, $L=6$ and $K=4$. At each round, we randomly select $L$ matrices $\W_\ell$, $\ell\in[L]$ from $\bar\Fc$ so that for any input $\x$, we can form a chain
\[
\x\to\tb^{1}\to\tb^{2}\dots\to\tb^{6}:=\y,
\]
where $\tb^{\ell}=\W_\ell\tb^{\ell-1}$, $\ell\in[L]$ and $\tb^{0}:=\x$. Let CoT-$X$ denote $X$-step $\COT$ method. For example, the in-context sample of CoT-6 has form of $(\x,\tb^{1},\tb^{2},\dots\tb^{5},\y)$, which contains all the intermediate outputs from each layer; while CoT-3, CoT-2 have prompt samples formed as $(\x,\tb^{2},\tb^{4},\y)$ and $(\x,\tb^{3},\y)$ respectively. In this setting, ICL is also termed as CoT-1, as its prompt contains only $(\x,\y)$ pairs. To solve the length-6 chain, CoT-$X$ needs to learn a model that can remember $4^{6/X}$ matrices. Therefore, ICL is face a significantly challenge sincd it needs to remember $4^6=4,096$ matrices (all combinations of the 4 matrices used for training and testing) compared to just $4$ for CoT-6. 

We train small GPT-2 models using the CoT-2/-3/-6 and ICL methods, and present the results in Fig.~\ref{fig:d5}. As evident from the figure, the performance curves of CoT-2 (orange), CoT-3 (green) and CoT-6 (red)  overlap, and they can all make precise predictions in one shot (given an in-context example $n=1$). It seems that $\TF$ has effectively learned to remember up to $64$ matrices (for CoT-2) and compose up to $6$ layers (for CoT-6). However, ICL (blue) struggles to learn the 6-layer MLPs in one shot. The black dashed curve shows the solution for linear regression $y=\bbeta^\top\x$ computed directly via least squares given $n$ random training samples, where $\x$ is the input and $y$ is from the output of the 6-layer MLPs (e.g., $\y[0]$). The test risks for $n=1,\hdots 10$ are plotted (in Fig.~\ref{fig:d5}), which show that the ICL curve aligns with the least squares performance. This implies that, instead of remembering all $4,096$ matrices, ICL solves the problem from the linear regression phase. 

In addition to the meta-learning results which highlight the approximation benefits of multi-step CoT, we also investigate the convergence rate of CoT-2/-3/-6 and ICL, with results displayed in Fig.~\ref{fig:one shot}. We test the one-shot performance during training and find that CoT-6 converges fastest. This is because it has the smallest sub-function set, and given the same tasks (e.g., deep neural nets), shortening the chain leads to slower convergence. This supports the evidence that taking more steps facilitates faster and more effective learning of complex problems.

% the more steps we have, the less complex each subfunction set is and 

 % given different function complexities, and results are shown in Fig.~\ref{fig:one shot}  are  results, we also 

% Then given different training sample size $n$, test risks when trained with 

\begin{figure}[!t]
\vspace{-30pt}
\centering
    \hspace{-3mm}
\subfigure[Composed risk]{
    \includegraphics[height=.25\columnwidth]{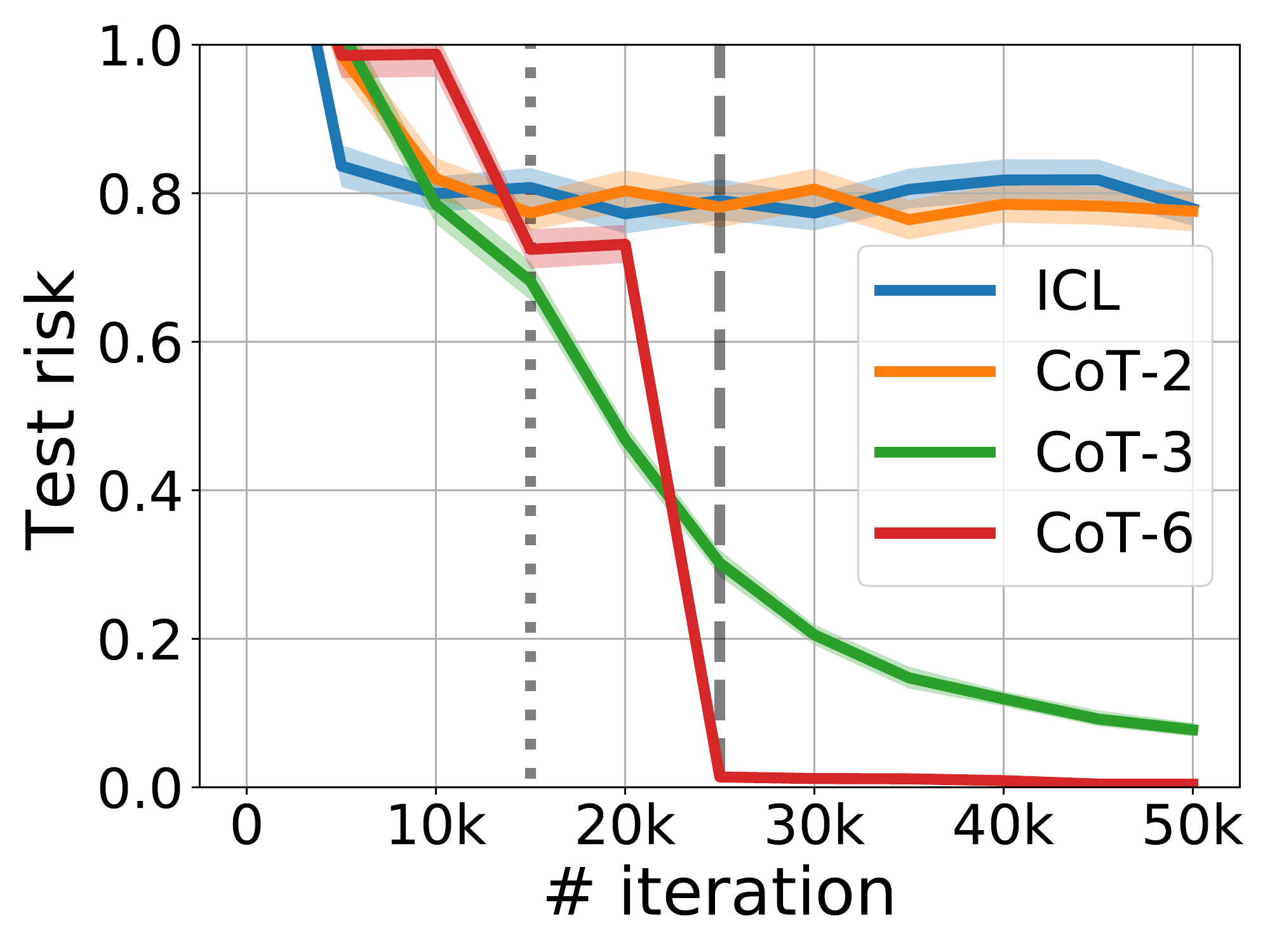}
    \label{fig:composed 50k}
    \hspace{-3mm}
}
\subfigure[Risk of each layer output]{
    \includegraphics[height=.25\columnwidth]{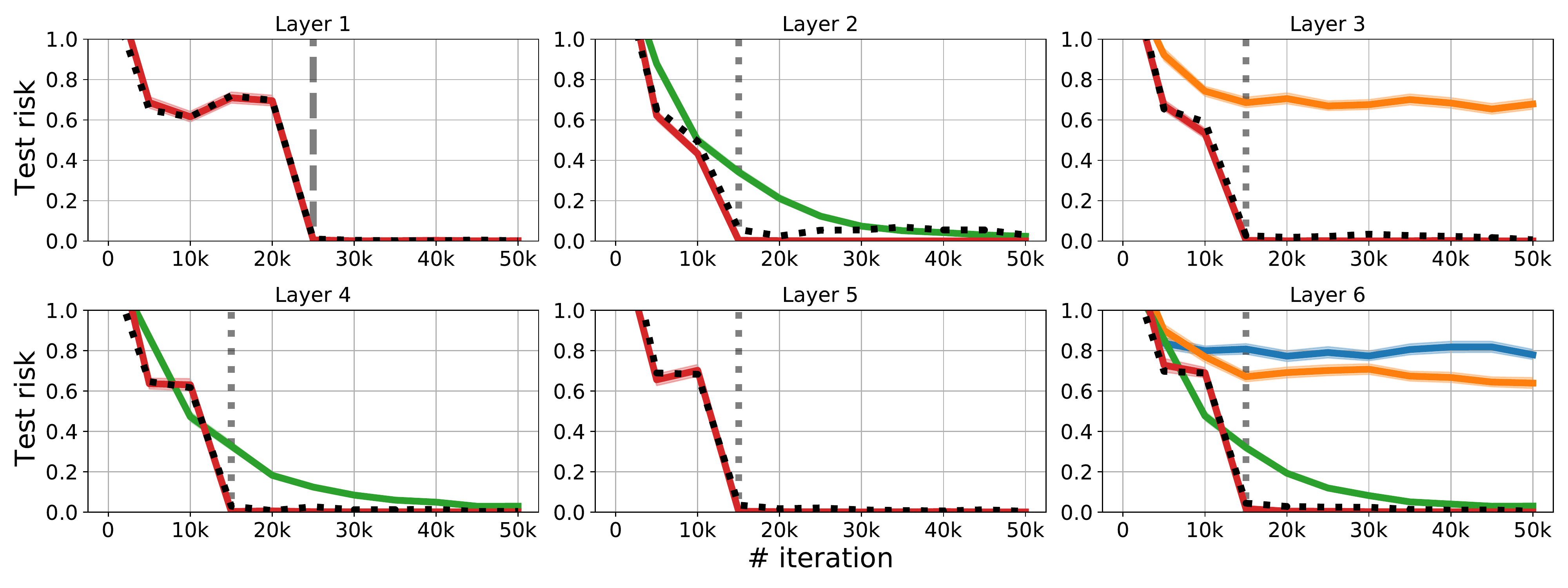}
    \label{fig:each layer 50k}
    \hspace{-3mm}
}
\vspace{-1mm}
\caption{Fig.~\ref{fig:composed 50k} is generated by magnifying the initial 50k iterations of Fig.~\ref{fig:one shot}, and we decouple the composed risks from predicting $6$-layer linear MLPs into predictions for each layer, and the results are depicted in Fig.~\ref{fig:each layer 50k}. Additional implementation details can be found in Section~\ref{sec:deep nn}.}
\label{fig:long chain 50k}
\vspace{-10pt}
\end{figure}

\textbf{Evidence of Filtering (Figure~\ref{fig:long chain 50k}). } {As per Theorem~\ref{thm:MLP} and the appendix, transformers can perform filtering over CoT prompts, and the results from 2-layer MLPs align with our theoretical findings. However, can we explicitly observe filtering behaviors?} In Fig.~\ref{fig:composed 50k}, we display the results of the first 50k iterations from Fig.~\ref{fig:one shot}, and observe risk drops in CoT-6 (red) at the 15k and 25k iteration (shown as grey dotted and dashed lines). Subsequently, in Fig.~\ref{fig:each layer 50k}, we plot the test risk of each layer prediction (by feeding the model with correct intermediate features not the predicted ones), where CoT-6 (red) predicts the outputs from all 6 layers ($\tb^1,\cdots,\tb^L$). From these figures, we can identify risk drops when predicting different layers, which appear at either 15k (for layer 2, 3, 4, 5, 6) or 25k (for layer 1) iteration. This implies that the model learns to predict each step/layer function independently. Further, we test the filtering evidence of the $\ell$th layer by filling irrelevant positions with random features. Specifically, an in-context example is formed by
\[
(\z^0,\cdots,\tb^{\ell-1},\tb^{\ell},\z^{\ell+1},\dots\z^{L}),~~\text{where}~~\tb^{\ell}=\W_\ell(\tb^{\ell-1})~~\text{and}~~\z\sim\Nc(0,\Iden_d).
\]
 The test risks are represented by black dotted curves in Fig.~\ref{fig:each layer 50k}, which aligned precisely with the CoT-6 curves (red). %It means that 
This signifies that each layer's prediction concentrate solely on the corresponding intermediate steps in the prefix, while disregarding irrelevant features. This observation provides evidence that the process of filtering is indeed performed.

% and then padding the irrelevat positions do not harm the prediction performance. 

% \subsection{Unordered CoT Prompting}

% The CoT prompt defined in \eqref{cot pmt} is homogeneous, since step messages ($\tb^1,\cdots,\tb^{L-1}$) are given fully and ordered. However, in real world applications, prompt as a hint will not provide the whole and precise messages, and therefore we are expecting that even the whole chain is not provided in the prompt, $\TF$ can still complete the chain automatically by identifying the token space. Recall that $\tb^0=\x$ and $\tb^L=\y$, then an heterogeneous CoT example is no longer as \eqref{cot pmt} with $L+1$ tokens but can be any example from length $2$ to $L+1$, that is
% \begin{align}
% \text{homogeneous: }(\tb^0,\tb^1,\cdots\tb^{L-1},\tb^L)\longrightarrow\text{heterogeneous: }(\tb^{\ell_{\text{start}}}\cdots\tb^{\ell_{\text{end}}}).\label{heter cot pmt}
% \end{align}
% where $0\leq\ell_{\text{start}}<\ell_{\text{end}}\leq L$. We emphasize that $\ell_{\text{start}}$ and $\ell_{\text{end}}$ are arbitrary for any example within any prompt, and then if given a new input $\x$, same $L$-step procedure is applied. 

% \input{fig_sec/fig_skill}
% \input{fig_sec/fig_method}
% \input{sec/exp.tex}
\section{Related Work}\label{sec:related}
% \ks{breaking it up into paragraphs for simplicity. But it wastes a lot of space. Can remove that later if we need.}

% \textbf{In-context learning:} 
With the success of LLMs and prompt structure \citep{lester2021power}, there is growing interest in in-context learning (ICL) from both theoretical and experimental lens~\citep{garg2022can,brown2020language, von2022transformers, dai2022can, min2022rethinking, lyu2022z,li2023transformers, balim2023can, xie2021explanation,min2021metaicl, wei2023larger,li2023theoretical}. 
% \citet{garg2022can} analyze the question of expressivity by training TFs to learn simple function classes from in-context examples without the issues of tokenization and language.
As an extension, chain-of-thought (CoT) prompts have made impressive improvements in performing complex reasoning by decomposing it into step-by-step intermediate solutions \citep{wei2022chain, narang2020wt5, lampinen2022can, wei2022emergent,zhou2022teaching, nye2021show, velivckovic2021neural, lanchantin2023learning}, which in general, shows the ability of transformer in solving compositional functions. \citet{lee2023teaching, dziri2023faith} study the problem of teaching arithmetic to small transformers and show that breaking the task down into small reasoning steps allows the model to learn faster. \citet{li2023symbolic} show that small student transformer models can learn from rationalizations sampled from significantly larger teacher models. The idea of learning how to compose skills has been well studied in other literatures~\citep{sahni2017learning, livska2018memorize}. More specifically, for the problem of learning shallow networks, there are several well known hardness results~\citet{goel2017reliably, goel2020tight, zhang2019learning}. In particular, \citet{hahn2023theory} shows a formal learnability bound which implies that compositional structure can benefit ICL. However, most of the work focuses on investigating empirical benefits and algorithmic designing of CoT, and there exists little effort studying the underlying mechanisms of  CoT. 

%There is also literature showing expressivity of transformer. 
Considering the expressivity of the transformer architecture itself, \citet{yun2019transformers} showed that TFs are universal sequence to sequence approximators. More recently, \citet{giannou2023looped} use an explicit construction to show that shallow $\TF$s can be used to run general purpose programs as long as we loop them. Other works have also shown the turing-completeness of the $\TF$ architecture but these typically require infinite/high precision and recursion around attention layers~\citep{wei2022statistically, perez2019turing, perez2021attention, liu2022transformers}. Closer to our work, \citet{akyurek2022learning, von2023transformers,von2023uncovering} prove that a transformer with constant number of layers can implement gradient descent in solving linear regression, and \citet{giannou2023looped} introduce similar results by looping outputs back into inputs. \citet{ahn2023transformers} prove this result from an optimization perspective and show that the global minimum of the training objective implements a single iteration of preconditioned gradient descent for transformers with a single layer of linear attention. \citet{zhou2023algorithms} introduce the RASP-Generalization Conjecture which says that Transformers tend to length generalize on
a task if the task can be solved by a short RASP program which works for all input lengths. In this work, we prove CoT can be treated as: first apply filtering on the CoT prompts using special construction, and then in-context learn the filtered prompt.

% \textbf{Chain-of-thought:} TFs also display a property commonly known as chain-of-thought~\citep{wei2022chain, narang2020wt5, lampinen2022can, wei2022emergent}, showing improved performance when prompted to go \emph{step-by-step}~\citep{zhou2022teaching, nye2021show, kaiser2015neural, velivckovic2021neural, lanchantin2023learning}. \citet{hahn2023theory} in particular, show a formal learnability bound using that implies that compositional structure can benefit ICL.

% \textbf{Emergent Properties:} 
% ICL and CoT are just few of the remarkable properties seem to ``emerge'' from LLMs with scale~\citep{chowdhery2022palm, ganguli2022predictability, srivastava2022beyond, radford2019language}. \citet{wei2022emergent} define \emph{emergent properties} as those that appear in models only beyond a certain amount of scale. These may include the ability to do arithmetic, program execution, instruction following, etc.
% More recently, \citet{schaeffer2023emergent} claim that scale may not be the only factor in deciding if a property is emergent and that the evaluation metric may also play a significant role.

\section{\yl{Conclusion, Limitations, and Discussion}}\label{sec:discuss}
In this work, we investigate chain-of-thought prompting and shed light on how it enables compositional learning of multilayer perceptrons step-by-step. Specially, we have explored and contrasted three methods: $\ICL$, $\EXP$ and $\COT$, and found that $\COT$ facilitates better approximation and faster convergence through looping and sample efficiency. Additionally, we empirically and theoretically demonstrated that to learn a 2-layer MLP with $d$-dimensional input and $k$ neurons, $\COT$ requires ${\cal{O}}(\max(d,k))$ in-context samples whereas $\ICL$ runs into approximation error bottlenecks. 

\yl{While we have provided both experimental and theoretical results to validate the advantages of CoT, it is important to note that our analysis in the main text pertains to in-distribution scenarios. In an effort to address this limitation and demonstrate the robustness of CoT, we have conducted additional simulations, as detailed in Appendix~\ref{sec:robust}, where the test samples follow a different distribution than the training examples. Also, we note that our focus has been primarily on MLP-based tasks, where the subproblems are essentially instances of simple linear regression. It would be valuable to explore how CoT might influence the training of tasks characterized by more complex structures, longer compositional chains, and a broader variety of subproblems.}

%potential questions regarding our findings
There are several interesting avenues for future research to build on our findings. To what extent does our decoupling of CoT (filtering followed by ICL) align with the empirical evidence in practical problems such as code generation and mathematical reasoning? We have shown that $\COT$ can rely on a linear regression oracle to learn an MLP. To what extent can transformers approximate MLPs without $\COT$ (e.g.~with $\EXP$) and what are the corresponding lower/upper bounds?

%Since $\EXP$ and $\COT$ provide intermediate messages in the prefix, is there proper algorithm using input-output only but can get competitive performance as CoT. 

%How can we evidence filtering if training with real/complex dataset? 2) 
% assume correct data in the prefix, can we do better ICL so that $\TF$ will improved itself by looping 

% In a nutshell, $\TF$ can 

\section{Acknowledgements}
This work was supported by the NSF CAREER awards \#2046816 and \#1844951, NSF grant \#2212426, AFOSR \& AFRL Center of Excellence Award FA9550-18-1-0166, a Google Research Scholar award, an Adobe Data Science Research award, an Army Research Office grant W911NF2110312, and an ONR Grant No. N00014-21-1-2806.

%NSF CAREER Award #1844951, an AFOSR & AFRL Center of Excellence Award FA9550-18-1-0166, and an ONR Grant No. N00014- 21-1-2806.
% \bibliographystyle{plain}
\bibliographystyle{plainnat}
\bibliography{refs, biblio}

\begin{thebibliography}{47}
\providecommand{\natexlab}[1]{#1}
\providecommand{\url}[1]{\texttt{#1}}
\expandafter\ifx\csname urlstyle\endcsname\relax
  \providecommand{\doi}[1]{doi: #1}\else
  \providecommand{\doi}{doi: \begingroup \urlstyle{rm}\Url}\fi

\bibitem[Ahn et~al.(2023)Ahn, Cheng, Daneshmand, and Sra]{ahn2023transformers}
Kwangjun Ahn, Xiang Cheng, Hadi Daneshmand, and Suvrit Sra.
\newblock Transformers learn to implement preconditioned gradient descent for
  in-context learning.
\newblock \emph{arXiv preprint arXiv:2306.00297}, 2023.

\bibitem[Aky{\"u}rek et~al.(2022)Aky{\"u}rek, Schuurmans, Andreas, Ma, and
  Zhou]{akyurek2022learning}
Ekin Aky{\"u}rek, Dale Schuurmans, Jacob Andreas, Tengyu Ma, and Denny Zhou.
\newblock What learning algorithm is in-context learning? investigations with
  linear models.
\newblock \emph{arXiv preprint arXiv:2211.15661}, 2022.

\bibitem[Balim et~al.(2023)Balim, Du, Oymak, and Ozay]{balim2023can}
Haldun Balim, Zhe Du, Samet Oymak, and Necmiye Ozay.
\newblock Can transformers learn optimal filtering for unknown systems?
\newblock \emph{arXiv preprint arXiv:2308.08536}, 2023.

\bibitem[Brown et~al.(2020)Brown, Mann, Ryder, Subbiah, Kaplan, Dhariwal,
  Neelakantan, Shyam, Sastry, Askell, et~al.]{brown2020language}
Tom Brown, Benjamin Mann, Nick Ryder, Melanie Subbiah, Jared~D Kaplan, Prafulla
  Dhariwal, Arvind Neelakantan, Pranav Shyam, Girish Sastry, Amanda Askell,
  et~al.
\newblock Language models are few-shot learners.
\newblock \emph{Advances in neural information processing systems},
  33:\penalty0 1877--1901, 2020.

\bibitem[Dai et~al.(2022)Dai, Sun, Dong, Hao, Sui, and Wei]{dai2022can}
Damai Dai, Yutao Sun, Li~Dong, Yaru Hao, Zhifang Sui, and Furu Wei.
\newblock Why can gpt learn in-context? language models secretly perform
  gradient descent as meta optimizers.
\newblock \emph{arXiv preprint arXiv:2212.10559}, 2022.

\bibitem[Dziri et~al.(2023)Dziri, Lu, Sclar, Li, Jian, Lin, West, Bhagavatula,
  Bras, Hwang, et~al.]{dziri2023faith}
Nouha Dziri, Ximing Lu, Melanie Sclar, Xiang~Lorraine Li, Liwei Jian,
  Bill~Yuchen Lin, Peter West, Chandra Bhagavatula, Ronan~Le Bras, Jena~D
  Hwang, et~al.
\newblock Faith and fate: Limits of transformers on compositionality.
\newblock \emph{arXiv preprint arXiv:2305.18654}, 2023.

\bibitem[Garg et~al.(2022)Garg, Tsipras, Liang, and Valiant]{garg2022can}
Shivam Garg, Dimitris Tsipras, Percy~S Liang, and Gregory Valiant.
\newblock What can transformers learn in-context? a case study of simple
  function classes.
\newblock \emph{Advances in Neural Information Processing Systems},
  35:\penalty0 30583--30598, 2022.

\bibitem[Giannou et~al.(2023)Giannou, Rajput, Sohn, Lee, Lee, and
  Papailiopoulos]{giannou2023looped}
Angeliki Giannou, Shashank Rajput, Jy-yong Sohn, Kangwook Lee, Jason~D Lee, and
  Dimitris Papailiopoulos.
\newblock Looped transformers as programmable computers.
\newblock \emph{arXiv preprint arXiv:2301.13196}, 2023.

\bibitem[Goel et~al.(2017)Goel, Kanade, Klivans, and Thaler]{goel2017reliably}
Surbhi Goel, Varun Kanade, Adam Klivans, and Justin Thaler.
\newblock Reliably learning the relu in polynomial time.
\newblock In \emph{Conference on Learning Theory}, pages 1004--1042. PMLR,
  2017.

\bibitem[Goel et~al.(2020)Goel, Klivans, Manurangsi, and
  Reichman]{goel2020tight}
Surbhi Goel, Adam Klivans, Pasin Manurangsi, and Daniel Reichman.
\newblock Tight hardness results for training depth-2 relu networks.
\newblock \emph{arXiv preprint arXiv:2011.13550}, 2020.

\bibitem[Hahn and Goyal(2023)]{hahn2023theory}
Michael Hahn and Navin Goyal.
\newblock A theory of emergent in-context learning as implicit structure
  induction.
\newblock \emph{arXiv preprint arXiv:2303.07971}, 2023.

\bibitem[Imani et~al.(2023)Imani, Du, and Shrivastava]{imani2023mathprompter}
Shima Imani, Liang Du, and Harsh Shrivastava.
\newblock Mathprompter: Mathematical reasoning using large language models.
\newblock \emph{arXiv preprint arXiv:2303.05398}, 2023.

\bibitem[Lampinen et~al.(2022)Lampinen, Dasgupta, Chan, Matthewson, Tessler,
  Creswell, McClelland, Wang, and Hill]{lampinen2022can}
Andrew~K Lampinen, Ishita Dasgupta, Stephanie~CY Chan, Kory Matthewson,
  Michael~Henry Tessler, Antonia Creswell, James~L McClelland, Jane~X Wang, and
  Felix Hill.
\newblock Can language models learn from explanations in context?
\newblock \emph{arXiv preprint arXiv:2204.02329}, 2022.

\bibitem[Lanchantin et~al.(2023)Lanchantin, Toshniwal, Weston, Szlam, and
  Sukhbaatar]{lanchantin2023learning}
Jack Lanchantin, Shubham Toshniwal, Jason Weston, Arthur Szlam, and Sainbayar
  Sukhbaatar.
\newblock Learning to reason and memorize with self-notes.
\newblock \emph{arXiv preprint arXiv:2305.00833}, 2023.

\bibitem[Lee et~al.(2023)Lee, Sreenivasan, Lee, Lee, and
  Papailiopoulos]{lee2023teaching}
Nayoung Lee, Kartik Sreenivasan, Jason~D Lee, Kangwook Lee, and Dimitris
  Papailiopoulos.
\newblock Teaching arithmetic to small transformers.
\newblock \emph{arXiv preprint arXiv:2307.03381}, 2023.

\bibitem[Lester et~al.(2021)Lester, Al-Rfou, and Constant]{lester2021power}
Brian Lester, Rami Al-Rfou, and Noah Constant.
\newblock The power of scale for parameter-efficient prompt tuning.
\newblock \emph{arXiv preprint arXiv:2104.08691}, 2021.

\bibitem[Li et~al.(2023{\natexlab{a}})Li, Wang, Liu, and
  Chen]{li2023theoretical}
Hongkang Li, Meng Wang, Sijia Liu, and Pin-Yu Chen.
\newblock A theoretical understanding of shallow vision transformers: Learning,
  generalization, and sample complexity.
\newblock \emph{arXiv preprint arXiv:2302.06015}, 2023{\natexlab{a}}.

\bibitem[Li et~al.(2023{\natexlab{b}})Li, Hessel, Yu, Ren, Chang, and
  Choi]{li2023symbolic}
Liunian~Harold Li, Jack Hessel, Youngjae Yu, Xiang Ren, Kai-Wei Chang, and
  Yejin Choi.
\newblock Symbolic chain-of-thought distillation: Small models can also" think"
  step-by-step.
\newblock \emph{arXiv preprint arXiv:2306.14050}, 2023{\natexlab{b}}.

\bibitem[Li et~al.(2023{\natexlab{c}})Li, Ildiz, Papailiopoulos, and
  Oymak]{li2023transformers}
Yingcong Li, M~Emrullah Ildiz, Dimitris Papailiopoulos, and Samet Oymak.
\newblock Transformers as algorithms: Generalization and stability in
  in-context learning.
\newblock \emph{International Conference on Machine Learning},
  2023{\natexlab{c}}.

\bibitem[Li{\v{s}}ka et~al.(2018)Li{\v{s}}ka, Kruszewski, and
  Baroni]{livska2018memorize}
Adam Li{\v{s}}ka, Germ{\'a}n Kruszewski, and Marco Baroni.
\newblock Memorize or generalize? searching for a compositional rnn in a
  haystack.
\newblock \emph{arXiv preprint arXiv:1802.06467}, 2018.

\bibitem[Liu et~al.(2022)Liu, Ash, Goel, Krishnamurthy, and
  Zhang]{liu2022transformers}
Bingbin Liu, Jordan~T Ash, Surbhi Goel, Akshay Krishnamurthy, and Cyril Zhang.
\newblock Transformers learn shortcuts to automata.
\newblock \emph{arXiv preprint arXiv:2210.10749}, 2022.

\bibitem[Lyu et~al.(2022)Lyu, Min, Beltagy, Zettlemoyer, and
  Hajishirzi]{lyu2022z}
Xinxi Lyu, Sewon Min, Iz~Beltagy, Luke Zettlemoyer, and Hannaneh Hajishirzi.
\newblock Z-icl: Zero-shot in-context learning with pseudo-demonstrations.
\newblock \emph{arXiv preprint arXiv:2212.09865}, 2022.

\bibitem[Min et~al.(2021)Min, Lewis, Zettlemoyer, and
  Hajishirzi]{min2021metaicl}
Sewon Min, Mike Lewis, Luke Zettlemoyer, and Hannaneh Hajishirzi.
\newblock Metaicl: Learning to learn in context.
\newblock \emph{arXiv preprint arXiv:2110.15943}, 2021.

\bibitem[Min et~al.(2022)Min, Lyu, Holtzman, Artetxe, Lewis, Hajishirzi, and
  Zettlemoyer]{min2022rethinking}
Sewon Min, Xinxi Lyu, Ari Holtzman, Mikel Artetxe, Mike Lewis, Hannaneh
  Hajishirzi, and Luke Zettlemoyer.
\newblock Rethinking the role of demonstrations: What makes in-context learning
  work?
\newblock \emph{arXiv preprint arXiv:2202.12837}, 2022.

\bibitem[Narang et~al.(2020)Narang, Raffel, Lee, Roberts, Fiedel, and
  Malkan]{narang2020wt5}
Sharan Narang, Colin Raffel, Katherine Lee, Adam Roberts, Noah Fiedel, and
  Karishma Malkan.
\newblock Wt5?! training text-to-text models to explain their predictions.
\newblock \emph{arXiv preprint arXiv:2004.14546}, 2020.

\bibitem[Nye et~al.(2021)Nye, Andreassen, Gur-Ari, Michalewski, Austin, Bieber,
  Dohan, Lewkowycz, Bosma, Luan, et~al.]{nye2021show}
Maxwell Nye, Anders~Johan Andreassen, Guy Gur-Ari, Henryk Michalewski, Jacob
  Austin, David Bieber, David Dohan, Aitor Lewkowycz, Maarten Bosma, David
  Luan, et~al.
\newblock Show your work: Scratchpads for intermediate computation with
  language models.
\newblock \emph{arXiv preprint arXiv:2112.00114}, 2021.

\bibitem[P{\'e}rez et~al.(2019)P{\'e}rez, Marinkovi{\'c}, and
  Barcel{\'o}]{perez2019turing}
Jorge P{\'e}rez, Javier Marinkovi{\'c}, and Pablo Barcel{\'o}.
\newblock On the turing completeness of modern neural network architectures.
\newblock \emph{arXiv preprint arXiv:1901.03429}, 2019.

\bibitem[P{\'e}rez et~al.(2021)P{\'e}rez, Barcel{\'o}, and
  Marinkovic]{perez2021attention}
Jorge P{\'e}rez, Pablo Barcel{\'o}, and Javier Marinkovic.
\newblock Attention is turing complete.
\newblock \emph{The Journal of Machine Learning Research}, 22\penalty0
  (1):\penalty0 3463--3497, 2021.

\bibitem[Perez et~al.(2021)Perez, Ottens, and Viswanathan]{perez2021automatic}
Luis Perez, Lizi Ottens, and Sudharshan Viswanathan.
\newblock Automatic code generation using pre-trained language models.
\newblock \emph{arXiv preprint arXiv:2102.10535}, 2021.

\bibitem[Radford et~al.(2019)Radford, Wu, Child, Luan, Amodei, Sutskever,
  et~al.]{radford2019language}
Alec Radford, Jeffrey Wu, Rewon Child, David Luan, Dario Amodei, Ilya
  Sutskever, et~al.
\newblock Language models are unsupervised multitask learners.
\newblock \emph{OpenAI blog}, 1\penalty0 (8):\penalty0 9, 2019.

\bibitem[Sahni et~al.(2017)Sahni, Kumar, Tejani, and Isbell]{sahni2017learning}
Himanshu Sahni, Saurabh Kumar, Farhan Tejani, and Charles Isbell.
\newblock Learning to compose skills.
\newblock \emph{arXiv preprint arXiv:1711.11289}, 2017.

\bibitem[Vaswani et~al.(2017)Vaswani, Shazeer, Parmar, Uszkoreit, Jones, Gomez,
  Kaiser, and Polosukhin]{vaswani2017attention}
Ashish Vaswani, Noam Shazeer, Niki Parmar, Jakob Uszkoreit, Llion Jones,
  Aidan~N Gomez, {\L}ukasz Kaiser, and Illia Polosukhin.
\newblock Attention is all you need.
\newblock \emph{Advances in neural information processing systems}, 30, 2017.

\bibitem[Veli{\v{c}}kovi{\'c} and Blundell(2021)]{velivckovic2021neural}
Petar Veli{\v{c}}kovi{\'c} and Charles Blundell.
\newblock Neural algorithmic reasoning.
\newblock \emph{Patterns}, 2\penalty0 (7):\penalty0 100273, 2021.

\bibitem[von Oswald et~al.(2022)von Oswald, Niklasson, Randazzo, Sacramento,
  Mordvintsev, Zhmoginov, and Vladymyrov]{von2022transformers}
Johannes von Oswald, Eyvind Niklasson, Ettore Randazzo, Jo{\~a}o Sacramento,
  Alexander Mordvintsev, Andrey Zhmoginov, and Max Vladymyrov.
\newblock Transformers learn in-context by gradient descent.
\newblock \emph{arXiv preprint arXiv:2212.07677}, 2022.

\bibitem[Von~Oswald et~al.(2023)Von~Oswald, Niklasson, Randazzo, Sacramento,
  Mordvintsev, Zhmoginov, and Vladymyrov]{von2023transformers}
Johannes Von~Oswald, Eyvind Niklasson, Ettore Randazzo, Jo{\~a}o Sacramento,
  Alexander Mordvintsev, Andrey Zhmoginov, and Max Vladymyrov.
\newblock Transformers learn in-context by gradient descent.
\newblock In \emph{International Conference on Machine Learning}, pages
  35151--35174. PMLR, 2023.

\bibitem[von Oswald et~al.(2023)von Oswald, Niklasson, Schlegel, Kobayashi,
  Zucchet, Scherrer, Miller, Sandler, Vladymyrov, Pascanu,
  et~al.]{von2023uncovering}
Johannes von Oswald, Eyvind Niklasson, Maximilian Schlegel, Seijin Kobayashi,
  Nicolas Zucchet, Nino Scherrer, Nolan Miller, Mark Sandler, Max Vladymyrov,
  Razvan Pascanu, et~al.
\newblock Uncovering mesa-optimization algorithms in transformers.
\newblock \emph{arXiv preprint arXiv:2309.05858}, 2023.

\bibitem[Wei et~al.(2022{\natexlab{a}})Wei, Chen, and Ma]{wei2022statistically}
Colin Wei, Yining Chen, and Tengyu Ma.
\newblock Statistically meaningful approximation: a case study on approximating
  turing machines with transformers.
\newblock \emph{Advances in Neural Information Processing Systems},
  35:\penalty0 12071--12083, 2022{\natexlab{a}}.

\bibitem[Wei et~al.(2022{\natexlab{b}})Wei, Tay, Bommasani, Raffel, Zoph,
  Borgeaud, Yogatama, Bosma, Zhou, Metzler, et~al.]{wei2022emergent}
Jason Wei, Yi~Tay, Rishi Bommasani, Colin Raffel, Barret Zoph, Sebastian
  Borgeaud, Dani Yogatama, Maarten Bosma, Denny Zhou, Donald Metzler, et~al.
\newblock Emergent abilities of large language models.
\newblock \emph{arXiv preprint arXiv:2206.07682}, 2022{\natexlab{b}}.

\bibitem[Wei et~al.(2022{\natexlab{c}})Wei, Wang, Schuurmans, Bosma, Chi, Le,
  and Zhou]{wei2022chain}
Jason Wei, Xuezhi Wang, Dale Schuurmans, Maarten Bosma, Ed~Chi, Quoc Le, and
  Denny Zhou.
\newblock Chain of thought prompting elicits reasoning in large language
  models.
\newblock \emph{arXiv preprint arXiv:2201.11903}, 2022{\natexlab{c}}.

\bibitem[Wei et~al.(2023)Wei, Wei, Tay, Tran, Webson, Lu, Chen, Liu, Huang,
  Zhou, et~al.]{wei2023larger}
Jerry Wei, Jason Wei, Yi~Tay, Dustin Tran, Albert Webson, Yifeng Lu, Xinyun
  Chen, Hanxiao Liu, Da~Huang, Denny Zhou, et~al.
\newblock Larger language models do in-context learning differently.
\newblock \emph{arXiv preprint arXiv:2303.03846}, 2023.

\bibitem[Wolf et~al.(2019)Wolf, Debut, Sanh, Chaumond, Delangue, Moi, Cistac,
  Rault, Louf, Funtowicz, et~al.]{wolf2019huggingface}
Thomas Wolf, Lysandre Debut, Victor Sanh, Julien Chaumond, Clement Delangue,
  Anthony Moi, Pierric Cistac, Tim Rault, R{\'e}mi Louf, Morgan Funtowicz,
  et~al.
\newblock Huggingface's transformers: State-of-the-art natural language
  processing.
\newblock \emph{arXiv preprint arXiv:1910.03771}, 2019.

\bibitem[Xie et~al.(2021)Xie, Raghunathan, Liang, and Ma]{xie2021explanation}
Sang~Michael Xie, Aditi Raghunathan, Percy Liang, and Tengyu Ma.
\newblock An explanation of in-context learning as implicit bayesian inference.
\newblock \emph{arXiv preprint arXiv:2111.02080}, 2021.

\bibitem[Yuan et~al.(2023)Yuan, Yuan, Tan, Wang, and Huang]{yuan2023well}
Zheng Yuan, Hongyi Yuan, Chuanqi Tan, Wei Wang, and Songfang Huang.
\newblock How well do large language models perform in arithmetic tasks?
\newblock \emph{arXiv preprint arXiv:2304.02015}, 2023.

\bibitem[Yun et~al.(2019)Yun, Bhojanapalli, Rawat, Reddi, and
  Kumar]{yun2019transformers}
Chulhee Yun, Srinadh Bhojanapalli, Ankit~Singh Rawat, Sashank~J Reddi, and
  Sanjiv Kumar.
\newblock Are transformers universal approximators of sequence-to-sequence
  functions?
\newblock \emph{arXiv preprint arXiv:1912.10077}, 2019.

\bibitem[Zhang et~al.(2019)Zhang, Yu, Wang, and Gu]{zhang2019learning}
Xiao Zhang, Yaodong Yu, Lingxiao Wang, and Quanquan Gu.
\newblock Learning one-hidden-layer relu networks via gradient descent.
\newblock In \emph{The 22nd international conference on artificial intelligence
  and statistics}, pages 1524--1534. PMLR, 2019.

\bibitem[Zhou et~al.(2022)Zhou, Nova, Larochelle, Courville, Neyshabur, and
  Sedghi]{zhou2022teaching}
Hattie Zhou, Azade Nova, Hugo Larochelle, Aaron Courville, Behnam Neyshabur,
  and Hanie Sedghi.
\newblock Teaching algorithmic reasoning via in-context learning.
\newblock \emph{arXiv preprint arXiv:2211.09066}, 2022.

\bibitem[Zhou et~al.(2023)Zhou, Bradley, Littwin, Razin, Saremi, Susskind,
  Bengio, and Nakkiran]{zhou2023algorithms}
Hattie Zhou, Arwen Bradley, Etai Littwin, Noam Razin, Omid Saremi, Josh
  Susskind, Samy Bengio, and Preetum Nakkiran.
\newblock What algorithms can transformers learn? a study in length
  generalization.
\newblock \emph{arXiv preprint arXiv:2310.16028}, 2023.

\end{thebibliography}

\newpage
\appendix
\addcontentsline{toc}{section}{Appendix} % Add the appendix text
\part{Appendix}
\parttoc % Insert the appendix TOC
% \section*{Organization of Appendix}
% \begin{myitemize}
%     \item
% \end{myitemize}
%\onecolumn
% \input{sec/todo}
% \input{sec/outline.tex}
% \input{sec/setup (backup)}
% \input{sec/instruction_selection.tex}
% \input{sec/experiment_setup.tex}
\section{Construction}\label{ss:construction}
\subsection{The Transformer Architecture }\label{ss:transformer}
For the purpose of this proof, we consider encoder-based transformer architectures and assume that the positional encodings are appended to the input\footnote{We note here that in terms of our construction adding the encodings or appending them to the input can be viewed in a similar manner. Specifically, we can consider that the up-projection step projects the input to  zero padded vectors, while the encodings are orthogonal to that projection in the sense that the have zero in the non-zero coordinates of the input. In that case adding the positional encodings  corresponds to appending them to the input. }. We also consider that the heads are added and each one has each own key, query and value weight matrices. Formally, we have 
\begin{align}\label{eq:transformer}
    \attn(\X) &= \X + \sum_{h=1}^H\weights^h_\val\X\sfm((\weights^h_\key\X)^\top\weights^h_\query\X)\\
    \TF(\X) &= \attn(\X) + \weights_2(\weights_1\attn(\X) +\vb_1\one_n)_{+} +\vb_2\one_n
\end{align}
where $\X\in\R^{d\times n}$, $H$ is the number of heads used and $f(x) = (x)_{+}$ is the ReLU activation. We also make use of the temperature $\lambda$, which is a parameter of the softmax. Specifically, $\sfm(\x) = \{\e^{\lambda x_i}/\sum_j\e^{\lambda x_j}\}_{i}$. Notice that as $\lambda\to\infty$ we have $\sfm(\x)\to \max_i x_i$. We also assume that the inputs are bounded; we denote with $N_{max}$ the maximum sequence length of the model.
\begin{assumption}\label{asm:bounded}Each entry is bounded by some constant $c$, which implies that  $\|\X\|\leq c'$, for some large $c'$ that depends on the maximum sequence length and the width of the transformer.
\end{assumption}
\subsection{Positional Encodings}\label{ss:pencodings}  In the constructions below we use a combination of the binary representation of each position, as well as some additional information (e.g. 0-1 bits) as described in the following sections. The binary representations and in general encodings we construct, require only logarithmic space with respect to the sequence length. Notice that binary representation of the positions satisfy the following two conditions:
\begin{enumerate}
    \item Let $\emb_i$ be the binary representation of the $i$-th position, then there exists $\varepsilon > 0$ such that $\emb_i^\top\emb_i > \emb_i^\top\emb_j +\varepsilon$ for all $j\neq i$.
    \item There exists a one layer transformer that can implement the addition of two pointers (see Lemma \ref{lem:addition} ).
\end{enumerate}
%We note here that any choice of encodings could be made that satisfy the following two conditions: 
% \begin{assumption}\label{asm:distinct}
%     Let $\emb_i$ be the encoding of the $i$-th position, then there exists $\varepsilon > 0$ such that $\emb_i^\top\emb_i > \emb_i^\top\emb_j +\varepsilon$ for all $j\neq i$. 
% \end{assumption} 
% \begin{assumption}\label{asm:increasing}
%     There exists an operator $*$ that the transformer can implement in one layer (or constant layers) such that $\emb_i *\emb_j  =\emb_{i+j}$.
% \end{assumption}
% We note here that multiple choices of positional encodings satisfy these two assumptions. For example, \cite{..} have used binary positional encodings

%One useful fact about the encodings is that as 
\subsection{Constructing Some Useful ``Black-box'' Functions}
We follow the construction of previous work \citep{giannou2023looped}  and use the following individual implementations as they do. We repeat the statements here for convenience. The first lemma is also similar to \cite{akyurek2022learning}, we however follow a slightly different proof. 
\begin{lemma}\label{lem:copy}
A transformer with one layer, two heads, and embedding dimension of ${\cal{O}}(\log n +d)$, where $d$ is the input dimension and $n$ is the sequence length, can copy any block of the input to any other block of the input with error $\varepsilon$ arbitrarily small.
\end{lemma}
\begin{proof}
 Consider an input of the following form 
 \begin{equation}
     \X = \begin{bmatrix}
         \x_1 & \x_2 &\hdots &\x_n\\
         \y_1 &\y_2 &\hdots &\y_n\\
         \emb_1 &\emb_2 &\hdots &\emb_n\\
         \emb_k &\emb_2 &\hdots &\emb_n
     \end{bmatrix}
 \end{equation}
 where $\x_i,\y_i\in\R^d$ are column vectors for all $i=1,\hdots, n$.
We present how we can move a block of data from the $(d+1:2d,k)$ position block to the $(1:d,1)$ position block, meaning to move the point $\y_k$ to the position of the point $\x_1$.   It is straightforward to generalize the proof that follows to move blocks from $(i:i+k, j)$ to $(i':i'+k,j')$ for arbitrary $i,i',j,j',k$. 

Let in Eq. \ref{eq:transformer} $\weights^1_\key =\begin{bmatrix}\zero &\zero &\Iden& \zero \end{bmatrix}$ and $\weights^1_\query = \begin{bmatrix}\zero &\zero & \zero &\Iden\end{bmatrix}$ and   $\weights^2_\key =\begin{bmatrix}\zero &\zero& \zero  &\Iden\end{bmatrix}$ and $\weights^2_\query = \begin{bmatrix}\zero &\zero & \zero &\Iden\end{bmatrix}$ then we have
\begin{equation}
(\weights^1_\key\X)^\top(\weights^1_\query\X) = \begin{bmatrix}
    \emb_1^\top\emb_k & \emb_1^\top\emb_2 &\hdots &\emb_1^\top\emb_k &\hdots &\emb_1^\top\emb_n\\
    \emb_2^\top\emb_k & \emb_2^\top\emb_2 &\hdots &\emb_2^\top\emb_k &\hdots &\emb_2^\top\emb_n\\
    \vdots &\vdots &\ddots &\vdots &\ddots & \vdots\\
    \emb_k^\top\emb_k & \emb_k^\top\emb_2 &\hdots &\emb_k^\top\emb_k &\hdots &\emb_k^\top\emb_n\\
    \vdots &\vdots &\ddots &\vdots &\ddots & \vdots\\
    \emb_n^\top\emb_k & \emb_n^\top\emb_2 &\hdots &\emb_n^\top\emb_k &\hdots &\emb_n^\top\emb_n\\
\end{bmatrix}
\end{equation}
As $\lambda\to \infty$ and after the application of the softmax operator the above matrix becomes equal to 
$\begin{bmatrix}
    \e_k & \e_2 &\hdots &\e_k &\hdots \e_n
\end{bmatrix} + \eps\rmM$, where $\e_i$ is the one-hot vector with $1$ in the $k$-th position, $\|M\|\leq 1$ and $\eps$ is controllable by the temperature parameter and can be arbitrary small. Let finally $\weights^1_\val = \begin{bmatrix}
    \zero &\Iden &\zero &\zero \\
    \zero &\zero &\zero &\zero
\end{bmatrix}$ we get that 
\begin{equation}
    \weights^1_\val\X\sfm((\weights^1_\key\X)^\top(\weights^1_\query\X)) = \begin{bmatrix}
        \y_k &\x_2 &\hdots &\x_n\\
        \zero &\zero &\hdots &\zero 
    \end{bmatrix} + \eps\rmM
\end{equation}By repeating the exact same steps for the second head and letting $\weights_\val^2 = \begin{bmatrix}
    -\Iden &\zero &\zero &\zero\\
    \zero &\zero &\zero &\zero 
\end{bmatrix}$ and adding back the residual we get the desired result.
\end{proof}
A slightly different implementation can be found in \cite{giannou2023looped}, in which the ReLU layers are also used, together with indicator vectors that define whether a position should be updated or not.
% This is the move operation in \cite{akyurek2022learning}. Notice that since \cite{akyurek2022learning} use decoder only models they can only move things only from $(i,j)$ to $(i',j')$ such that $i'>i$ and $j'>j$. However, since we consider encoder models by repeating their proof we can get the result for arbitrary positions. We also consider the positional encodings of \cite{giannou2023looped} which leads to the $\varepsilon$ error.

% \begin{remark} For the idea on how this is possible we illustrate a simple example. Assume that $\seqlen=5$ and we want to copy and element from position (1,2) where the first element signifies the row and the second the column to position (3,5). The first step is to copy the second column to the fifth one; to do so we force the softmax matrix to acts as a permutation matrix, \emph{i.e.,} $[\ve_1\;\ve_2\;\ve_3\;\ve_4\;\ve_2]$ where $\ve_i$ is the one hot vector for the $i$-th row. This is achieved by the bianary representations by using as $\weights_\key\X = [\emb_1\;\hdots\;\emb_5]$, $\weights_\query\X = [\emb_1\;\hdots\;\emb_2]$ and by using $\lambda$ large enough the result follows. Then, we use a value matrix that only keeps the row $1$ and permutes it to the row 3. Finally, to account for the residual, we use an extra head, which softmax is approximately the identity and the value weight matrix is zero expect for the first row which is again permuted to the third one. This description follows the construction of \cite{akyurek2022learning}.
% \end{remark}
\begin{lemma}\label{lem:addition}
   There exists a 1-hidden layer feedforward, ReLU network, with $8d$ activations in the hidden layer and $d$ neurons in the output layer that when given two $d$-dimensional binary vectors representing two non-negative integers, can output the binary vector representation of their sum, as long as the sum is less than $2^{d+1}$. 
   % Further, there is a special bit indicator that is 1 for the scratchpad and 0 for the rest of the input. If this bit is 1 then the operation above is performed, otherwise 0 is returned.
\end{lemma}
% \begin{lemma}\label{lem:transpose}
% Fix $\epsilon >0$ and consider an input of the following form
% \begin{equation}
%     \X = \left[\begin{array}{c|c|c|cc}
%         \rmA  &\zero&\zero &\dots & \zero \\
%  \zero&      \zero&\zero &\dots & \zero\\
%          \emb_{1:d}&\emb_{1:d}&\emb_{1:d}& \dots  &\emb_{1:d}\\
%          \embed_1'&\embed_{2}'&\embed_{3}'& \dots  &\embed_{d}'
%     \end{array}\right]. \nonumber
% \end{equation}
% where $\rmA\in\R^{d\times d}$; then there exists transformer-based function block with 4 layers, 1 head and dimensionality $r = 2d+2\log d = O(d)$ that outputs the following matrix
% \begin{equation}
%     \X = \left[\begin{array}{c|c|c|cc}
% \rmA' &\rmA'&\rmA'&\dots &\rmA' \\
%         \zero&\zero&\zero &\dots &\zero \\
%          \emb_{1:d}& \emb_{1:d}& \emb_{1:d}& \dots &\emb_{1:d} \\
%         \embed_1'&\embed_{2}'&\embed_{3}'& \dots  &\embed_{d}'
%     \end{array}\right]. \nonumber
% \end{equation}
% where $\rmA' = \rmA^\top +\epsilon\rmM$, for some $\|\rmM\|\leq 1$. The error $\epsilon$ depends on the choice of the temperature $\lambda$, as it is  a consequence of the \texttt{read/write} operations.
% \end{lemma}
\begin{lemma}\label{lem:matrixmul_app}
Let $\rmA\in \R^{d\times m}$ and $\rmB \in \R^{d\times n}$. Then for any $\epsilon >0$ there exists a transformer-based function block with  2 layers, 1 head and width  $r = O(d)$ that  outputs the multiplication $\rmA^\top\rmB +\eps\rmM$, for some $\|\rmM\|\leq 1$ .
\end{lemma}
% \begin{lemma}[Lemma 7 in \cite{giannou2023looped}]\label{lem:matrixmul_app} 
% Let $\rmA\in \R^{k\times m}$ and $\rmB \in \R^{k\times n}$. Then for any $\epsilon >0$ there exists a transformer-based function block with  2 layers, 1 head and \red{dimensionality $r = O(d)$} that  outputs the multiplication $\rmA^\top\rmB +\eps\rmM$, for some $\|\rmM\|\leq 1$ .\YL{What do the $r$ and $d$ stand for here?}
% \end{lemma}
\begin{remark}
    Notice that based on the proof of this lemma, the matrices/scalars/vectors need to be in the same rows, \emph{i.e.},
$\rmQ =\begin{bmatrix}
            \rmA &\rmB 
        \end{bmatrix}
$. By also appending the appropriate binary encodings we can move the output at any specific place we choose as in Lemma \ref{lem:copy}. Also, following the proof of the paper the input matrix is
\begin{equation*}
    \X = \begin{bmatrix}
        \rmQ& \zero &\zero \\
        \zero & \one\one^{\top} &\zero\\
        \Iden&\zero &\zero\\ 
        &\emb^{(1)}&\\
        &\emb^{(2)}&\\
        %\zero &\one^{\top} &\zero 
    \end{bmatrix}  
\end{equation*} 
where $\emb^{(1)}, \emb^{(2)}$ are chosen as to specify the position of the result.
\end{remark}
\subsection{Results on Filtering}\label{ss:filtering}
\begin{lemma}\label{lem:binary-app}
    Assume that the input to a transformer layer is of the following form
    \begin{equation}
        \X = \begin{bmatrix}
            \x_1 &\x_2 &\hdots&\x_{n-1} &\x_n\\
            \zero &\zero &\hdots &\zero &\zero\\
        b_1 & b_2 & \hdots &b_{n-1} &b_n\\
        b_1' &b_2' &\hdots &b_{n-1}' &b_n
        \end{bmatrix}
    \end{equation}
    where $b_i,b_i'\in\{0,1\}$, with zero indicating that the corresponding point should be ignored. Then there exists a transformer $\TF$ consisting only of a ReLU layer that performs this filtering, \textit{i.e.}, 
    \begin{equation}
        \TF(\X) = \begin{bmatrix}
           \one\{b_1\neq 0\}\x_1&\one\{b_2\neq 0\}\x_2 &\hdots &\one\{b_{n-1}\neq 0\}\x_{n-1} &\one\{b_n\neq 0\}\x_n\\
            \one\{b_1'\neq 0\}\x_1&\one\{b_2'\neq 0\}\x_2 &\hdots &\one\{b_{n-1}'\neq 0\}\x_{n-1} &\one\{b_n'\neq 0\}\x_n\\
            b_1 & b_2 & \hdots &b_{n-1} &b_n\\
            b_1' &b_2' &\hdots &b_{n-1}' &b_n
        \end{bmatrix}
    \end{equation}
\end{lemma}
\begin{proof}
    The layer is the following:
    \begin{equation}
       \TF(\x_i) = x_i + (-Cb_i - \x_i)_+ - (-Cb_i + \x_i)_+
    \end{equation}
    for some large constant $C$. Notice that if $b_i =1$ the output is just $\x_i$. But if $b_i = 0$ then the output is zero. For the second set instead of using the bits $b_i$,  we use the $b_i'$.
\end{proof}
\begin{remark}\label{rem:error-app} Notice that if some $b_i$ is instead of $1$, $1\pm \varepsilon$, $\varepsilon < c/C$. Then the output of the above layer would be \begin{align}
\TF(\x_i) &= \x_i + (-C \pm C\varepsilon - \x_i) -( -C \pm C\varepsilon +\x_i)\\
& = \x_i
\end{align}
while if some $b_i =\pm \varepsilon$ instead of zero, and assuming that $\x_i > c>0$ or $\x_i < -c< 0$ the output would be 
\begin{align}
    \TF(\x_i) &= \x_i  +(-C\varepsilon - \x_i)_+ - (-Cb_i +\x_i)_+\\
    &= \x_i \pm C\varepsilon - \x_i \\
    &\leq c
\end{align}
If $\abs{\x_i} \leq c$, again the output would be less than or equal to $c$, where $c$ can be arbitrarily small.
\end{remark}
Our target is to create these binary tokens, as to perform the filtering. 
  We now describe for clarity how next word prediction is performed. An $d\times N$-dimensional input is given to the transformer architecture, which is up-projected using the embedding layer;  the positional encodings are also added/appended to the input. At the last layer, the last token predicted is appended to the initial input to the transformer architecture.  The only difference of the new input including the positional encodings ( the input for the next iteration ) is the $n+1$-th token. This is a property that our construction maintains, \emph{i.e.,} the positional encodings used are oblivious to the prediction step performed and are always the same for each individual token. We consider encoder-based architectures in all of our lemmas below. 

In the subsequent lemma we  construct an automated process that works along those guidelines.
To do so we assume that the input to the transformer contains the following information:
\begin{enumerate}
    \item An enumeration of the tokens from $1$ to $N$. 
    \item The $\ln$ of the above enumeration.
    \item Zeros for the tokens that correspond to the data points, $1$ for each token that is the $\x_{\text{test}}$ or it is a prediction based on it.
    \item An enumeration $1$ to $L$ for each one of the data points provided. For example, if we are given three sets of data we would have : $1\hdots L \; 1\hdots L\; 1\hdots L$.
    \item Some extra information that is needed to implement a multiplication step as described in Lemma \ref{lem:matrixmul_app} and to move things to the correct place.
\end{enumerate}
The above information can be viewed as part of the encodings that are appended to the input of the transformer. Formally, we have
\begin{lemma}\label{lem:filtering-app}
Consider that a prompt with $n$ in-context samples is given and 
the $\ell-1$-th prediction has been made, and the transformer is to  predict the $\ell$-th one. Assume the input to the transformer is:
\renewcommand{\arraystretch}{1.5}

\begin{equation*}{\small
\arraycolsep=1.5pt
\X = \left[\begin{array}{*{12}c}
    \x_1 & \hdots &\s_1^{\ell-1}& \s_1^{\ell}&\hdots & \s_2^{\ell-1} &\s_2^{\ell} &\hdots & \s_{n}^{L} &\xtest &\hdots &\hat\tb^{\ell-1}\\
    \zero &\hdots &\zero  &\zero &\hdots &\zero &\zero &\hdots &\zero &\zero &\hdots &\zero\\
        \hline
     1 & \hdots & \ell & \ell+1& \hdots &\ell &\ell+1 &\hdots &L+1 &1&\hdots & \ell\\
     1 & \hdots & \ell &\ell+1 &\hdots &\ell &\ell+1&\hdots &L+1 &1&\hdots & \ell\\
     \hline
     1 & \hdots & \ell &\ell+1 &\hdots &L+\ell+1 &L+\ell+2 &\hdots &n(L+1) &n(L+1)+1 &\hdots & \seqlen\\
    \ln(1) & \hdots & \ln(\ell) &\ln(\ell+1) &\hdots &\ln(L+\ell+1) &\ln(L+\ell+2) &\hdots &\ln(n(L+1)) &\ln(n(L+1)+1) &\hdots & \ln(\seqlen)\\
 0 & \hdots & 0 &0&\hdots &0  &0 &\hdots &0 &1 &\hdots & 1\\
1 & \hdots & 1 &1 &\hdots &1&1 &\hdots &1 &1 &\hdots & 1\\ 
\m_1 & \hdots & \m_\ell &\m_{\ell+1} &\hdots &\m_{L+\ell+1} &\m_{L+\ell+2} &\hdots &\m_{n(L+1)} &\m_{n(L+1)+1} &\hdots & \m_{\seqlen}\\
 \end{array}\right]}
\end{equation*}
where $(\hat\tb^i)_{i=1}^{\ell-1}$ denote the first $\ell-1$ recurrent outputs of $\TF$ and for simplicity, let $\seqlen:=n(L+1)+\ell$ denote the total number of tokens. 
Then there exists a transformer $\TF$ consisting of 7 layers that has as output
\renewcommand{\arraystretch}{1.3}
\begin{equation*}
{\small
\arraycolsep=1.2pt
  \TF(\X)  =  \left[\begin{array}{*{12}c}
      \zero  & \hdots &\s_1^{\ell-1}& \zero&\hdots & \s_2^{\ell-1} &\zero &\hdots & \zero &\zero &\hdots &\hat\tb^{\ell-1}\\
      \zero & \hdots &\zero& \s_1^{\ell}&\hdots & \zero &\s_2^{\ell} &\hdots & \zero &\zero &\hdots &\zero\\
      \hline
     0 & \hdots & 1 &0 &\hdots &1 &0 &\hdots &0&0 &\hdots & 1\\
     0 & \hdots & 0 &1 &\hdots &0 &1 &\hdots &0 &0 &\hdots &0 \\
     \hline
     1 & \hdots & \ell &\ell+1 &\hdots &L+\ell+1 &L+\ell+2 &\hdots &n(L+1) &n(L+1)+1 &\hdots & N\\
       \ln(1) & \hdots & \ln(\ell) &\ln(\ell+1) &\hdots &\ln(L+\ell+1) &\ln(L+\ell+2) &\hdots &\ln(n(L+1)) &\ln(n(L+1)+1) &\hdots & \ln(\seqlen)\\
 0 & \hdots & 0 &0&\hdots &0  &0 &\hdots &0 &1 &\hdots & 1\\
1 & \hdots & 1 &1 &\hdots &1&1 &\hdots &1 &1 &\hdots & 1\\ 
\m_1 & \hdots & \m_\ell &\m_{\ell+1} &\hdots &\m_{L+\ell+1} &\m_{L+\ell+2} &\hdots &\m_{n(L+1)} &\m_{n(L+1)+1} &\hdots & \m_{\seqlen}\\
  \end{array}\right] }
\end{equation*}
with error up to $\delta\rmM$, where $\|{\rmM}\|\leq 1$ and $\delta>0$ is a constant that is controlled and can be arbitrarily small.
\end{lemma}
\begin{proof}
%Since the error of each step we use is controlled and can be arbitrarily small, we can omit it  in the sequence to make the proof simpler. Each one of the seven layers will induce an error of the form $\varepsilon_i \rmM_i$ with $\|\rmM_i\|\leq 1$; at the end the error is aggregated and we have a total error of $\sum_{i=1}^7\varepsilon_i \rmM_i $ with $\|\sum_{i=1}^7\varepsilon_i \rmM_i\| \leq \sum_{i=1}^7\varepsilon_i$. By setting $\varepsilon_i = \delta/7$ we get the desired error bound.

\noindent\textbf{Step 1: Extract the sequence length (1 layer).} Let 
\begin{align}
\weights_{\key}= \begin{bmatrix}
    \zero &\zero &0 &0 &0&1 &0 &0 &\zero
\end{bmatrix} 
 \;\;\; \weights_{\query} =  \begin{bmatrix}
    \zero &\zero &0 &0 &0 &0 &0 &0 &1 &\zero 
\end{bmatrix}
\end{align}
and thus 
\begin{equation}
    (\weights_{\key}\X)^\top(\weights_{\query}\X) = \begin{bmatrix}
        \ln(1) &\ln(1) &\hdots &\ln(1)\\
        \ln(2) &\ln(2) &\hdots &\ln(2)\\
        \vdots &\vdots &\ddots &\vdots\\
        \ln(\seqlen) &\ln(\seqlen) &\hdots &\ln(\seqlen)\\
    \end{bmatrix}.
\end{equation}
So after the softmax is applied we have
\begin{equation} \sfm((\weights_{\key}\X)^\top(\weights_{\query}\X)) = \begin{bmatrix}
        \dfrac{1}{\sum_{i=1}^\seqlen i} &\dfrac{1}{\sum_{i=1}^\seqlen i}&\hdots &\dfrac{1}{\sum_{i=1}^\seqlen i}\\
        \dfrac{2}{\sum_{i=1}^\seqlen i} &\dfrac{2}{\sum_{i=1}^\seqlen i} &\hdots &\dfrac{2}{\sum_{i=1}^\seqlen i}\\
        \vdots &\vdots &\ddots &\vdots\\
        \dfrac{\seqlen}{\sum_{i=1}^\seqlen i} &\dfrac{\seqlen}{\sum_{i=1}^\seqlen i}  &\hdots &\dfrac{\seqlen}{\sum_{i=1}^\seqlen i} \\
    \end{bmatrix}.
\end{equation}
We then set the weight value matrix as to zero-out all lines except for one line as follows
\begin{equation}
   \weights_{\val}\X =\begin{bmatrix}
   \zero &\zero & \hdots &\zero \\
    1 & 2 &\hdots &\seqlen\\
 \zero &\zero &\hdots &\zero 
   \end{bmatrix}.
\end{equation}
% \begin{equation}
%    \weights_{\val}\rmI =\begin{bmatrix}
%    \zero &\zero & \hdots &\zero \\
%   \zero &\zero & \hdots &\zero \\
%     \zero &\zero & \hdots &\zero \\
%     \zero &\zero & \hdots &\zero \\
%     1 & 2 &\hdots &\seqlen\\
%     0 &0 &\hdots &0\\
%      0 &0 &\hdots &0\\
%        0 &0 &\hdots &0
%    \end{bmatrix}
% \end{equation}
After adding the residual and using an extra head where the softmax returns identity matrix and the value weight matrix is minus the identity, we get attention output
\renewcommand{\arraystretch}{1.5}
\begin{equation*}
{\small
\arraycolsep = 1.5pt
    \left[\begin{array}{*{12}c}
    \x_1 & \hdots &\s_1^{\ell-1}& \s_1^{\ell}&\hdots & \s_2^{\ell-1} &\s_2^{\ell} &\hdots & \s_{n}^{L} &\xtest &\hdots &\hat\tb^{\ell-1}\\
    \zero &\hdots &\zero  &\zero &\hdots &\zero &\zero &\hdots &\zero &\zero &\hdots &\zero\\
        \hline
     1 & \hdots & \ell & \ell+1& \hdots &\ell &\ell+1 &\hdots &L+1 &1&\hdots & \ell\\
     1 & \hdots & \ell &\ell+1 &\hdots &\ell &\ell+1&\hdots &L+1 &1&\hdots & \ell\\
     \hline
     \dfrac{\sum_{i=1}^\seqlen i^2}{\sum_{i=1}^\seqlen i} & \hdots & \dfrac{\sum_{i=1}^\seqlen i^2}{\sum_{i=1}^\seqlen i} &\dfrac{\sum_{i=1}^\seqlen i^2}{\sum_{i=1}^\seqlen i} &\hdots &\dfrac{\sum_{i=1}^\seqlen i^2}{\sum_{i=1}^\seqlen i}&\dfrac{\sum_{i=1}^\seqlen i^2}{\sum_{i=1}^\seqlen i} &\hdots &\dfrac{\sum_{i=1}^\seqlen i^2}{\sum_{i=1}^\seqlen i} &\dfrac{\sum_{i=1}^\seqlen i^2}{\sum_{i=1}^\seqlen i} &\hdots & \dfrac{\sum_{i=1}^\seqlen i^2}{\sum_{i=1}^\seqlen i}\\
    \ln(1) & \hdots & \ln(\ell) &\ln(\ell+1) &\hdots &\ln(L+\ell+1) &\ln(L+\ell+2) &\hdots &\ln(n(L+1)) &\ln(n(L+1)+1) &\hdots & \ln(\seqlen)\\
 0 & \hdots & 0 &0&\hdots &0  &0 &\hdots &0 &1 &\hdots & 1\\
1 & \hdots & 1 &1 &\hdots &1&1 &\hdots &1 &1 &\hdots & 1\\ 
\m_1 & \hdots & \m_\ell &\m_{\ell+1} &\hdots &\m_{L+\ell+1} &\m_{L+\ell+2} &\hdots &\m_{n(L+1)} &\m_{n(L+1)+1} &\hdots & \m_{\seqlen}\\
 \end{array}\right].}
\end{equation*}
Notice that $\dfrac{\sum_{i=1}^\seqlen i^2}{\sum_{i=1}^\seqlen i} = \dfrac{2\seqlen(\seqlen+1)(2\seqlen+1)}{6\seqlen(\seqlen+1)} = \dfrac{2\seqlen+1}{3}$. We then use the ReLU layer as to multiply with $3/2$ and subtract $1$ from this column. This results to attention output
\begin{equation*}
{\small
\arraycolsep = 1.5pt
    \left[\begin{array}{*{12}c}
    \x_1 & \hdots &\s_1^{\ell-1}& \s_1^{\ell}&\hdots & \s_2^{\ell-1} &\s_2^{\ell} &\hdots & \s_{n}^{L} &\xtest &\hdots &\hat\tb^{\ell-1}\\
    \zero &\hdots &\zero  &\zero &\hdots &\zero &\zero &\hdots &\zero &\zero &\hdots &\zero\\
        \hline
     1 & \hdots & \ell & \ell+1& \hdots &\ell &\ell+1 &\hdots &L+1 &1&\hdots & \ell\\
     1 & \hdots & \ell &\ell+1 &\hdots &\ell &\ell+1&\hdots &L+1 &1&\hdots & \ell\\
     \hline
     \seqlen & \hdots & \seqlen &\seqlen&\hdots &\seqlen&\seqlen &\hdots &\seqlen&\seqlen&\hdots & \seqlen\\
    \ln(1) & \hdots & \ln(\ell) &\ln(\ell+1) &\hdots &\ln(L+\ell+1) &\ln(L+\ell+2) &\hdots &\ln(n(L+1)) &\ln(n(L+1)+1) &\hdots & \ln(\seqlen)\\
 0 & \hdots & 0 &0&\hdots &0  &0 &\hdots &0 &1 &\hdots & 1\\
1 & \hdots & 1 &1 &\hdots &1&1 &\hdots &1 &1 &\hdots & 1\\ 
\m_1 & \hdots & \m_\ell &\m_{\ell+1} &\hdots &\m_{L+\ell+1} &\m_{L+\ell+2} &\hdots &\m_{n(L+1)} &\m_{n(L+1)+1} &\hdots & \m_{\seqlen}\\
 \end{array}\right]}
\end{equation*}
Notice that this step does not involve any error.

\textbf{Step 2: Extract the identifier of the prediction to be made $\ell$ (3 layers). }Now, by setting the key and query weight matrices of the attention as  to just keep the all ones row, we get a matrix that attends equally  to all tokens. Then the value weight matrix keeps the row that contains $\ell$ ones and thus we get the number $\dfrac{\ell}{\seqlen}$, propagated in all the sequence length. Then the attention output is as follows:
\begin{equation}
    \arraycolsep = 1.5pt
    \left[\begin{array}{*{12}c}
    \x_1 & \hdots &\s_1^{\ell-1}& \s_1^{\ell}&\hdots & \s_2^{\ell-1} &\s_2^{\ell} &\hdots & \s_{n}^{L} &\xtest &\hdots &\hat\tb^{\ell-1}\\
    \zero &\hdots &\zero  &\zero &\hdots &\zero &\zero &\hdots &\zero &\zero &\hdots &\zero\\
        \hline
     1 & \hdots & \ell & \ell+1& \hdots &\ell &\ell+1 &\hdots &L+1 &1&\hdots & \ell\\
     1 & \hdots & \ell &\ell+1 &\hdots &\ell &\ell+1&\hdots &L+1 &1&\hdots & \ell\\
     \hline
     \seqlen & \hdots & \seqlen &\seqlen&\hdots &\seqlen&\seqlen &\hdots &\seqlen&\seqlen&\hdots & \seqlen\\
    \dfrac{\ell}{\seqlen} & \hdots & \dfrac{\ell}{\seqlen} &\dfrac{\ell}{\seqlen} &\hdots &\dfrac{\ell}{\seqlen} &\dfrac{\ell}{\seqlen} &\hdots &\dfrac{\ell}{\seqlen} &\dfrac{\ell}{\seqlen} &\hdots &\dfrac{\ell}{\seqlen}\\
 0 & \hdots & 0 &0&\hdots &0  &0 &\hdots &0 &1 &\hdots & 1\\
1 & \hdots & 1 &1 &\hdots &1&1 &\hdots &1 &1 &\hdots & 1\\ 
\m_1 & \hdots & \m_\ell &\m_{\ell+1} &\hdots &\m_{L+\ell+1} &\m_{L+\ell+2} &\hdots &\m_{n(L+1)} &\m_{n(L+1)+1} &\hdots & \m_{\seqlen}\\
 \end{array}\right].
\end{equation}
So far this step does not involve any error.
As described in Lemma \ref{lem:matrixmul_app} to implement multiplication of two values up to any error $\eps$, we need 1) have the two numbers to be multiplied next to each other and an extra structure (some constants, which we consider that are encoded in the last rows of the matrix $\m_i$) and 2) the corresponding structure presented in Lemma \ref{lem:matrixmul_app} . So we need to make the following transformation in the rows containing $\seqlen$ and $\dfrac{\ell}{\seqlen}$
\begin{align}
    \begin{bmatrix}
        \seqlen & \seqlen &\hdots &\seqlen\\
        \dfrac{\ell}{\seqlen} &\dfrac{\ell}{\seqlen} &\hdots &\dfrac{\ell}{\seqlen}
    \end{bmatrix} &\to \begin{bmatrix}
        \seqlen & \dfrac{\ell}{\seqlen} &\hdots &\seqlen\\
        \dfrac{\ell}{\seqlen} &\dfrac{\ell}{\seqlen} &\hdots &\dfrac{\ell}{\seqlen}
    \end{bmatrix}\\
    &\to 
    \begin{bmatrix}
         \ell& *&\hdots &*\\
        * &* &\hdots &*
    \end{bmatrix}\\
    &\to \begin{bmatrix}
        \ell& \ell&\hdots &\ell\\
         * &* &\hdots &*
    \end{bmatrix}
\end{align}
where $*$ denotes inconsequential values.
For the first step we assume that we have the necessary binary representations \footnote{the size that we need will be $2\log \seqlen_{\text{max}}+1$, where $\seqlen_{\text{max}}$ is the maximum sequence length} and then we use Lemma \ref{lem:copy} which shows how we can perform the operation of copy/paste. This step involves an error and the current output would be $\tilde{\X}_1 = \X^* +\varepsilon_1\M_1$, where $\varepsilon_1$ is controlled and we will determine it in the sequence, $\|\M_1\|\leq 1$ and $\X_1^*$ is the desired output. For the second step now we use Lemma \ref{lem:matrixmul_app} to perform the multiplication, this will affect some of the other values. 
To analyze the error of this step, we know that given $\X^*_1$ we will get the desired output $\X_2^*+ \varepsilon_2\M_2$, where $\|\M_2\|\leq 1$ and $\varepsilon_2$ is to be determined. However, are given $\tilde{X}_1$ in the multiplication procedure\footnote{Note that the true error is even smaller since the operation performed only affects one row.} which results in the following output $\tilde{\X}_2 = (\tilde{\X}_1)^\top \tilde{\X}_1 +\varepsilon_2\M_2 = {\X^*_1}^\top\X_1^* + \varepsilon_1\M_1^\top\X^*_1 + \varepsilon_1\M_1^\top\M_1 + \varepsilon_2\M_2$. 
Thus by choosing $\varepsilon_1, \varepsilon_2$ to be of order $\mathcal{O}(\delta)$we get that $\tilde{\X}_2 = \X^*_2 + \frac{\delta}{C}\|M\|$ for some $C$ chosen to be large enough.
Then for the last step consider the follow (sub-)rows of the matrix $\X$
\begin{equation}
    \begin{bmatrix}
        \ell & * &* &\hdots &*\\
        \emb_1 &\emb_1 &\emb_1 &\hdots &\emb_1\\
        \emb_1 &\emb_2 &\emb_3 &\hdots &\emb_\seqlen
    \end{bmatrix}
\end{equation}
where $\emb_i$ is the binary representation of position $i$. By choosing $\weights_\key,\weights_\query$ as to
\begin{equation}
    \weights_\key\X = \begin{bmatrix}
        \emb_1 &\emb_2 &\hdots &\emb_\seqlen
    \end{bmatrix}, \;\;\;\weights_\query\X = \begin{bmatrix}
        \emb_1 &\emb_1 &\hdots &\emb_1 
    \end{bmatrix} 
\end{equation}
and consider $\weights_\val\X = \begin{bmatrix}
        \ell & * & \hdots &*\\
        \zero &\zero &\hdots &\zero 
    \end{bmatrix}$ we have that 
    \begin{align*}
       \attn(\X) &= \X+ \weights_\val\X \sfm((\weights_\key\X)^\top\weights_\query\X)\\
       &= \X + \begin{bmatrix}
        \ell & * & \hdots &*\\
        \zero &\zero &\hdots &\zero 
    \end{bmatrix} \sfm\left(\begin{bmatrix}
        \emb_1^\top\emb_1 &\emb_1^\top\emb_1 &\hdots &\emb_1^\top\emb_1\\
         \emb_1^\top\emb_2 &\emb_1^\top\emb_2 & \hdots &\emb_1^\top\emb_2 \\
         \vdots &\vdots &\ddots &\vdots\\
         \emb_1^\top\emb_\seqlen &\emb_1^\top\emb_\seqlen &\hdots &\emb_1^\top\emb_\seqlen
    \end{bmatrix}\right)\\
    &= \X+ \begin{bmatrix}
        \ell & * & \hdots &*\\
        \zero &\zero &\hdots &\zero 
    \end{bmatrix}   \begin{bmatrix}
        1 & 1 &\hdots &1 \\
        0 & 0 &\hdots &0\\
        \vdots &\vdots &\hdots &\vdots\\
        0 & 0 &\hdots & 0
    \end{bmatrix} +\varepsilon_3 \rmM_3
    \end{align*}
    By subtracting one identity head for the first that we focus on as described in Lemma \ref{lem:copy} we have that $\attn(\X)$ results in the desired matrix. We output this result and overwrite $\ell/\seqlen$, thus we have
    \begin{equation}
    \arraycolsep = 1.5pt
    \left[\begin{array}{*{12}c}
    \x_1 & \hdots &\s_1^{\ell-1}& \s_1^{\ell}&\hdots & \s_2^{\ell-1} &\s_2^{\ell} &\hdots & \s_{n}^{L} &\xtest &\hdots &\hat\tb^{\ell-1}\\
    \zero &\hdots &\zero  &\zero &\hdots &\zero &\zero &\hdots &\zero &\zero &\hdots &\zero\\
        \hline
     1 & \hdots & \ell & \ell+1& \hdots &\ell &\ell+1 &\hdots &L+1 &1&\hdots & \ell\\
     1 & \hdots & \ell &\ell+1 &\hdots &\ell &\ell+1&\hdots &L+1 &1&\hdots & \ell\\
     \hline
     \seqlen & \hdots & \seqlen &\seqlen&\hdots &\seqlen&\seqlen &\hdots &\seqlen&\seqlen&\hdots & \seqlen\\
   \ell & \hdots & \ell &\ell &\hdots &\ell &\ell &\hdots &\ell &\ell &\hdots &\ell\\
 0 & \hdots & 0 &0&\hdots &0  &0 &\hdots &0 &1 &\hdots & 1\\
1 & \hdots & 1 &1 &\hdots &1&1 &\hdots &1 &1 &\hdots & 1\\ 
\m_1 & \hdots & \m_\ell &\m_{\ell+1} &\hdots &\m_{L+\ell+1} &\m_{L+\ell+2} &\hdots &\m_{n(L+1)} &\m_{n(L+1)+1} &\hdots & \m_{\seqlen}\\
 \end{array}\right]
\end{equation}

The new error introduced by this operation is again controlled as to be $\varepsilon_3\M_3 = \dfrac{\delta}{C}$, with $\|\M_3\|\leq 1$.
    \paragraph{Step 3: Create $\ell+1$ (0 layer). } We first copy the row with $\ell$ to the row with $N$, this can trivially be done with a ReLU layer that outputs zero everywhere else except for the row of $N$s that output $(\ell)_{+} - (N)_{+}$ to account also for the residual. We now use one of the bias terms (notice that $\ell$ is always positive) and set to one in one of the two rows that contain the $\ell$, again we account for the residual as before; everything else remains unchanged. Thus, we have
    \begin{equation}
    \arraycolsep = 1.5pt
    \left[\begin{array}{*{12}c}
    \x_1 & \hdots &\s_1^{\ell-1}& \s_1^{\ell}&\hdots & \s_2^{\ell-1} &\s_2^{\ell} &\hdots & \s_{n}^{L} &\xtest &\hdots &\hat\tb^{\ell-1}\\
    \zero &\hdots &\zero  &\zero &\hdots &\zero &\zero &\hdots &\zero &\zero &\hdots &\zero\\
        \hline
     1 & \hdots & \ell & \ell+1& \hdots &\ell &\ell+1 &\hdots &L+1 &1&\hdots & \ell\\
     1 & \hdots & \ell &\ell+1 &\hdots &\ell &\ell+1&\hdots &L+1 &1&\hdots & \ell\\
     \hline
     \ell+1 & \hdots & \ell+1 &\ell+1&\hdots &\ell+1 &\ell+1 &\hdots &\ell+1 &\ell+1 &\hdots &\ell+1\\
    \ell & \hdots & \ell &\ell &\hdots &\ell &\ell &\hdots &\ell &\ell &\hdots &\ell\\
 0 & \hdots & 0 &0&\hdots &0  &0 &\hdots &0 &1 &\hdots & 1\\
1 & \hdots & 1 &1 &\hdots &1&1 &\hdots &1 &1 &\hdots & 1\\ 
\m_1 & \hdots & \m_\ell &\m_{\ell+1} &\hdots &\m_{L+\ell+1} &\m_{L+\ell+2} &\hdots &\m_{n(L+1)} &\m_{n(L+1)+1} &\hdots & \m_{\seqlen}\\
 \end{array}\right]
\end{equation}
This operation can collectively be implemented (add the bias + copy the row) in the ReLU layer of the previous transformer layer that was not used in the previous step.
    \paragraph{Step 4: Create the binary bits (2 layers). }We will now use the information extracted in the previous steps, to create the binary indicators/bit to identify which tokens we want to filter. This can be easily implemented with one layer of transformer and especially the  ReLU part of it. Notice that if we subtract the row that contains the tokens that have already been predicted, \emph{i.e.,} $[\ell~\ell\hdots~\ell]$ from the row that contains $[1~2\hdots L~1~2\hdots L \hdots L]$ we will get zero only in the positions that we want to filter out and some non-zero quantity in the rest. This is trivially implemented with one ReLU layer. So, we need to implement an if..then type of operation. Basically, if the quantity at hand is zero we want to set the bit to one, while if it non-zero to set it to be zero. This can be implemented with the following ReLU part of a transformer layer 
    \begin{equation}
     \TF(x_i) = 1 - (x_i)_{+}-(-x_i)_{+} + (x_i -1 )_{+} +(-x_i -1)_{+} - ((x_i)_{+} - (-x_i)_{+})
    \end{equation}
    the last two terms are to account for the residual. Again the rest of the rows do not change and are zeroed-out. 
    
    \textbf{Step 5: Implement the filtering (1 layer).}  We now apply Lemma \ref{lem:binary-app} and our proof is completed. In this step as in Lemma \ref{lem:binary-app}, the error remains of the order of the error of the previous step. Thus, by fine-tuning the constants appropriately, depending on 1) the bound on the input as of Assumption \ref{asm:bounded}, the targeted error $\delta$ and the constant used in Lemma \ref{lem:binary-app} we achieve that the error has the target upped bound.
\end{proof}

\begin{theorem}[Theorem~\ref{thm:MLP} restated] Consider a prompt $\pmt_n(f)$ generated from an $L$-layer MLP $f(\cdot)$ as described in Definition \ref{MLP assump}, and assume given test example $(\xtest,\tb^1_{\text{\tiny test}}, \dots\tb^L_{\text{\tiny test}})$. 
% , containing $n$ examples and ends with $\x^{\texttt{test}}$. 
For any resolution $\eps>0$, there exists $\delta=\delta(\eps)$, iteration choice $T={\cal{O}}(\kappa_{\max}^2\log(1/\eps))$, and a backend transformer construction $\TFBE$ such that the concatenated transformer $\TF=\TFLR \circ \TFBE$ implements the following: Let $(\tht^{i})_{i=1}^{\ell-1}$ denote the first $\ell-1$ $\COT$ outputs of $\TF$ and set $\pb[\ell]=(\pb_n(f),\xtest,\tht^{1}\dots\tht^{\ell-1})$. At step $\ell$, $\TF$ implements
\begin{enumerate}
\item \textbf{Filtering.} Define the filtered prompt with input/output features of layer $\ell$,
\[
\arraycolsep = 1pt
\pbf_n=\left(\begin{array}{*{10}c}
\hdots \zero,&\tb^{\ell-1}_{1},&\zero&\hdots&\zero,&\tb^{\ell-1}_n, &\zero&\hdots&\zero,&\tht^{\ell-1}\\
\hdots\zero,&\zero,& \tb^{\ell}_{1}&\dots&\zero,&\zero,& \tb^{\ell}_n&\dots&\zero,&\zero\end{array}\right).
\]
There exists a fixed projection matrix $\bPi$ that applies individually on tokens such that the backend output obeys $\tn{\bPi(\TFBE(\pb[\ell]))-\pbf_n}\leq \delta$. %The more precise statement is provided in Lemma \ref{lem:filtering}. % up to $\eps$ accuracy via a fixed linear projection. 
\item \textbf{Gradient descent.} The combined model obeys $\tn{\TF(\pb[\ell])-\tb_{\text{\tiny test}}^{\ell}}\leq {\ell\cdot \eps}/{L}$.
\end{enumerate}
$\TFBE$ has constant number of layers independent of $T$ and $n$. Consequently, after $L$ rounds of $\COT$, \TF outputs $f(\xtest)$ up to $\eps$ accuracy.
\end{theorem}
\begin{proof}
    We apply Lemma \ref{lem:filtering-app} from which it is clear that there exists a projection such that the result stated in (1. Filtering) holds, and it is independent to $T$, $n$ and $\ell$. Next we turn to prove (2. Gradient descent). 
    In Definition~\ref{MLP assump}, we assume that the network's activation function is leaky-ReLU, \emph{i.e.,} 
\begin{equation}
     \phi(x) =   \left\{\begin{matrix}
 x,& \text{ if }x\geq 0 \\ 
 \alpha x,& \text{ otherwise.} 
\end{matrix}\right.
    \end{equation}
    Thus, as a first step we construct the inverse of leaky-ReLU and apply it in the second row of $\p_n^{\text{filter}}$ where the inverse of leaky-ReLU is 
    \begin{equation}
         \phi^{-1}(y) =   \left\{\begin{matrix}
 y,& \text{ if }y\geq 0 \\ 
y/\alpha,& \text{ otherwise.} 
\end{matrix}\right.
    \end{equation}
    This can be implemented with the following activation function (denoted by $\sigma(\cdot)$) using ReLUs:
    \begin{equation}
        \sigma(x) = (x)_+ - 1/\alpha (-x)_+.
    \end{equation}
    After it, it remains $\TFLR$ to solve linear regression problems. Taking $\ell$th layer, first neuron prediction as an example, and letting $\x'_i:=\tb_i^{\ell-1}$, $y'_i:=\phi^{-1}(\tb_i^{\ell}[0])$ and $\w=\W_\ell[0]$, linear regression has form of $y'_i=\w^\top\x'_i$ for $i\in[n]$. Notice that the extra zeros do not contribute in the update performed by gradient descent, and after gradient descent has been performed, we apply back the leaky ReLU. Then 
    following  Assumption~\ref{assume oracle}, since we assume $\TFLR$ performs the same as gradient descent optimizer, given matrix condition as described in Definition~\ref{MLP assump}, after running $T$ iterations of gradient descend on the linear regression problem and considering each layer prediction with resolution $\eps/L$, we can get that $\|\TFLR(\pmt[\ell])-\bar\tb^\ell\|\leq\eps/L$, where $\bar\tb^\ell=\phi(\W_\ell\hat\tb^{\ell-1})$ is the correct prediction if taking $\hat\tb^{\ell-1}$ as input. Then we have 
    \[
    \|\TFLR(\pmt[\ell]-\tbt^\ell)\|\leq\|\TFLR(\pmt[\ell])-\bar\tb^\ell\|+\|\W_\ell(\hat\tb^{\ell-1}-\tbt^{\ell-1})\|\lesssim\eps/L+\|\TFLR(\pmt[\ell-1]-\tbt^{\ell-1})\|.
    \]
    Let $\TFLR(\pmt[0])$ returns $\xtest$ and therefore $\|\TFLR(\pmt[0])-\xtest\|=0$. Combing results in that  $\|\TFLR(\pmt[\ell]-\tbt^\ell)\|\lesssim\ell\cdot\eps/L$. 
    Since from Lemma \ref{lem:filtering-app} we have that we can choose $\delta$ to be arbitrary. Let $\delta = \eps/L$, where $L$ is the total predictions we will make. Then it will result in $\|\TF(\pmt[\ell]-\tbt^\ell)\|\lesssim\ell\cdot\eps/L$, which completes the proof.
    
    % ; then from Assumption \ref{assume oracle} we run gradient descend on the input. 
    
    % We then apply this layer  to the labels as already mentioned. As a final step we add the two rows and use Assumption \ref{assume oracle} to run gradient descent. Notice that the extra zeros do not contribute in the update performed by gradient descent. After gradient descent has been performed, we apply back the leaky ReLU.  For the error bound, fix $\eps > 0$, from Lemma \ref{lem:filtering-app} we have that we can choose $\delta$ to be arbitrary. Let $\delta = \eps/L$, where $L$ is the total predictions we will make. Since from Assumption \ref{assume oracle} we do not induce any extra error, the error after $\ell$ predictions will be $\ell \delta = \ell\eps/L$.
\end{proof}

\section{Experimental Details}
In this section, we provide the implementation details of our experiments. 

\subsection{Model Evaluation} Recap the same setting as in Section~\ref{sec:train} and assume we have pretrained models with parameters $\hat\bt^{\EXP}$ and $\hat\bt^{\COT}$. Next we make predictions following Section~\ref{subsec cot}. Letting $\ell(\cdot,\cdot):\Yc\times\Yc\to\R$ be loss function, we can define test risks as follows.
\begin{align*}
    \Lc^{\EXP}(n)=\E_{(\x_i)_{i=1}^n,(f_\ell)_{\ell=1}^L}\left[\ell(\hat\y_n,f(\x_n))\right]~~\text{where}~~\hat\y_n=\TF(\pmt_n(f),\x_n;\hat\bt^{\EXP})
\end{align*}
and
\begin{align*}
    \Lc^{\COT}(n)=\E_{(\x_i)_{i=1}^n,(f_\ell)_{\ell=1}^L}\left[\ell(\hat\y_n,f(\x_n))\right]~~\text{where}~~\hat\y_n=\TF(\pmt_n(f),\x_n,\hat\tb^1\cdots,\hat\tb^{L-1};\hat\bt^{\COT}).
\end{align*}
Here, we use $\Lc(n)$ to define the test risk when given prompt with $n$ in-context samples. Then, results shown in Figures~\ref{fig:smallk}\&\ref{fig:largek}\&\ref{fig:diff gpt}\&\ref{fig:diff k}\&\ref{fig:d5} are test risks $\Lc(n)$ given $n\in[N]$. Following model training and evaluation, we can see that once loss function is the same for both training and predicting, $\EXP$ (as well as $\ICL$) accepts training risk $\Lc_{\text{\tiny train}}^{\EXP}=\frac{1}{N}\sum_{n=1}^N\Lc^{\EXP}(n)$. 

\subsection{Implementation} All the transformer experiments use the GPT-2 model \citep{radford2019language} and our codebase is based on prior works~\citep{garg2022can,wolf2019huggingface}. Specifically, the model is trained using the \texttt{Adam} optimizer with learning rate $0.0001$ and batch size $64$, and we train $500$k iterations in total for all $\ICL$, $\EXP$ and $\COT$ methods. Each iteration randomly samples inputs $\x$s and functions $f$s. We also apply curriculum learning over the prompt $n$ as \cite{garg2022can} did for 2-layer random MLPs (Sec.~\ref{sec:2nn}). For both training and testing, we use the squared error as the loss function, \ie $\ell(\hat\y,\y)=\|\hat\y-\y\|^2$ (or $\ell(\hat y,y)=(\hat y-y)^2$ for scalar).
% \begin{align*}
% \hat\bt^{\EXP}=\arg\min_{\bt}\E_{(\x_n)_{n=1}^N,(f_\ell)_{\ell=1}^L}\left[\frac{1}{N}\sum_{n=1}^N\ell(\hat\y_n,f(\x))\right]~~\text{where}~~\hat\y_n=\TF(\pmt_n(f),\x_n)%\label{erm cot i}
% \end{align*}
% and
% \begin{align*}
% \hat\bt^{\COT}=\arg\min_{\bt}\E_{(\x_n)_{n=1}^N,(f_\ell)_{\ell=1}^L}\left[\frac{1}{NL}\sum_{n=1}^N\sum_{\ell=1}^L\ell(\hat\tb_n^\ell,\tb_{n}^\ell)\right]~~\text{where}~~\hat\tb_n^\ell=\TF(\pmt_n(f),\x_n\cdots\tb_n^{\ell-1}). %\label{erm cot io}
% \end{align*}
\section{Additional Experimental Results}

\begin{figure}[t]
\centering
\subfigure[Test with noisy data]{
    \includegraphics[height=.25\columnwidth]{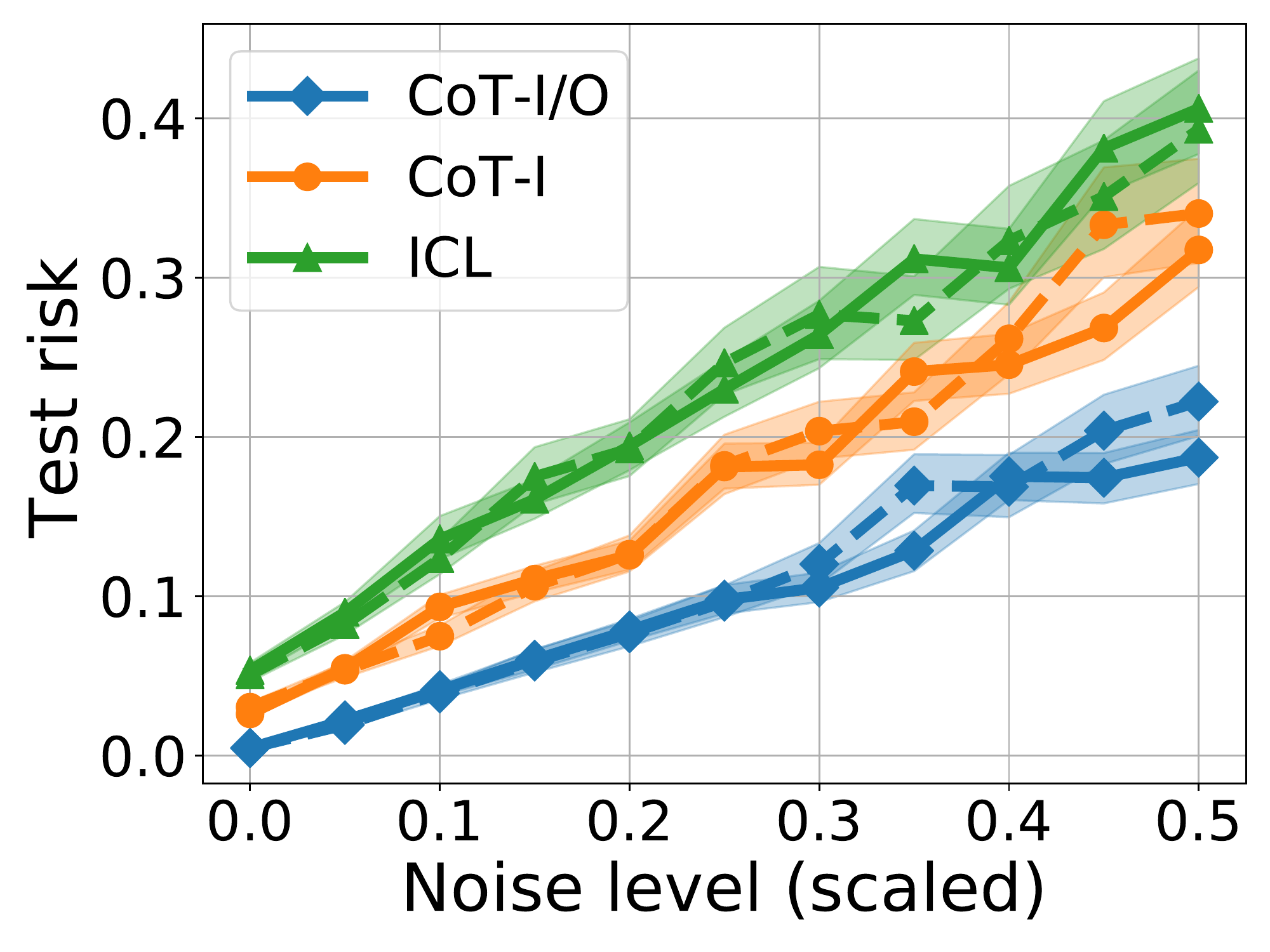}
    \label{fig:robust noise}
}
\hspace{-1mm}
\subfigure[Test with different $k$]{
    \includegraphics[height=.25\columnwidth,trim={1.5cm 0 0 0},clip]{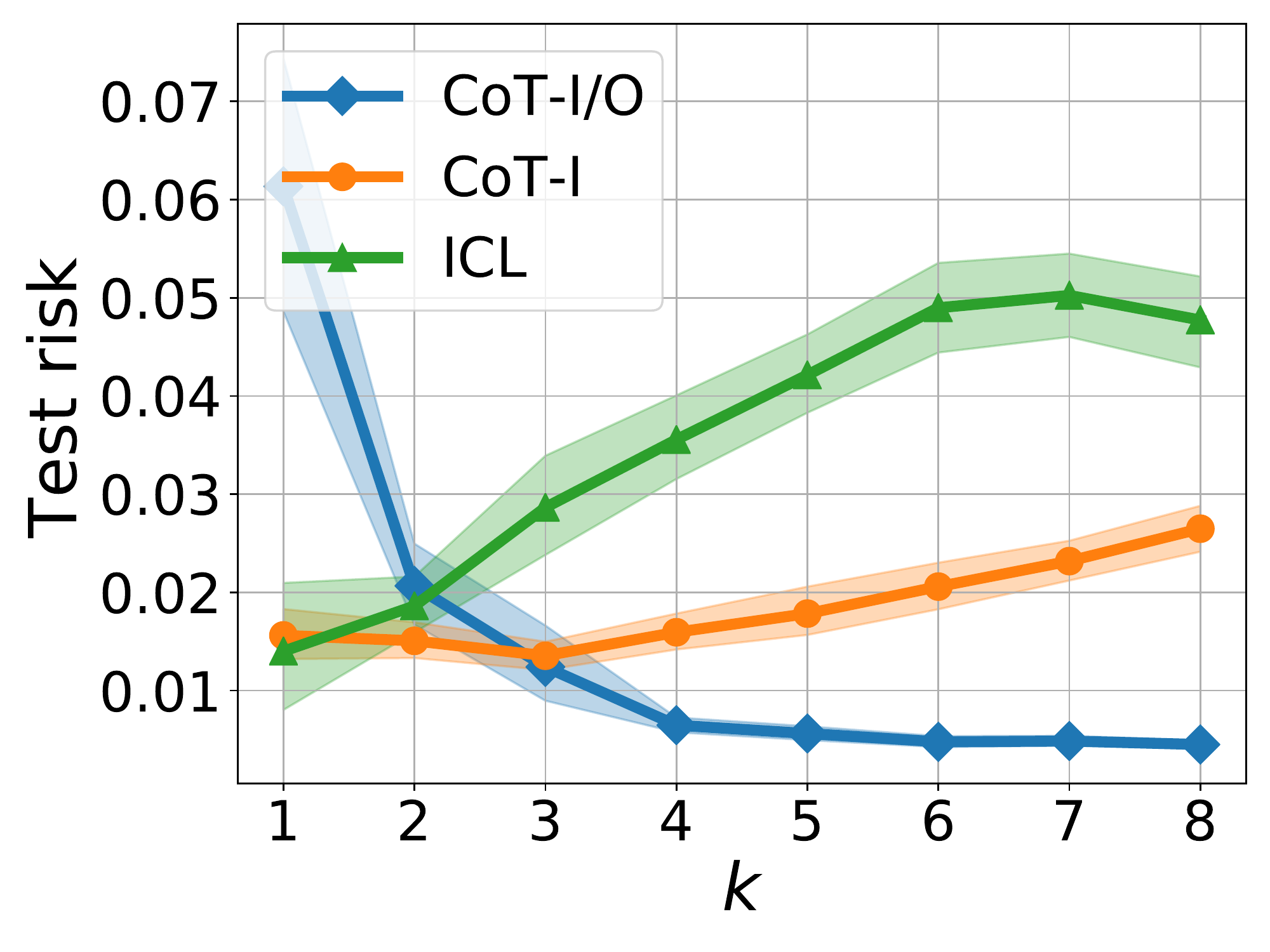}
    \label{fig:robust diff k}
}
\hspace{-1mm}
\subfigure[Test with different $d$]{
    \includegraphics[height=.25\columnwidth,trim={1.5cm 0 0 0},clip]{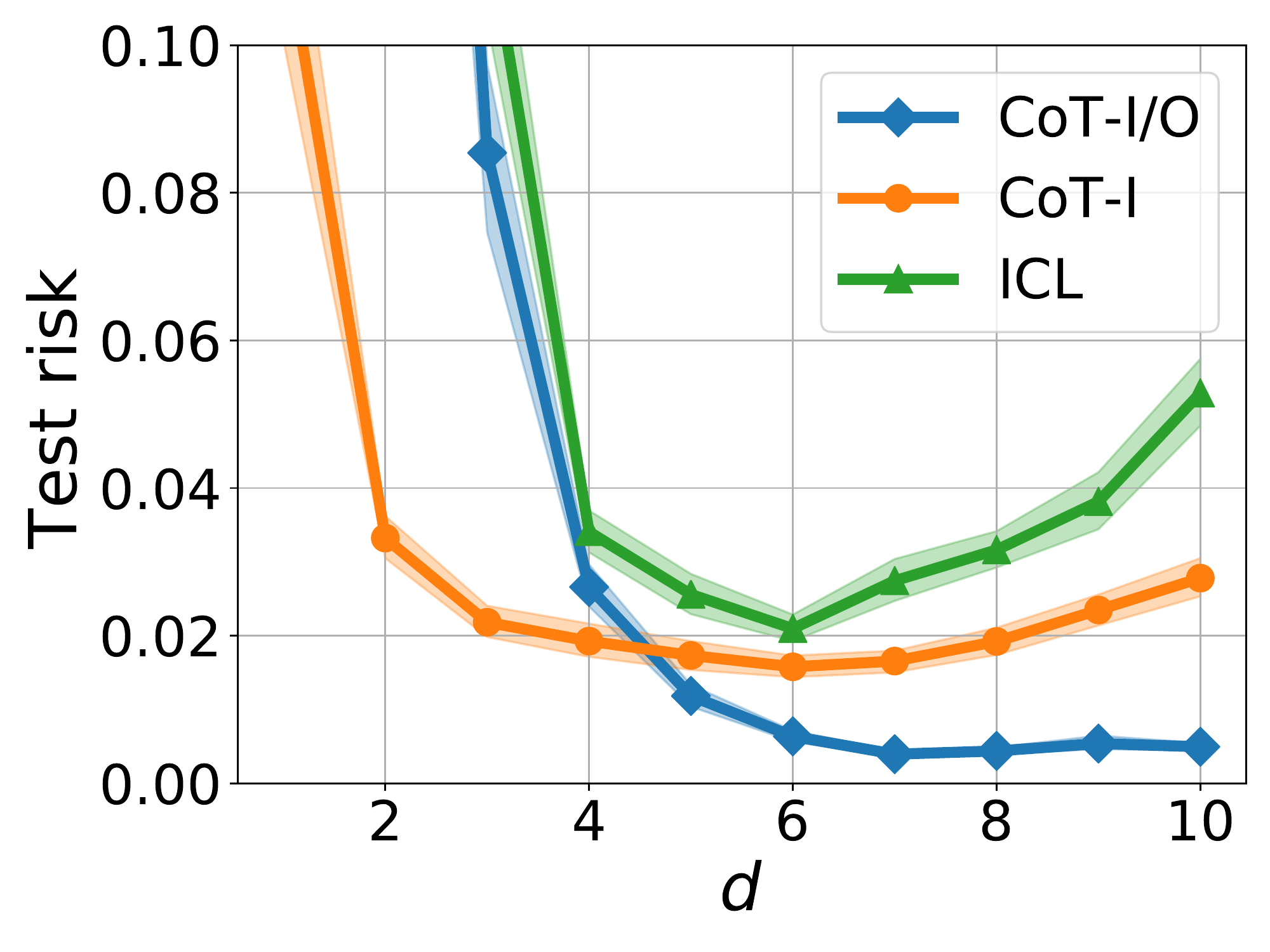}
    \label{fig:robust diff d}
}
\caption{We implement robustness experiments under different distribution shift levels. In Fig.~\ref{fig:robust noise} we add noise to the label $y$ (solid) or input features $\x$ (dashed). 
% This shows that CoT methods (trained on noiseless data) are fairly robust to noisy data at test-time.  
In Fig.~\ref{fig:robust diff k}, we in-context learn an MLP with $k\in[8]$ hidden nodes whereas transformer is trained for MLPs with $8$ hidden nodes. 
% This shows that CoT-I/O is robust to misspecification as long as it is small. However, if $k$ is very small, then CoT-I/O suffers more distribution shift. This makes sense because CoT-I/O relies more on hidden features compared to others. 
In Fig.~\ref{fig:robust diff k}, we consider misspecification of input dimension: TF is trained with $d=10$ but we feed a neural net with $d\leq10$. 
%This reveals that CoT is pretty robust and, in fact, CoT-I is the most robust to misspecification. We speculate this is because CoT-I makes similar use of input and hidden features.
}\label{fig:robust}

\end{figure}

\subsection{Out-of-distribution Experiments} \label{sec:robust}
While our primary focus lies in in-distribution scenarios, wherein the test examples are presumed to follow the same distribution as the training data, this subsection extends our analysis to out-of-distribution (OOD) settings. 

To provide a clearer understanding of our model's behavior under distribution shifts and to quantitatively assess the impact on test risk, we conduct experiments with varying levels of distribution shift. Specifically, we evaluate the test risk when the prompt is supplemented with $100$ in-context examples. The results of these experiments are depicted in Figure~\ref{fig:robust}, which serves to illustrate how our model's performance is affected under these OOD conditions.
% 
% To elucidate our results and examine the test risk under varying distribution shift levels, we plot test risks evaluated when prompt has $100$ in-context examples and results are presented in Figure~\ref{fig:robust}. 
In Fig.~\ref{fig:robust noise}, we analyze noisy in-context samples during testing. The solid and dashed curves represent the test risks, corresponding to the noisy in-context samples whose (input, output) takes the form of either $(\x,y+\text{noise})$ or $(\x+\text{noise},y)$, respectively. The results indicate that CoT exhibits greater robustness compared to ICL, and the test risks increase linearly with the noise level, with attributed to the randomized MLPs setting. Additionally, in Figs.~\ref{fig:robust diff k}\&\ref{fig:robust diff d}, we instead explore out-of-distribution test tasks where test MLPs differ in $(d,k)$ from the training phase. For both subfigures, we firstly train small GPT-2 using 2-layer MLPs with $d=10,k=8$. In Fig.~\ref{fig:robust diff k}, we fix $d=10$ and vary $k$ from $1$ to $8$, whereas in Fig.~\ref{fig:robust diff d}, we fix $k=8$ and vary $d$ from $1$ to $10$. In both instances, the findings reveal that CoT's performance remains almost consistent when $k\geq4$ or $d\geq6$, and ICL is unable to surpass it. The improved performance of ICL with smaller values of $d$ or $k$ again reinforces our central assertion that ICL requires $O(dk)$ samples for in-context learning of the 2-layer random MLP, and reducing either $d$ or $k$ helps in improving the performance. Given that we employ the ReLU activation function, smaller values of $d$ or $k$ can lead to significant bias in the intermediate feature. Consequently, CoT cannot derive substantial benefits from this scenario, resulting in a decline in performance. 
% To sum up, we thank the reviewer for raising this issue and will integrate OOD experiments in our revision.
\begin{figure}[t]
\centering
 \centering
   \subfigure[Small GPT-2, $d=10$, $k=4$]{
        \includegraphics[scale=.25]{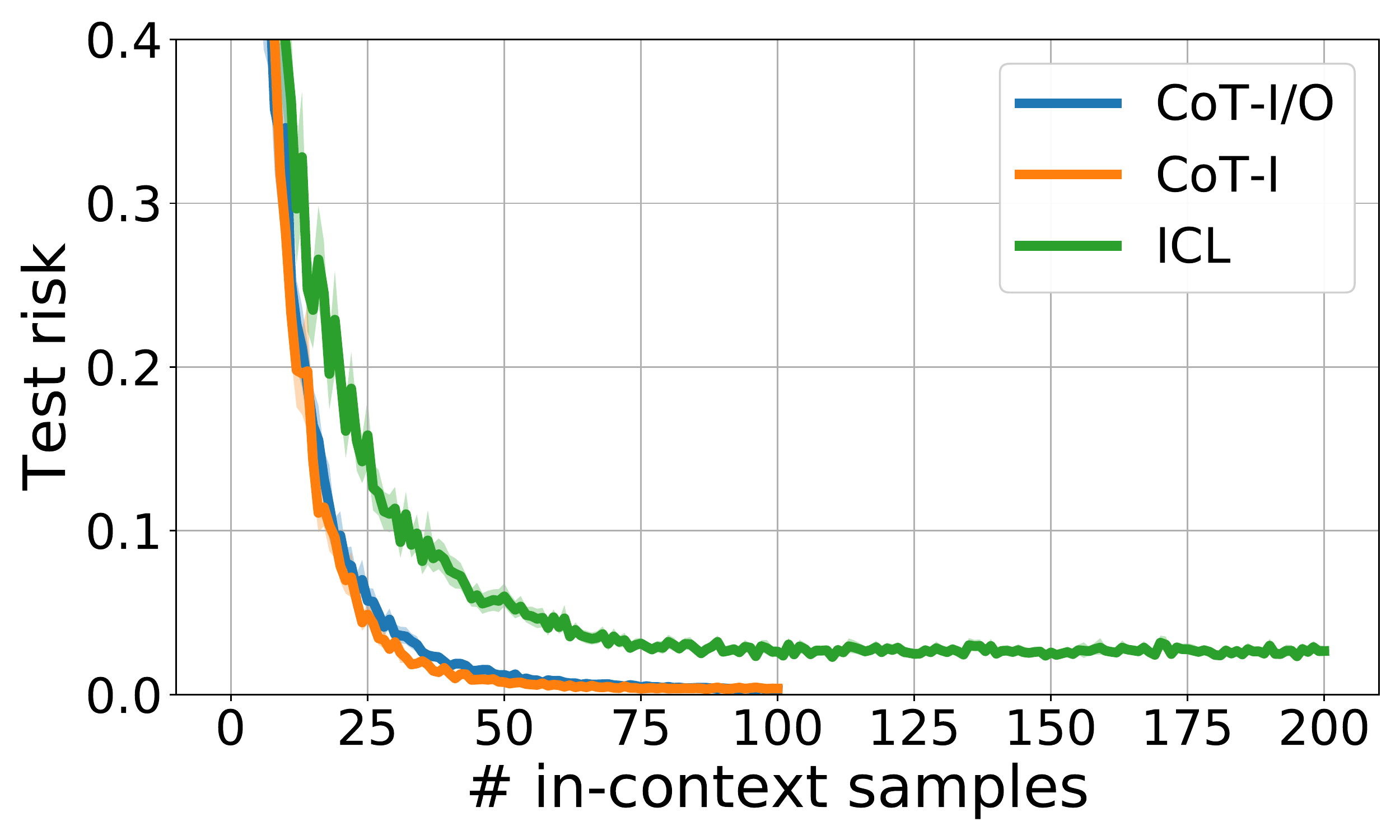}
        \label{fig:express small}
   }
    \hspace{10pt}
   \subfigure[Standard GPT-2, $d=10$, $k=4$]{
        \includegraphics[scale=.25]{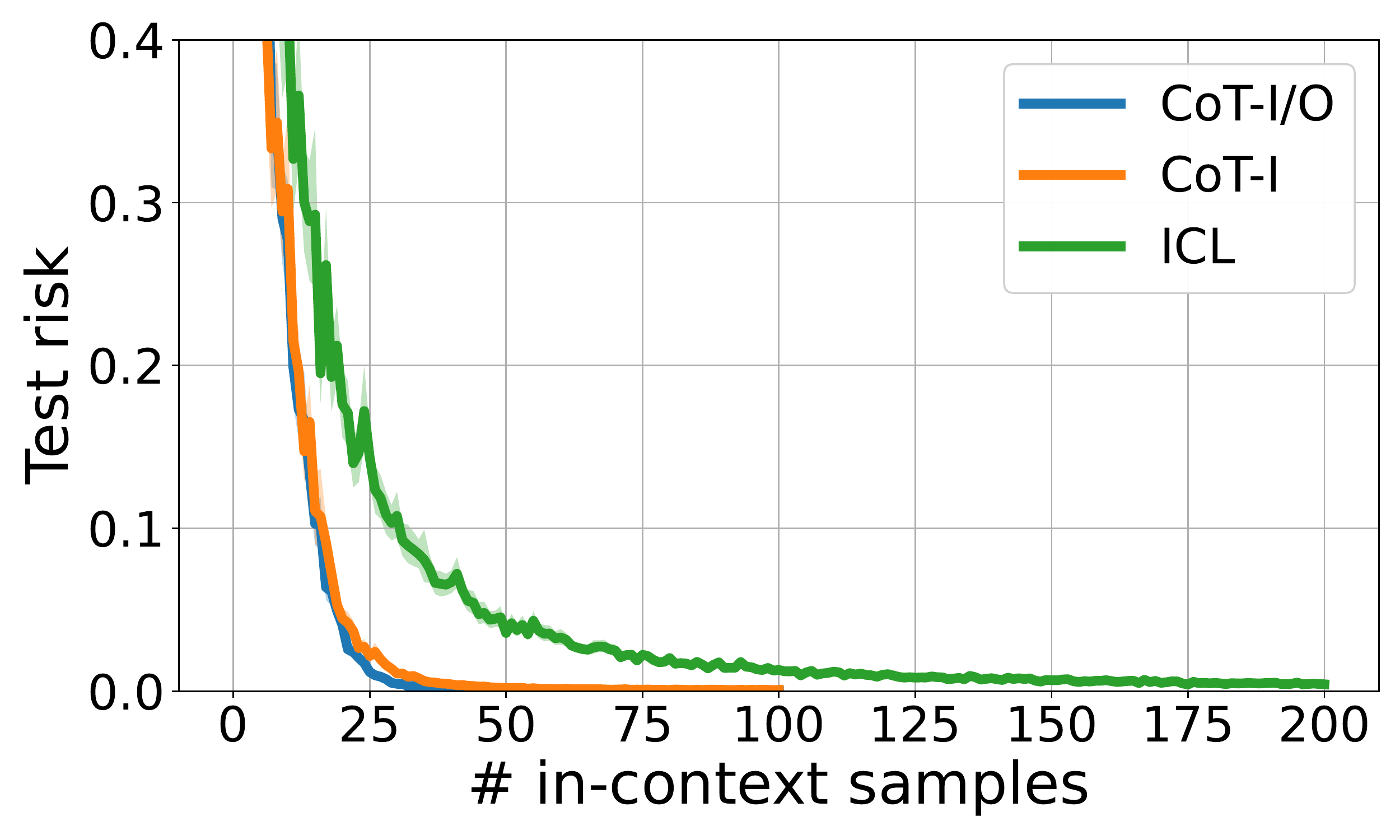}
        \label{fig:express stand}
   }
    \caption{We train ICL with more in-context examples. The main conclusion is that: For small GPT, ICL can indeed not approximate the neural net even with many examples (unlike CoT) whereas, for large GPT, ICL can do so (although much less efficient). This is in line with our theoretical intuitions on the expressivity benefits of CoT.}\label{fig:express}
\end{figure}
\subsection{Further Evidence of Model Expressivity}
% We appreciate the reviewer's query and the short answer is: NO (unless we enlarge the model size). 
There are two determinants of model learnability in in-context tasks: sample complexity and model expressivity. Sample complexity pertains to the number of samples needed to precisely solve a problem. However, when the transformer model is small, even with a sufficiently large number of samples, due to its lack of expressivity, ICL cannot achieve zero test risk. This contrasts with CoT, which decomposes complex tasks into simpler sub-tasks, thereby requiring smaller models for expression. Figure~\ref{fig:diff gpt} has illustrated the expressivity of different GPT-2 architectures, showing that the tiny GPT-2 model is too small to express even a single layer of 2-layer MLPs. Additionally, we have run more experiments, and the results are shown in Figure~\ref{fig:express}. Both Figures \ref{fig:express small} and \ref{fig:express stand} detail training models with MLP tasks of dimensions $d=10$ and $k=4$. In Fig.~\ref{fig:express small}, we use a small GPT-2 model, and the results show that the test risk stops decreasing even with more in-context examples. In Fig.~\ref{fig:express stand}, we train a larger model, and the results demonstrate that the standard GPT-2 is sufficient to express a 2-layer MLP with $d=10$ and $k=4$.
\begin{figure}[!t]
\centering
\subfigure[Test filtering on first layer prediction]{
    \includegraphics[height=.25\columnwidth]{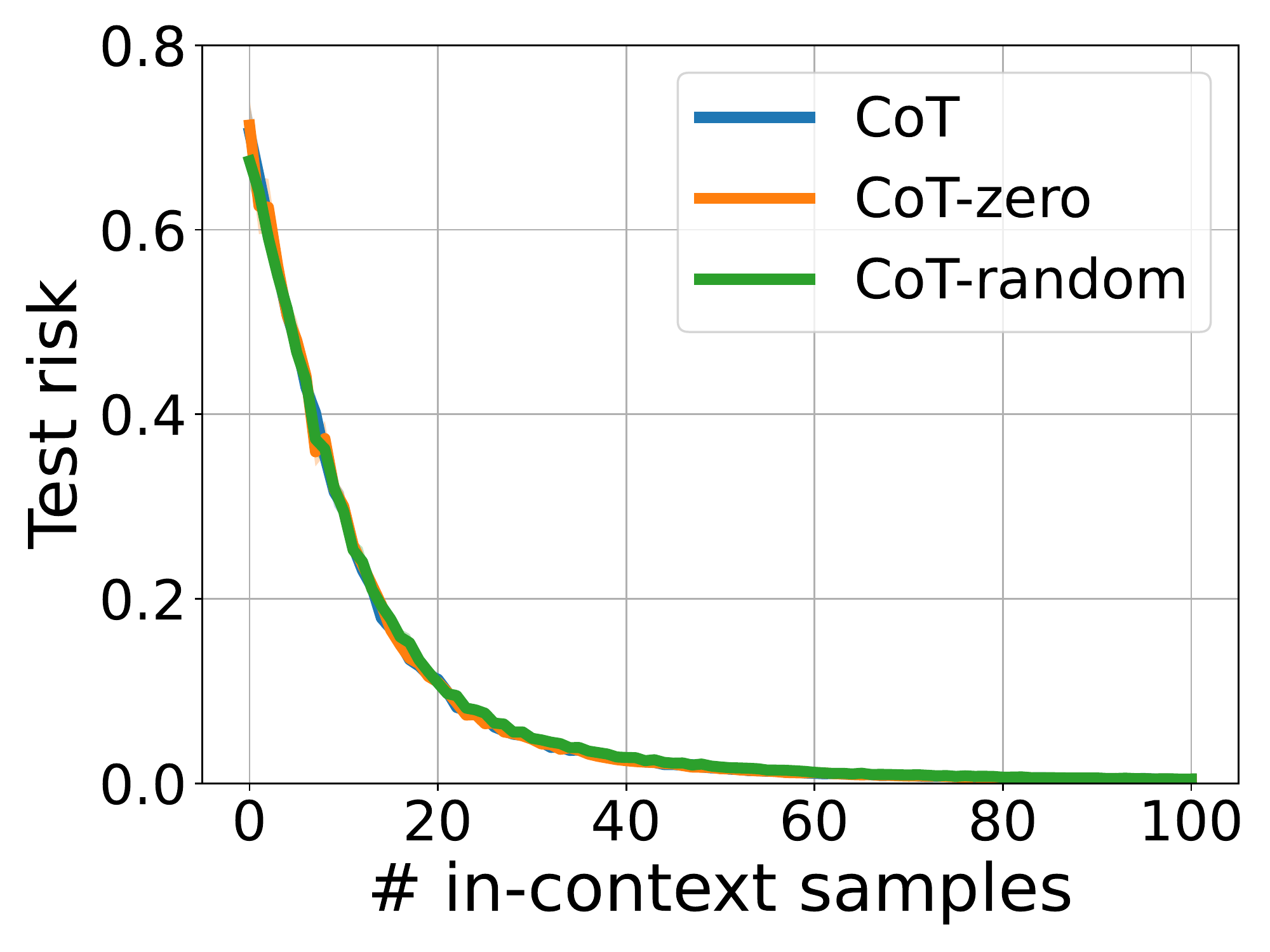}
    \label{fig:2nn filter}
    \hspace{10mm}
}
\subfigure[$\COT$: composed risk (same as Fig.~\ref{fig:cot_d10_vs_d20}) ]{
    \includegraphics[height=.25\columnwidth]{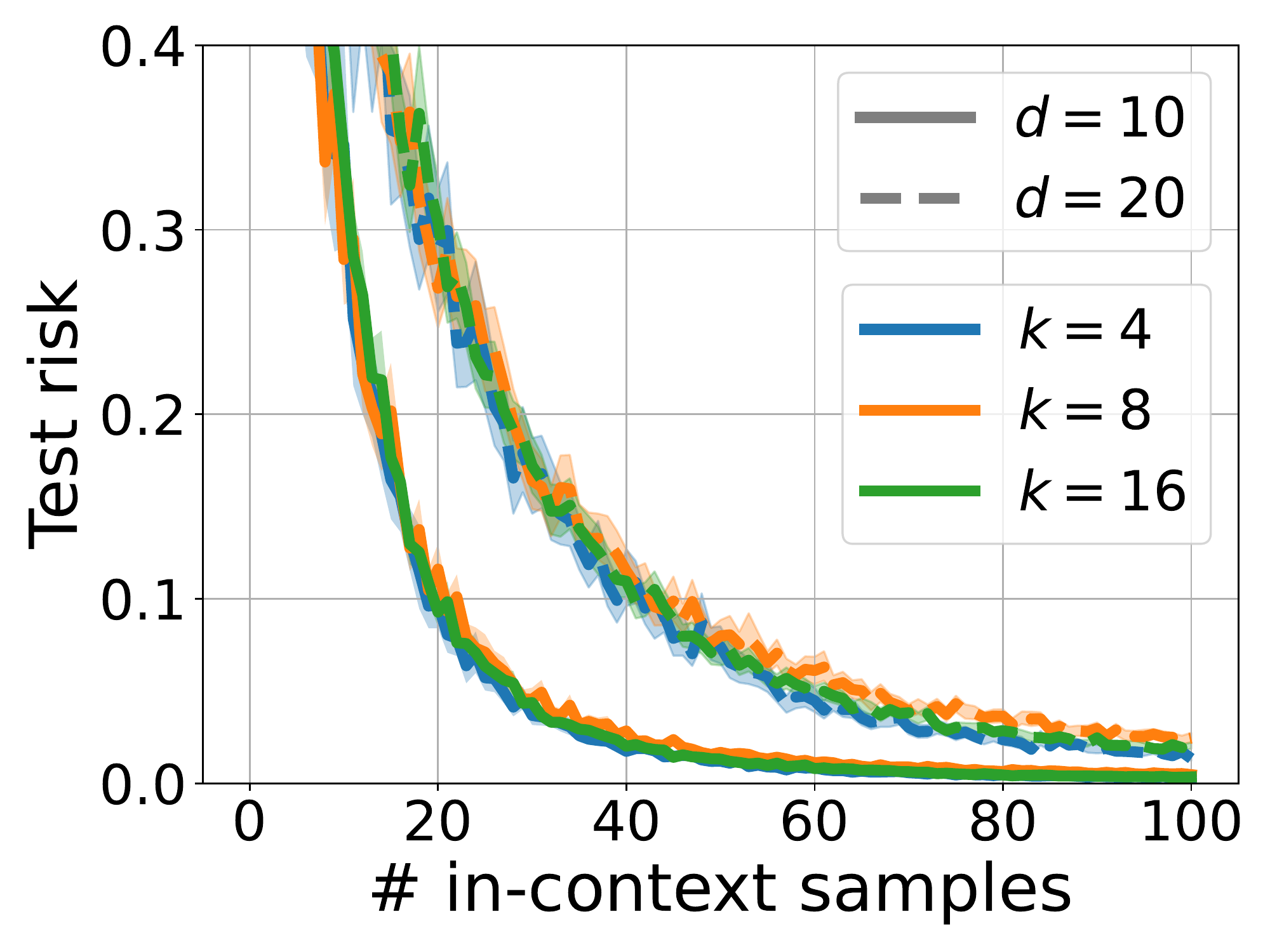}
    \label{fig:cot composed}
    \hspace{15mm}
}
\caption{Fig.~\ref{fig:2nn filter} presents a filtering evidence of 2-layer MLPs. Given a 2-layer MLP in-context example $(\x,\tb,y)$, CoT admits $(\x,\tb,y)$ as test sample; while test samples of CoT-zero and CoT-random are formed by $(\x,\tb,0)$ and $(\x,\tb,z)$ where $z\sim\Nc(0,d)$. Fig.~\ref{fig:cot composed} is directly cloned form Fig.~\ref{fig:cot_d10_vs_d20} with error bar for better comparison with results in Fig.~\ref{fig:cot vs icl}.}
% \label{fig:diff gpt comp}
\end{figure}

\begin{figure}[!t]
\centering
\subfigure[$\COT$: 1st layer prediction]{
    \includegraphics[height=.25\columnwidth]{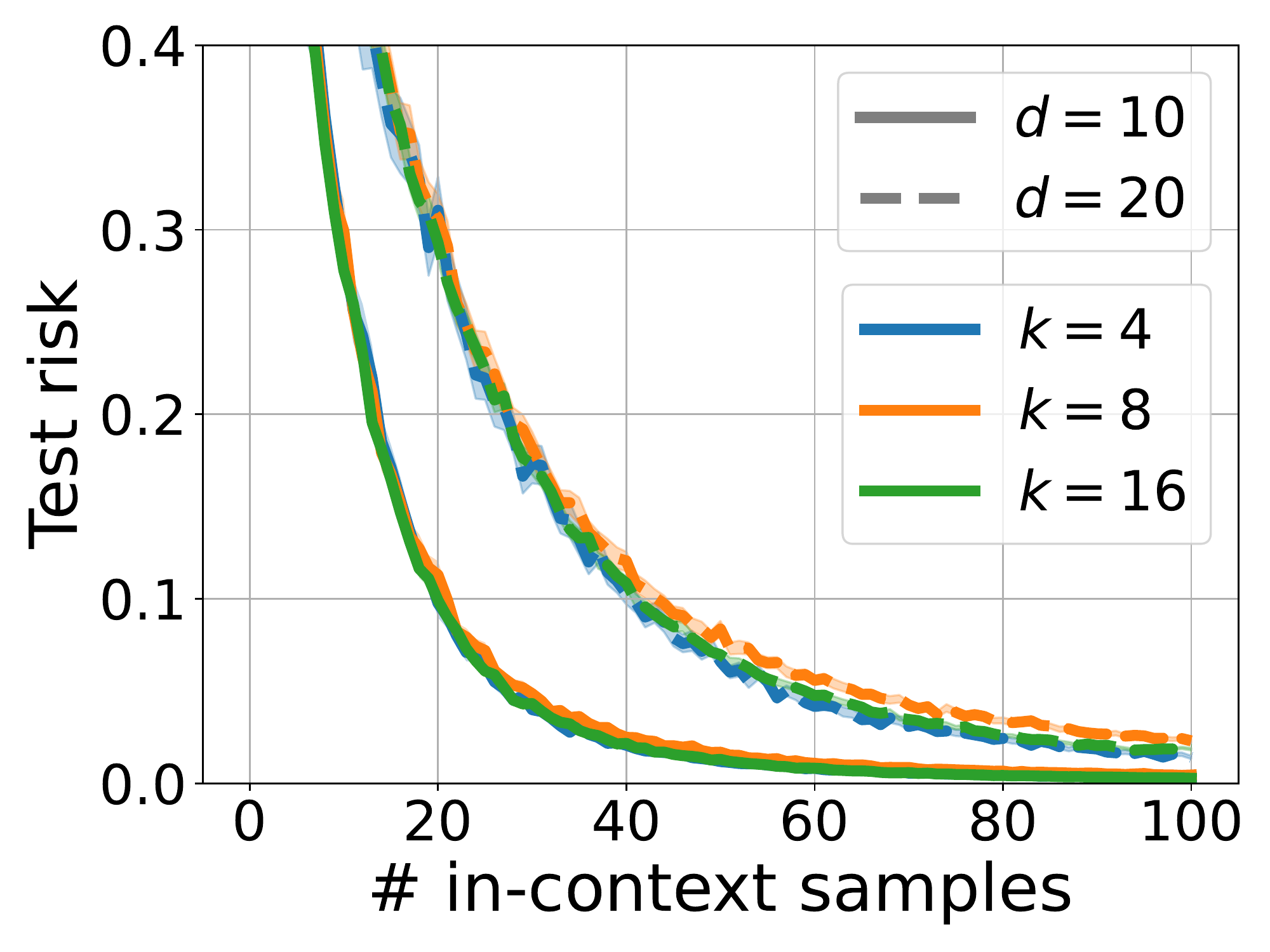}
    \label{fig:cot first}
    \hspace{10mm}
}
\subfigure[Train $\ICL$ with $(\x,\tb)$ pairs]{
    \includegraphics[height=.25\columnwidth]{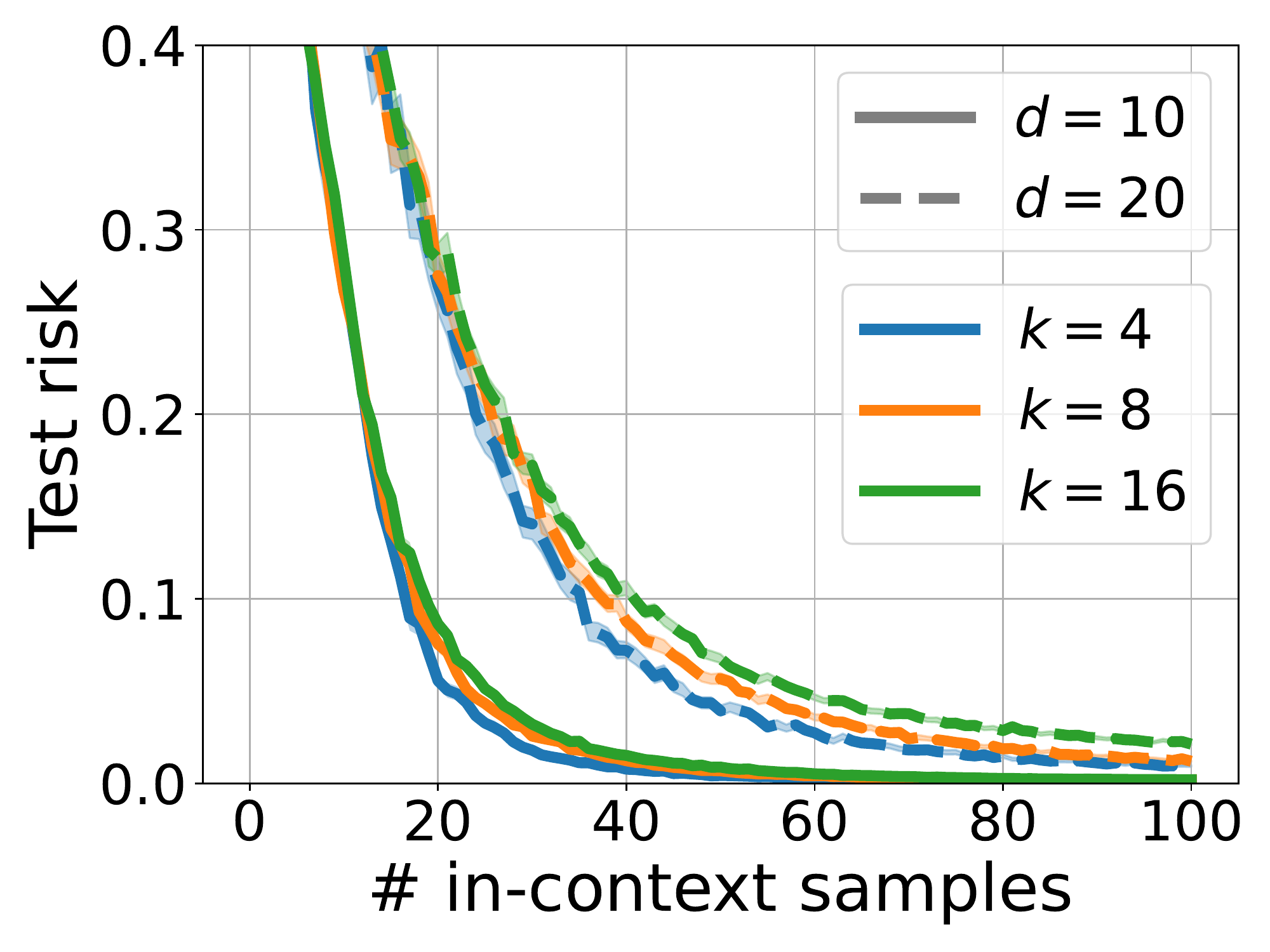}
    \label{fig:icl first}
    \hspace{10mm}
}
\subfigure[$\COT$: 2nd layer prediction]{
    \includegraphics[height=.25\columnwidth]{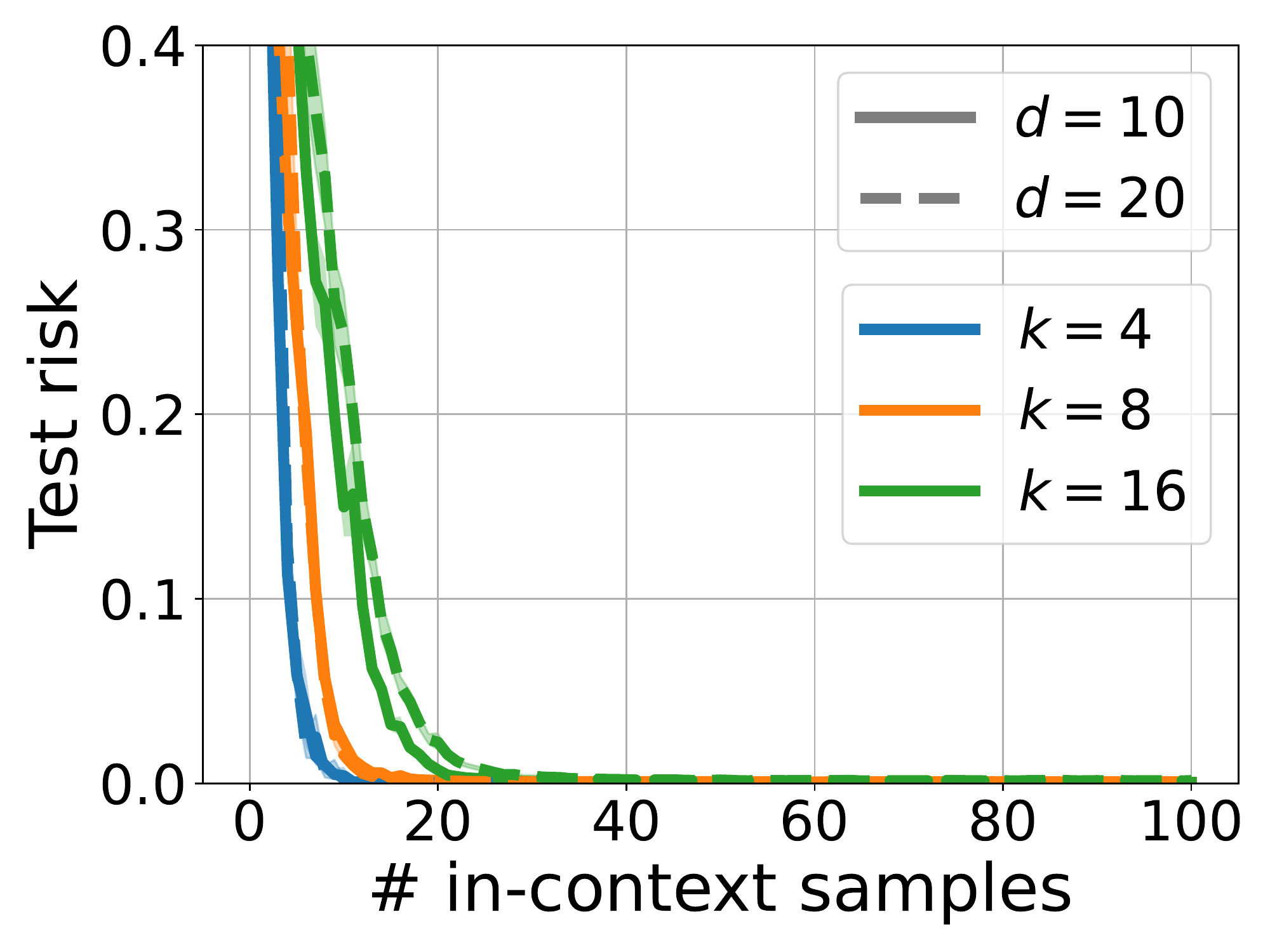}
    \label{fig:cot second}
    \hspace{10mm}
}
\subfigure[Train $\ICL$ with $(\tb,y)$ pairs]{
    \includegraphics[height=.25\columnwidth]{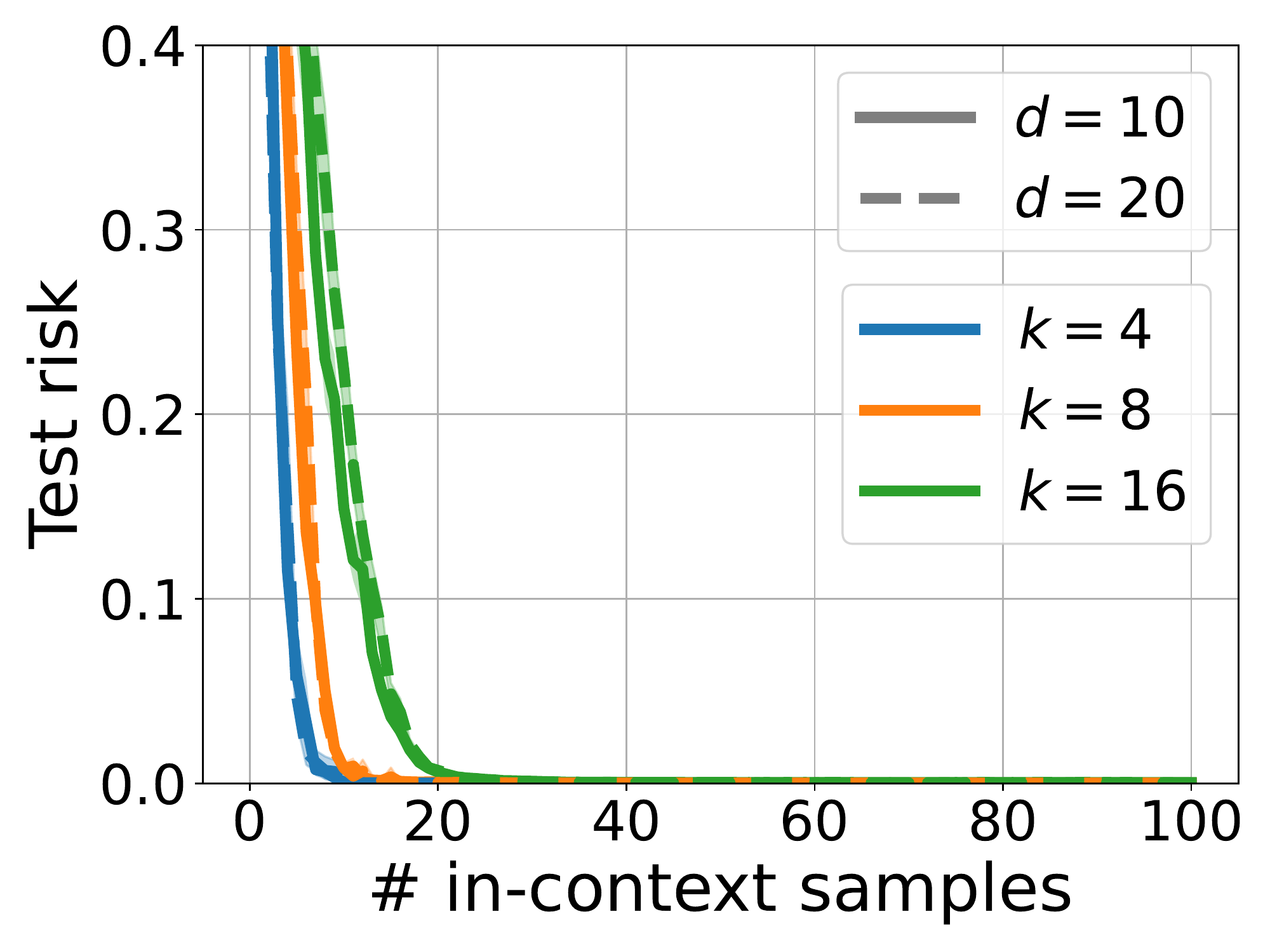}
    \label{fig:icl second}
    \hspace{10mm}
}
\caption{We compare the performance of filtered CoT and ICL. In Fig.~\ref{fig:cot first}\&\ref{fig:cot second}, we decouple the composed risk of predicting 2-layer MLPs into risks of individual layers (following Section~\ref{sec:emp}), which shows the filtered CoT results. In Fig.~\ref{fig:icl first}\&\ref{fig:icl second}, we train two additional models using ICL method taking $(\x,\tb)$ and $(\tb,y)$ as inputs.}
\label{fig:cot vs icl}
\end{figure}

\subsection{Filtering Evidence in 2-layer MLPs}\label{sec:2nn filter}
Section~\ref{sec:deep nn} has demonstrated the occurrence of filtering in the linear deep MLPs setting (black dotted curves in Fig.~\ref{fig:each layer 50k}). In this section, we present further empirical evidence based on the 2-layer MLPs setting discussed in Section~\ref{sec:2nn}.

Follow the same setting as Figure~\ref{fig:cot_d10_vs_d20} and choose $d=10$ and $k=8$. Assume we have a model pretrained using $\COT$ method. As described in Section~\ref{sec:2nn}, during training, the prompt consists of in-context samples in the form of $(\x,\tb,y)$ where $\tb=(\W\x)_+$ and $y=\vv^\top \tb$. To investigate filtering, we make three different predictions to evaluate the intermediate output, whose test prompts have in-context examples with the following forms:
\[
\text{CoT: }(\x,\tb,y),~~~~\text{CoT-zero: }(\x,\tb,0),~~~~\text{CoT-random: }(\x,\tb,z),
\]
where $z\sim\Nc(0,d)$. The results are displayed in Figure~\ref{fig:2nn filter} where blue, orange and green curves represent first layer prediction results using CoT, CoT-zero and CoT-random prompts, respectively. From this figure, we observe that the three curves are well aligned, indicating that when making a prediction for input $\x$, $\TF$ will attend only to $(\x,\tb)$ and ignore $y$. Therefore filling the positions of $y$ with any random values (or zero) will not change the performance of first layer prediction.  

\subsection{Comparison of Filtered CoT with ICL}
Until now, many experimental results have shown that $\COT$ provides benefits in terms of sample complexity and model expressivity compared to $\ICL$. As an interpretation, we state that CoT can be decoupled into two phases: \emph{Filtering} and \emph{ICL}, and theoretical results have been provided to prove this statement. As for the empirical evidence, Sections~\ref{sec:deep nn} and \ref{sec:2nn filter} precisely show that filtering does occur in practice. In this section, we provide additional experiments to demonstrate that, after filtering, CoT performs similarly to ICL.

For convenience and easier comparison, we repeat the same results as Fig.~\ref{fig:cot_d10_vs_d20} in Fig.~\ref{fig:cot composed}, where $d\in\{10,20\}$, $k\in\{4,8,16\}$, and train with a small GPT-2. We again recap the data setting for the 2-layer MLP, where the in-context examples of CoT prompt are in the form of $(\x,\tb,y)$. Given that filtering happens, we make first and second layer predictions following Section~\ref{sec:emp} and results are presented in Fig.~\ref{fig:cot first} and Fig.~\ref{fig:cot second}, respectively. These results show the performances of the filtered CoT prompts. Next, we need to compare the performance with separate ICL training. To achieve this goal, we train a small GPT-2 model using $\ICL$ method with prompt containing $(\x,\tb)$ pairs (first layer). The test results are shown in Fig.~\ref{fig:icl first}. Additionally, in Fig.~\ref{fig:icl second}, we train another small GPT-2 model using ICL but with prompts containing $(\tb,y)$ pairs (second layer). By comparing Fig.~\ref{fig:cot first} and \ref{fig:icl first}, as well as Fig.~\ref{fig:cot second} and \ref{fig:icl second}, we observe that after filtering, $\COT$ achieves similar performance as individually training a single-step problem through the ICL phase. 

\begin{figure}[!t]
\centering
\subfigure[Varying GPT-2 layers]{
    \includegraphics[height=.25\columnwidth]{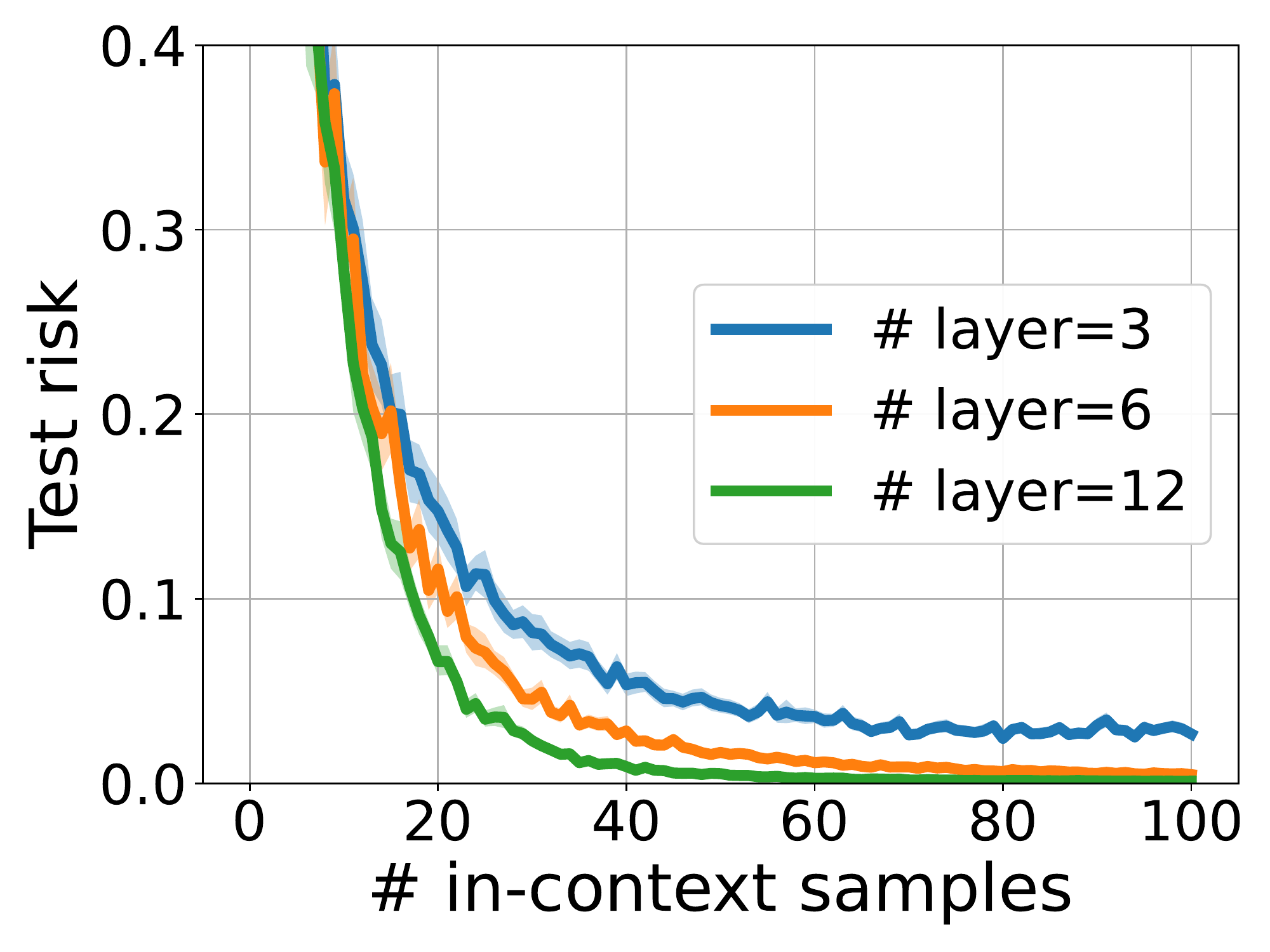}
    \label{fig:diff layer}
    \hspace{-1mm}
}
\subfigure[Varying GPT-2 heads]{
    \includegraphics[height=.25\columnwidth,trim={1.5cm 0 0 0},clip]{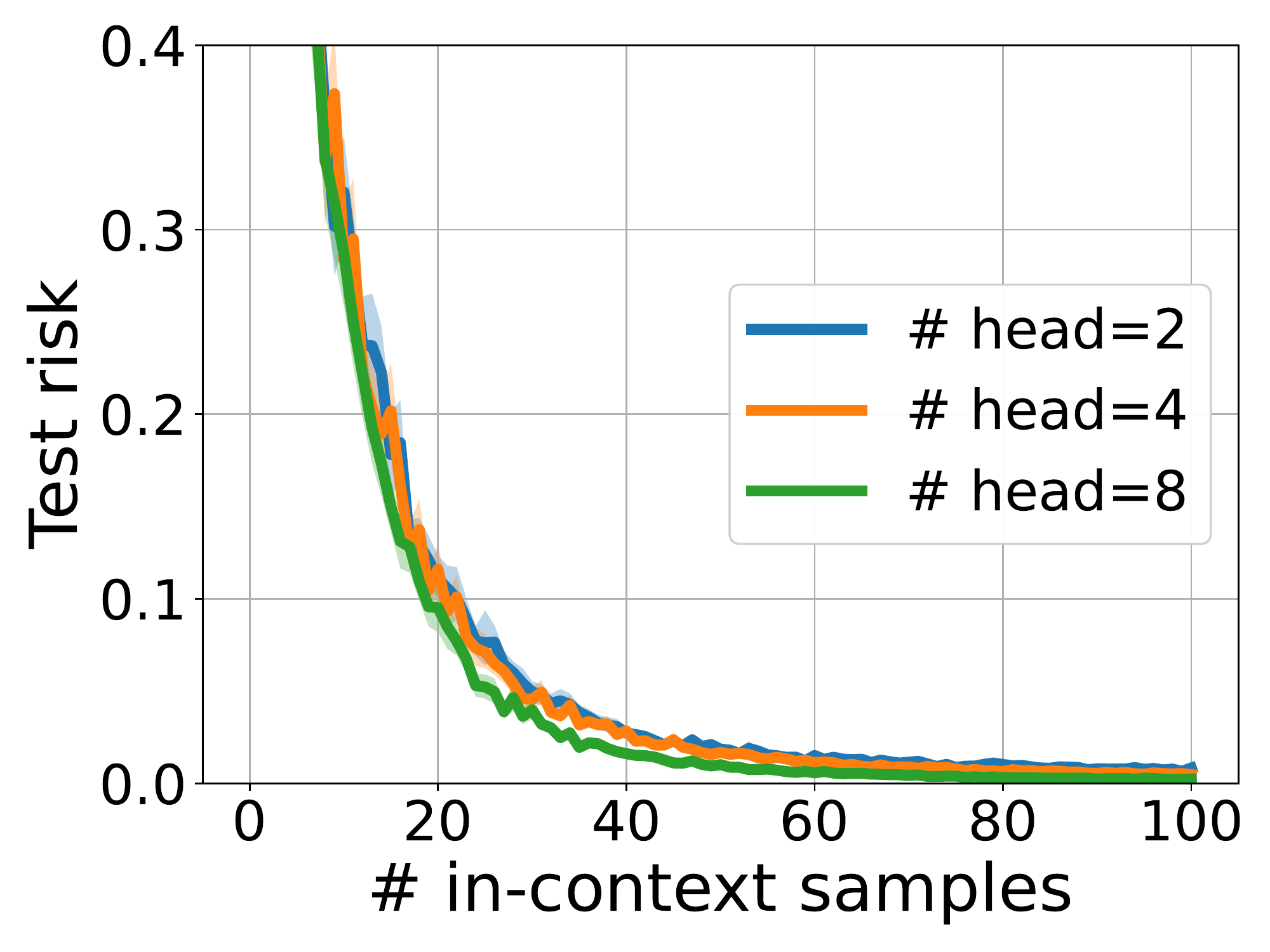}
    \label{fig:diff head}
    \hspace{-1mm}
}
\subfigure[Varying embedding dimensions]{
    \includegraphics[height=.25\columnwidth,trim={1.5cm 0 0 0},clip]{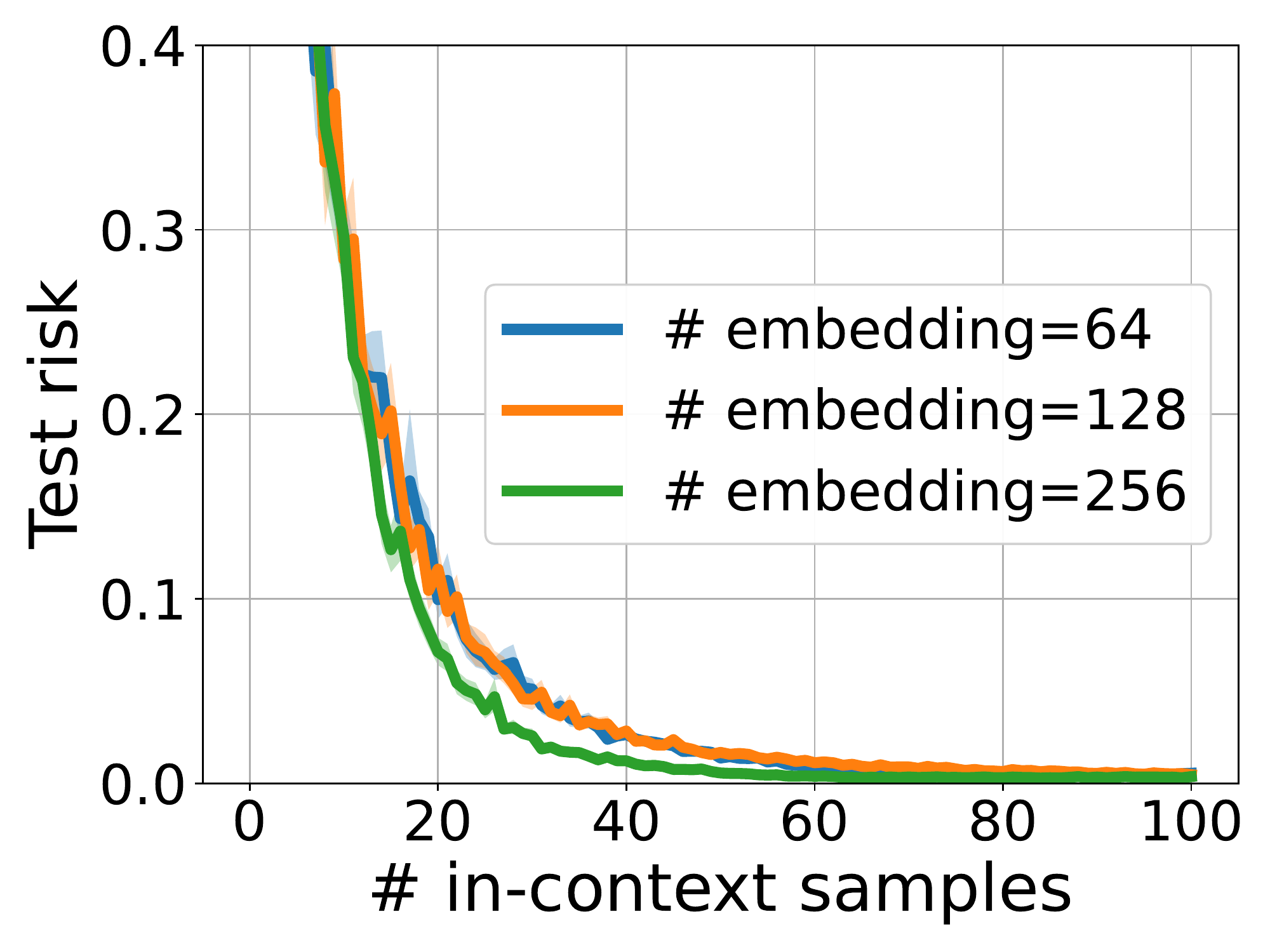}
    \label{fig:diff embed}
    \hspace{-1mm}
}
\caption{To further investigate how model architectures impact the prediction performance, we fix the number of heads and embedding dimension in Fig.~\ref{fig:diff layer} and change the layer number in $\{3,6,12\}$. Similarly for Fig.~\ref{fig:diff head}\&\ref{fig:diff embed} but instead, change number of heads (in $\{2,4,8\}$) and embedding dimensions (in $\{64,128,256\}$). }
\label{fig:diff gpt comp}
\end{figure}
\subsection{CoT across Different Sizes of GPT-2}
In Figure~\ref{fig:diff gpt}, we have demonstrated that larger models help in improving performance due to their ability of solving more complex function sets. However, since tiny, small and standard GPT-2 models scale the layer number, head number and embedding dimension simultaneously, it is difficult to determine which component has the greatest impact on performance. Therefore in this section, we  investigate how different components of transformer model affect the resulting performance by run $\COT$ on various GPT-2 architectures.

We maintain the same setting as in Section~\ref{sec:2nn}, fix $d=10$ and $k=8$, and consider a base GPT-2 model (small GPT-2) with $6$ attention layers, $4$ heads in each layer and $128$-dimensional embeddings. In Fig.~\ref{fig:diff layer}, we fix the number of heads at $4$ and the embedding dimension at $128$, while varying the number of layers in $\{3,6,12\}$. Similarly, we explore different models with different numbers of heads and embedding dimensions, and the results are respectively presented in Fig.~\ref{fig:diff head} and \ref{fig:diff embed}. Comparing them, we can observe the following: 1) once the problem is sufficiently solved, increasing the model size does not significantly improve the prediction performance (see Fig.~\ref{fig:diff head}\&\ref{fig:diff embed}); 2) the number of layers influences model expressivity, particularly for small GPT-2 architecture (see Fig.~\ref{fig:diff layer}).

\begin{figure}[!t]
\centering
\subfigure[1st layer: compare with GD solving ReLU]{
    \includegraphics[height=.25\columnwidth]{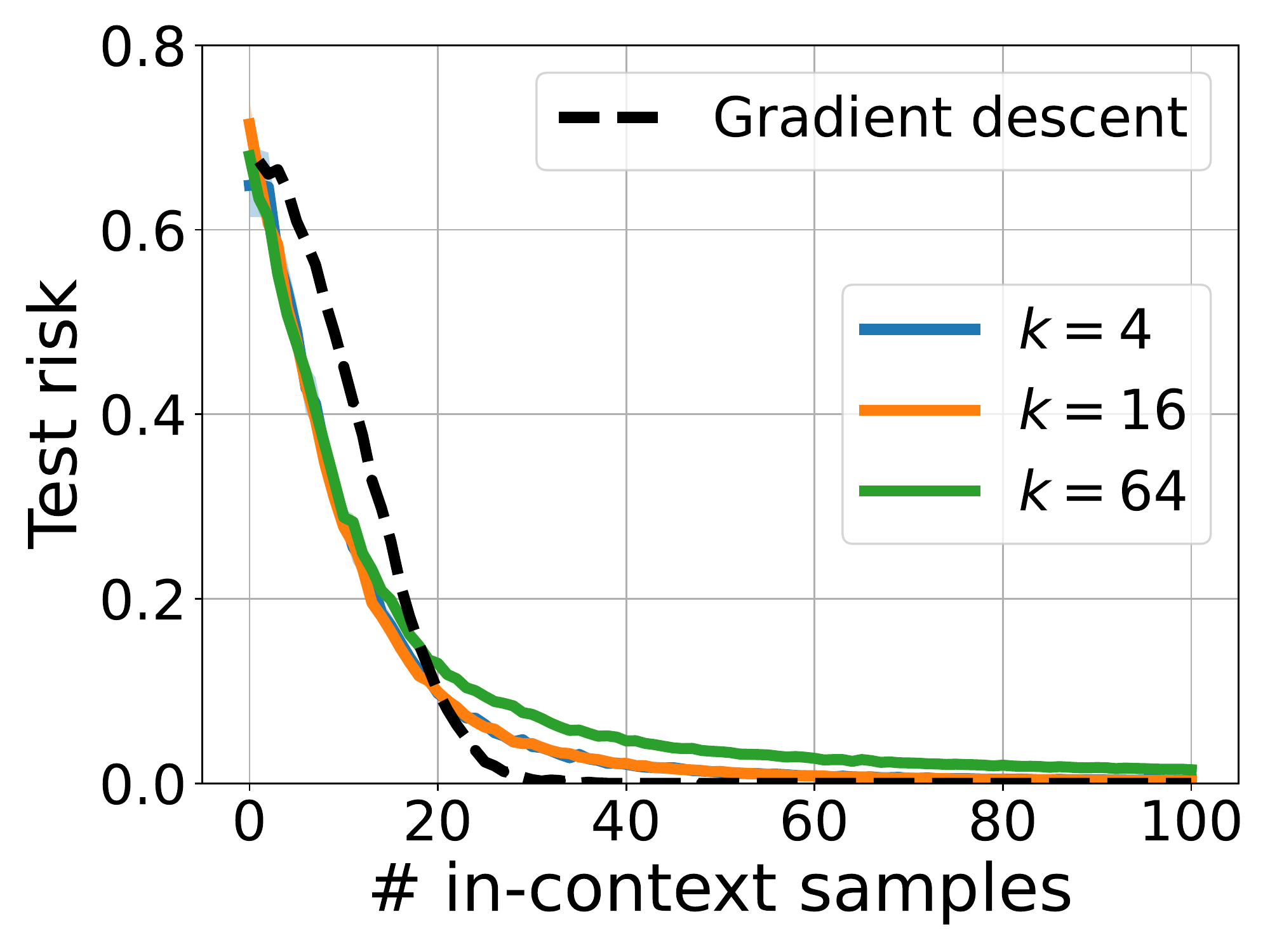}
    \label{fig:relu gd}
    \hspace{10mm}
}
\subfigure[2nd layer: compare with least squares]{
    \includegraphics[height=.25\columnwidth]{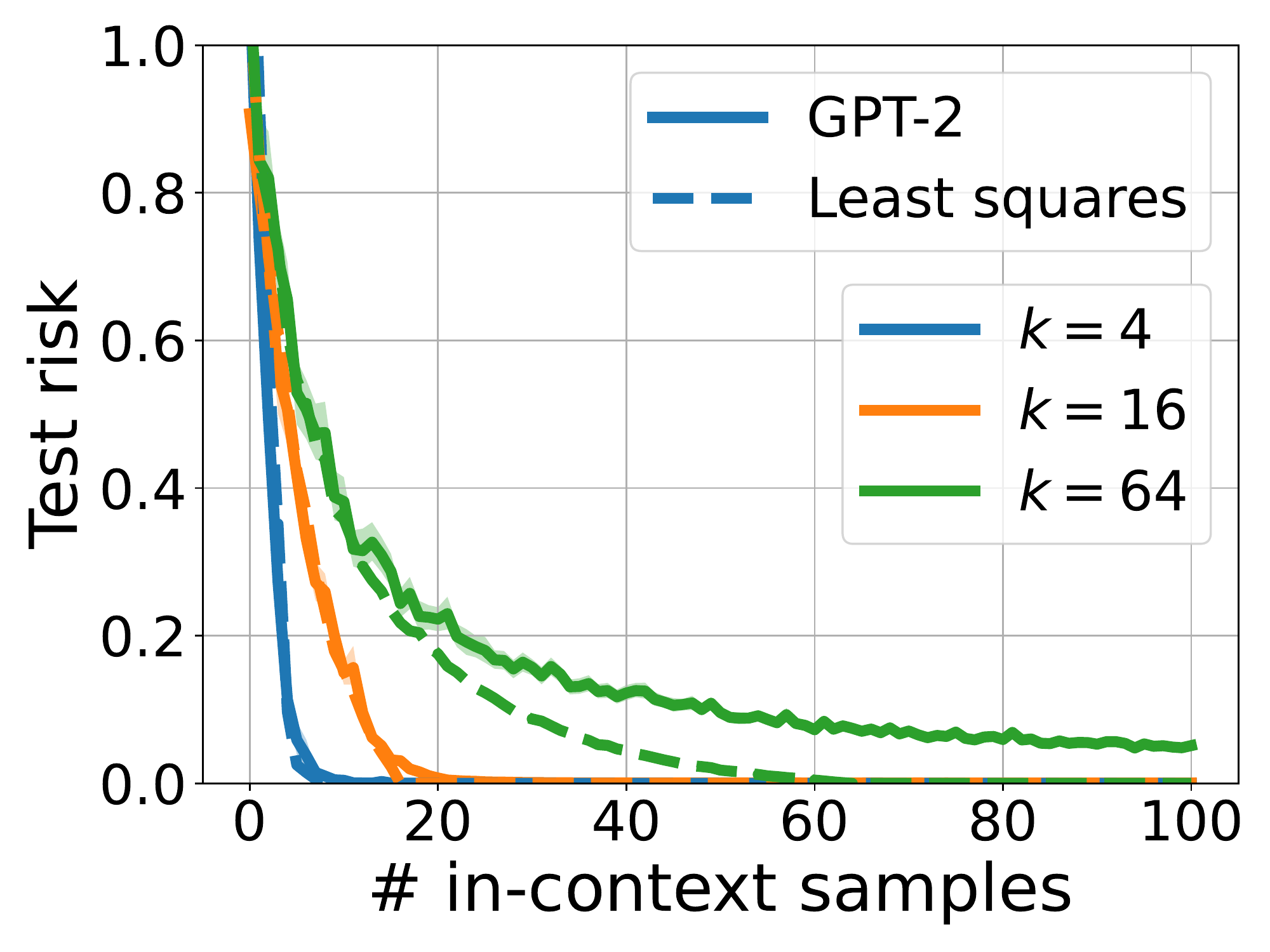}
    \label{fig:linear ls}
    \hspace{10mm}
}
\caption{Fig.~\ref{fig:relu gd}: compare transformer results (solid) with gradient descent optimizer (dashed) when solving first layer of 2-layer MLPs; Fig.~\ref{fig:linear ls}: compare transformer result (solid) with least squares optimizer (dashed) when solving the second layer of 2-layer MLPs.}
\label{fig:linear regression}
\end{figure}
\subsection{Comparison of Transformer Prediction and Linear Regression}
We also provide experimental findings to verify Assumption~\ref{assume oracle} in this section. Previous work \citep{giannou2023looped,akyurek2022learning} has theoretically proven that $\TF$ can perform similar to gradient descent, and empirical evidence from \citep{dai2022can, garg2022can, li2023transformers} suggests that $\TF$ can even be competitive with Bayes optimizer in certain scenario. To this end, we first repeat the same first/layer predictions from Figure~\ref{fig:largek} in Figure~\ref{fig:linear regression}, where $d=10$ and blue, orange and green solid curves represent the performances of $k=4,16,64$ using pretrained small GPT-2 models. We also display the evaluations of gradient descent/least square solutions in dashed curves. Specifically, in Fig.~\ref{fig:relu gd} , we solve problem 
\[
\hat\w_n=\arg\min_{\w}\frac{1}{n}\sum_{i=1}^n\|(\w^\top\x_i)_+-y_i\|^2~~\text{where}~~\x_i\sim\Nc(0,\Iden_d),~y_i=({\w^\st}^\top\x_i)_+
\]
for some $\w^\st\sim\Nc(0,2\Iden_d)$ and $n$ is the training sample size. Then the normalized test risks are computed by $\Lc(n)=\E_{\w^\st,\x}[\|(\hat\w_n^\top\x)_+-y\|^2]/d$, and we show point-to-point results for $n\in[N]$ in black dashed curve in Fig.~\ref{fig:relu gd}\footnote{To mitigate the bias introduced by ReLU activation, we subtract the mean value during prediction, \ie $\Lc(n)=\E_{\w^\st,\x}[\|(\hat\w_n^\top\x)_+-y-(\E_{\x}[(\hat\w_n^\top\x)_+-y])\|^2]/d$. }. As for the second layer, we solve least squares problems as follows
\[
\hat\vv_n=\Sb^\dagger\y~~\text{where}~~\Sb\in\R^{n\times k},~\Sb[i]=(\W^\st\x_i)_+,~ \y[i]={\vv^\st}^\top\Sb[i],~\x_i\in\Nc(0,\Iden_d)
\]
for some $\W^\st\in\R^{k\times d}\sim\Nc(0,2/k)$ and $\vv^\st\sim\Nc(0,\Iden_k)$. Here, $^\dagger$ represents the pseudo-inverse operator. Then we calculate the normalized test risk of least square solution (given $n$ training samples) as $\Lc(n)=\E_{\W^\st,\vv^\st,\x}[\|\hat\vv_n^\top\tb-y\|^2]/d$ where $\tb=(\W^\st\x)_+$ and $y={\vv^\st}^\top\tb$. The results are presented in Fig.~\ref{fig:linear ls} where blue, orange and green dashed curves correspond to solving the problem using different values of $k\in\{4,16,64\}$. In this figure, the curves for $k=4,16$ are aligned with GPT-2 risk curves, which indicates that $\TF$ can efficiently solve linear regression as a least squares optimizer. However, the curve for $k=64$ does not align, which can be attributed to the increased complexity of the function set with higher dimensionality ($k=64$). Learning such complex functions requires a larger $\TF$ model. 

\end{document}